\documentclass[11pt]{article}
\usepackage[preprint]{acl}
\usepackage{times}
\usepackage{latexsym}
\usepackage[T1]{fontenc}
\usepackage[utf8]{inputenc}
\usepackage{microtype}
\usepackage{inconsolata}
\usepackage{graphicx}
\usepackage{amsmath}
\usepackage{amssymb}
\usepackage{caption}
\usepackage{booktabs}
\usepackage{xcolor}
\usepackage{enumitem}
\usepackage{subcaption}
\usepackage{multirow}
\usepackage{mathrsfs} 
\usepackage{amsthm}

\usepackage{tabularx}
\usepackage{array}
\usepackage{makecell}

\theoremstyle{plain}
\newtheorem{theorem}{Theorem}
\newtheorem{proposition}{Proposition}

\theoremstyle{definition}
\newtheorem{assumption}{Assumption}

\theoremstyle{plain}

\newcommand{\up}{\textcolor{green!50!black}{$\uparrow$}}
\newcommand{\down}{\textcolor{red!70!black}{$\downarrow$}}

\usepackage[export]{adjustbox}
\usepackage{siunitx}

\usepackage[most]{tcolorbox}

\definecolor{AccentBlue}{HTML}{1F5FBF}
\definecolor{CardGray}{HTML}{F2F3F5}
\definecolor{RuleGray}{HTML}{1F1F1F}

\newtcolorbox{PromptCardFig}{
  enhanced,
  boxsep=7pt,
  left=10pt,right=10pt,top=8pt,bottom=10pt,
  colback=CardGray,
  colframe=RuleGray,
  arc=8pt,
  boxrule=1.0pt
}

\newcommand{\promptlbl}[1]{\textcolor{AccentBlue}{\bfseries #1}}

\tcbset{
  enhanced,
  boxsep=7pt,
  left=10pt,right=10pt,top=8pt,bottom=10pt,
}

\newtcolorbox{PromptCard}{
  colback=CardGray,
  colframe=RuleGray,
  arc=8pt,
  boxrule=1.0pt,
  breakable
}

\newtcolorbox{PromptCardNoBreak}{
  enhanced,
  boxsep=7pt,
  left=10pt,right=10pt,top=8pt,bottom=10pt,
  colback=CardGray,
  colframe=RuleGray,
  arc=8pt,
  boxrule=1.0pt
}

\definecolor{RespGray}{HTML}{F7F8FA}
\definecolor{ConsAccent}{HTML}{3B6FB6}
\definecolor{AggrAccent}{HTML}{D97706}
\definecolor{RespFrame}{HTML}{D5D9E0}

\newtcolorbox{ResponseCardNoBreak}[2]{
  enhanced,
  colback=RespGray,
  colframe=RespFrame,
  arc=7pt,
  boxrule=0.6pt,
  borderline west={2pt}{0pt}{#1},
  left=7pt,
  right=7pt,
  top=6pt,
  bottom=6pt,
  before skip=2pt,
  after skip=2pt,
  breakable=false,
  fontupper=\scriptsize,
  title={#2},
  fonttitle=\bfseries\small,
  coltitle=black,
  attach title to upper={\par\vspace{1mm}},
}

\title{Reliability-Prioritized Fine-Grained Generation in Multimodal Large Language Models}

\author{
\textbf{Xiaomeng Fan}$^{1,2,*}$,
\textbf{Wei Wu}$^{1,2,*}$,
\textbf{Yuwei Wu}$^{1,2}$,
\textbf{Zhi Gao}$^{1,2,\dagger}$ \\
\textbf{Shiyu Luo}$^{1,2}$,
\textbf{Mingyang Gao}$^{1,2}$,
\textbf{Haoyu Zhao}$^{1,2}$,
\textbf{Zhenxin Diao}$^{1,2}$ \\
\textbf{Yuxuan Ba}$^{1,2}$,
\textbf{Lijia Feng}$^{1,2}$,
\textbf{Yunde Jia}$^{2,1,\dagger}$,
\textbf{Mehrtash Harandi}$^{3}$ \\
\\[-0.3em]
\small $^{1}$Beijing Key Laboratory of Intelligent Information Technology, \\
\small School of Computer Science \& Technology, Beijing Institute of Technology \\
\small $^{2}$Guangdong Laboratory of Machine Perception and Intelligent Computing, Shenzhen MSU-BIT University \\
\small $^{3}$Department of Electrical and Computer System Engineering, Monash University \\
\\[-0.2em]
\small $^{*}$Equal contribution. \quad
\small $^{\dagger}$Corresponding authors.
}

\begin{document}
\maketitle

\begin{abstract}

Multimodal large language models (MLLMs) are increasingly expected to generate fine-grained descriptions of visual content.
However, we observe and theoretically show that generating fine-grained responses poses a reliability  challenge, \textit{i.e.}, fine-grained generation is more error-prone than coarse-grained generation. 
This phenomenon suggests that models should generate the finest description that remains reliable rather than simply produce more specific outputs.
To investigate this problem, we develop \textsc{GranFact}, a granularity-aware benchmark consisting of expert-verified multi-object images with coarse-to-fine category annotations. Then, we design a hierarchy-aware evaluation algorithm, which assesses both whether model predictions are visually correct and how specific the correct predictions are.
We also propose a reliability-prioritized preference optimization method based on Direct Preference Optimization, which penalizes unreliable fine-grained claims while rewarding reliable specificity.
Experiments on \textsc{GranFact} show that our method improves fine-grained generation while preserving reliability.
Code and data are available \href{https://github.com/WeiWu2025/GranFact}{here}.

\end{abstract}
\vspace{-1mm}
\section{Introduction}
\vspace{-1mm}
\begin{figure*}[h]
    \centering
    \includegraphics[width=0.87\linewidth]{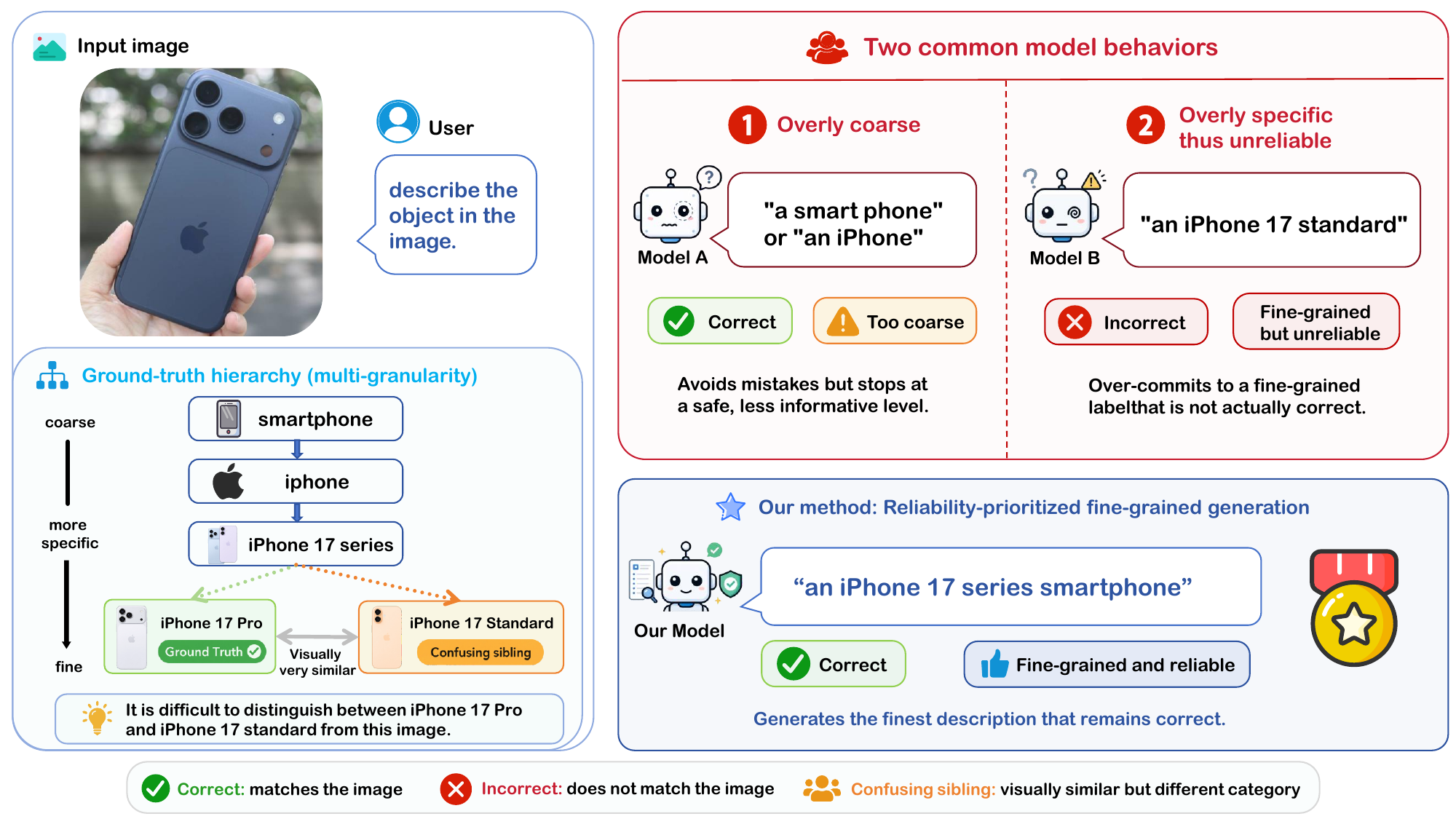}
    \caption{
Motivation of reliability-prioritized fine-grained generation.
Rather than simply maximizing specificity, a model should generate the finest description that remains reliable.
}
    \label{fig:motivation}
\end{figure*}

Multimodal large language models (MLLMs) have demonstrated strong visual understanding capabilities~\cite{alayrac2022flamingo,li2023blip2,liu2023visual,zhu2024minigpt,dai2023instructblip}, enabling open-ended descriptions of visual content across diverse scenarios. However, their reliability remains a significant challenge when generating fine-grained visual descriptions~\cite{kim2024finer,he2025finedefics,ye2025painting,yu2025benchmarking}. We observe and theoretically show that
producing a fine-grained generation is substantially more error-prone than producing a coarse-grained one for MLLMs.
As illustrated in Figure~\ref{fig:motivation}, given an image of an ``iPhone 17 Pro'', a model can often correctly describe it as a ``smartphone'' or an ``iPhone'', yet when asked for a fine-grained description, it may misidentify it as a visually similar model, such as an ``iPhone 17 standard''.

This phenomenon suggests that fine-grained visual description should not simply encourage more specific responses. Instead, models should generate the finest reliable description, and back off to a coarser but correct one when finer-grained predictions become unreliable. We refer to this problem as reliability-prioritized fine-grained visual description generation, where the model first ensures correctness and then chooses the finest description among reliable candidates.

Existing visual description benchmarks~\cite{dong2024benchmarking,cheng2025caparena,liu2025capability} and fine-grained recognition benchmarks~\cite{wah2011cub,krause2013cars,maji2013aircraft,kim2024finer,he2025finedefics,yu2025fgbmk} are insufficient for evaluating this capability. From the data perspective, existing benchmarks often focus on single objects or scenes with limited visual confusion, whereas realistic multi-object fine-grained benchmarks require substantial effort in image collection and coarse-to-fine annotation.
From the evaluation perspective, existing protocols usually judge correctness at a predefined granularity, making it non-trivial to determine whether coarser or finer descriptions remain correct. As a result, they fail to reliably evaluate models across granularities.
This issue is more challenging in open-ended multi-object generation, where evaluation must first identify which ground-truth (GT) object each prediction refers to, and then decide its granularity.
Without global alignment, repeated or unsupported predictions may be incorrectly rewarded, while missed objects are overlooked.

To address these challenges, we construct a granularity-aware benchmark, namely \textsc{GranFact}. \textsc{GranFact} consists of 581 images across 7 domains. The images are collected from the Internet and real-world photographs with multiple visually relevant objects. For each object, human experts provide category annotations at coarse-to-fine semantic granularities.
For evaluation, we introduce a hierarchy-aware evaluation algorithm that parses open-ended responses into structured entities and assigns them to GT entities under quantity constraints.
To resolve the assignments between predicted and GT entities, we introduce a reliability-prioritized objective that first prioritizes category-consistent or category-coarser assignments and then favors the finest-grained assignment.
Based on the final assignment, we design two types of metrics. Granularity-neutral metrics count a prediction as correct if it matches any annotated category level, while granularity-aware metrics further measure the specificity of correct descriptions.

Based on Direct Preference Optimization (DPO)~\cite{rafailov2023direct}, we propose Reliability-Prioritized DPO (RP-DPO) to improve fine-grained generation of MLLMs under a reliability-first constraint.
RP-DPO first applies rollback to hallucinated predicted objects by replacing them with coarse-grained ancestors shared with the GT. It then constructs reliability preferences between original and rolled-back responses, and granularity preferences among reliable responses. With larger margins for reliability, RP-DPO enforces reliability-first optimization while favoring finer-grained descriptions under comparable reliability.
Experiments with Qwen3-VL-8B on \textsc{GranFact} highlight the importance of reliability-first optimization for fine-grained multimodal understanding.

In summary, our contributions are threefold: 
\vspace{-3mm}
\begin{itemize}
    \item We introduce \textsc{GranFact}, a granularity-aware multimodal benchmark that includes expert-verified multi-object images with coarse-to-fine hierarchical labels.
    \vspace{-3mm}
    \item We propose a hierarchy-aware evaluation algorithm that aligns open-ended predictions with GT entities and measures the granularity and correctness of model responses.
    \vspace{-3mm}
    \item We propose RP-DPO, a reliability-prioritized preference optimization method that improves fine-grained visual description generation while preserving reliability.
\end{itemize}

\vspace{-3mm}
\section{Related Work}
\vspace{-1mm}
\subsection{MLLM Benchmarks}
\vspace{-1mm}
Recent MLLM benchmarks evaluate broad multimodal abilities, including perception and cognition~\cite{fu2025mme}, general multimodal capabilities~\cite{liu2024mmbench,zhang2025mme}, visual reasoning~\cite{lu2024mathvista}, and expert-level understanding~\cite{yue2024mmmu}. More targeted benchmarks study fine-grained recognition~\cite{yu2025fgbmk} and hallucination detection~\cite{li2023evaluating,wang2023amber,guan2024hallusionbench,cai2025mhalo,yin2026freak}. 
These benchmarks typically target fixed-granularity answers or binary correctness. In contrast, our benchmark jointly evaluates the correctness and granularity of open-ended descriptions.

\vspace{-2mm}
\subsection{Fine-Grained Visual Description}
\vspace{-1mm}
Prior work improves the visual specificity of MLLMs through detailed caption data~\cite{chen2024sharegpt4v}, grounded generation~\cite{you2024ferret,rasheed2024glamm}, and fine-grained recognition training~\cite{he2025finedefics,he2026fine,wu2026hypmodalalign}. These methods make outputs more detailed or discriminative, but do not explicitly prioritize reliability when increasing semantic specificity. We instead focus on reliability-prioritized fine-grained generation, encouraging more specific descriptions while preserving reliability.
\vspace{-2mm}
\subsection{Hallucination Mitigation in MLLMs}
\vspace{-1mm}
Hallucination mitigation aims to reduce visual claims that are not grounded in the input image. Existing methods address this issue through inference-time interventions, such as contrastive decoding~\cite{leng2024vcd,tong2025mitigating,chen2025dcd}, post-hoc correction~\cite{zhou2024lure,yin2024woodpecker}, or additional supervision and alignment signals~\cite{yu2024rlhfv}.
These methods improve reliability by reducing ungrounded claims, but do not explicitly consider the granularity of grounded descriptions.
Our work treats hallucination reduction as a prerequisite and further encourages finer-grained descriptions.

\vspace{-2mm}
\subsection{Direct Preference Optimization}
\vspace{-1mm}
Direct Preference Optimization (DPO) directly optimizes preference pairs without training an explicit reward model~\cite{rafailov2023direct}. Recent preference-optimization methods introduce target or adaptive margins to model different preference strengths~\cite{meng2024simpo,wu2025alphadpo,sun2025gamma}. Other works extend DPO to multi-objective alignment~\cite{zhou2024modpo} or multimodal settings~\cite{wang2024mdpo,fu2025hdpo}. Our method follows this line but focuses on reliability-prioritized fine-grained generation.

\vspace{-2mm}
\section{Analysis}
\label{sec:analysis}
\vspace{-1.5mm}

We analyze reliability--granularity trade-off from both theoretical and empirical perspectives.

\vspace{-2mm}
\subsection{Theoretical Analysis}
\label{sec:analysis_theory}
\vspace{-1mm}

Let $T$ denote the visual representation of image $I$ produced by the
model's visual encoder.
For two comparable semantic granularities, let $\Omega_f$ and $\Omega_c$
be the corresponding fine- and coarse-grained cut sets, and let $Y_f$ and
$Y_c$ denote their semantic labels.
We define the refinement information deficit as
$\eta_{f\mid c}=H(Y_f\mid T,Y_c)$ and use it to quantify the degree of
semantic refinement.
For a predictive distribution $\widehat P_\Omega(\cdot\mid T)$, we define
the log-score uncertainty risk as
$\mathcal R_{\mathrm{ls}}(\widehat P_\Omega)
=\mathbb E[-\log \widehat P_\Omega(Y_\Omega\mid T)]$.

\begin{theorem}[Uncertainty-Risk Gap]
\label{thm:semantic_risk_monotonicity}
Under Assumption~\ref{assumption:appendix_average_hierarchy_stability}
in Appendix~\ref{sec:appendix_theoretical_analysis}, the expected
coarse-to-fine risk gap is lower bounded by the expected refinement
information deficit:
\vspace{-2mm}
\begin{equation}
\small
\mathbb E_{(f,c)}
\left[
\mathcal R_{\mathrm{ls}}(\widehat P_f)
-
\mathcal R_{\mathrm{ls}}(\widehat P_c)
\right]
\ge
\mathbb E_{(f,c)}
\left[
\eta_{f\mid c}
\right].
\label{eq:semantic_risk_monotonicity_main}
\end{equation}
\end{theorem}
\vspace{-2mm}
The formal definitions, assumption, and proof are provided in
Appendix~\ref{sec:appendix_theoretical_analysis}.
Theorem~\ref{thm:semantic_risk_monotonicity} formalizes the
reliability--granularity trade-off in a hierarchical semantic space,
showing that generating at a finer granularity is inherently more error-prone on average.
\vspace{-2mm}

\subsection{Reliability--Granularity Trade-off}
\label{sec:analysis_prompt_tradeoff}
\vspace{-1mm}

\vspace{-2mm}
\begin{figure}[ht]
    \centering
    \includegraphics[width=0.9\linewidth]{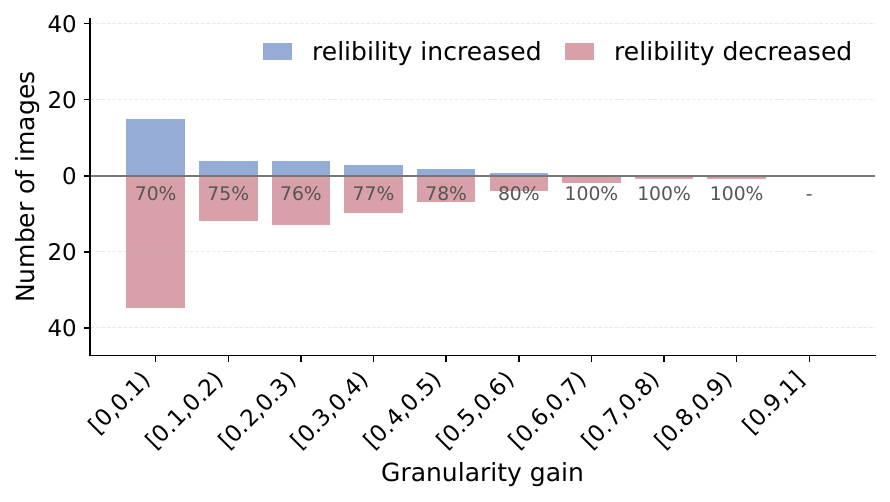}
    \caption{
    Reliability changes from conservative to aggressive prompting across per-image granularity gains.
    }
    \label{fig:prompt_specificity_tradeoff}
\end{figure}
\vspace{-3.5mm}
We conduct an empirical analysis to examine the trade-off between granularity and reliability. We prompt a model to describe the same input images using either a conservative prompt or an aggressive prompt that asks for maximal detail.
We then measure how reliability changes, as shown in Figure~\ref{fig:prompt_specificity_tradeoff}.
Finer-grained generations are more frequently associated with reliability drops.

This suggests that simply eliciting more specific descriptions is insufficient. Instead, models should generate finer-grained descriptions only when the additional specificity can be stated reliably.

 \vspace{-1mm}
\section{\textsc{GranFact}}
\label{sec:benchmark}
\vspace{-2mm}
\textsc{GranFact} consists of a manually annotated fine-grained image dataset and an evaluation protocol.



\vspace{-1mm}
\subsection{Evaluation Set Construction}
\label{sec:benchmark_construction}

\begin{figure*}[t]
    \centering
    \includegraphics[width=\linewidth]{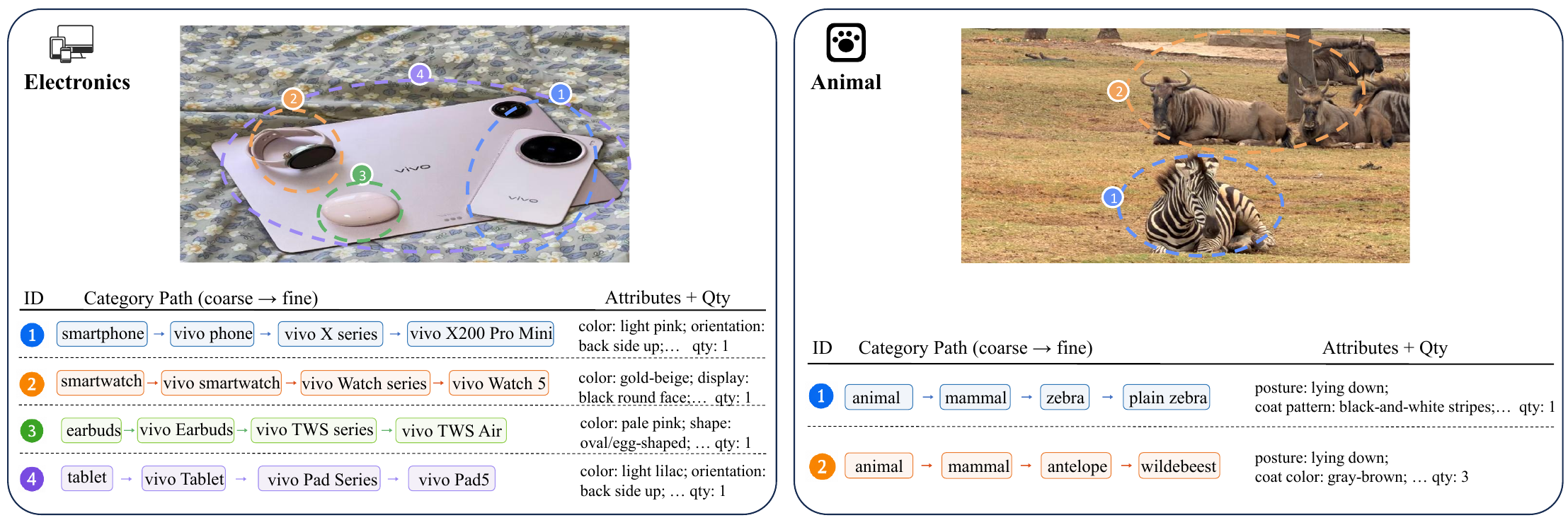}
    \caption{Qualitative examples of dataset annotations across different domains.
}
    \label{fig:dataset_example}
\end{figure*}
\begin{figure*}[t]
    \centering
    \captionsetup[subfigure]{
        font=small,
        labelfont=bf,
        labelformat=parens,
        labelsep=none,
        justification=centering,
        skip=2pt
    }

    \newcommand{\subfigheight}{3.15cm}

    \begin{subfigure}[t]{0.31\textwidth}
        \centering
        \includegraphics[height=\subfigheight, width=\linewidth, keepaspectratio]
        {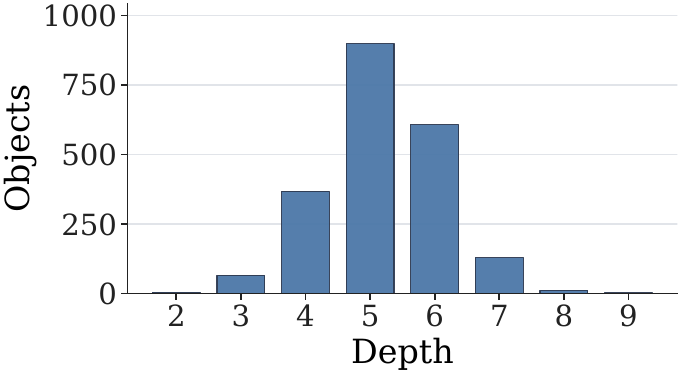}
        \caption{}
        \label{fig:level_distribution}
    \end{subfigure}
    \hspace{0.035\textwidth}
    \begin{subfigure}[t]{0.31\textwidth}
        \centering
        \includegraphics[height=\subfigheight, width=\linewidth, keepaspectratio]
        {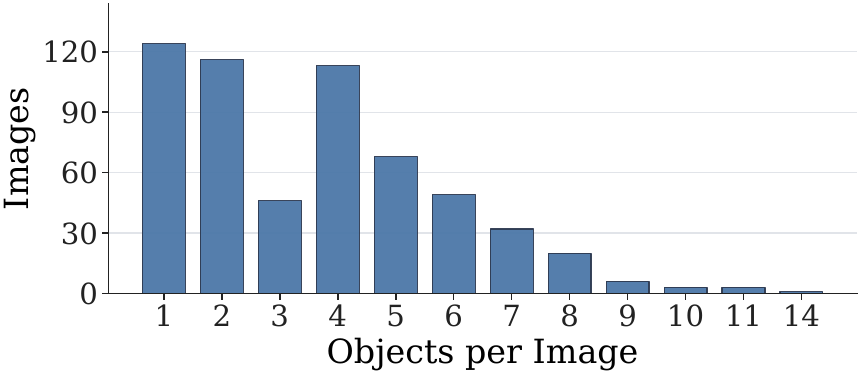}
        \caption{}
        \label{fig:objects_per_image}
    \end{subfigure}
    \hspace{0.035\textwidth}
    \begin{subfigure}[t]{0.25\textwidth}
        \centering
        \includegraphics[height=\subfigheight, keepaspectratio]{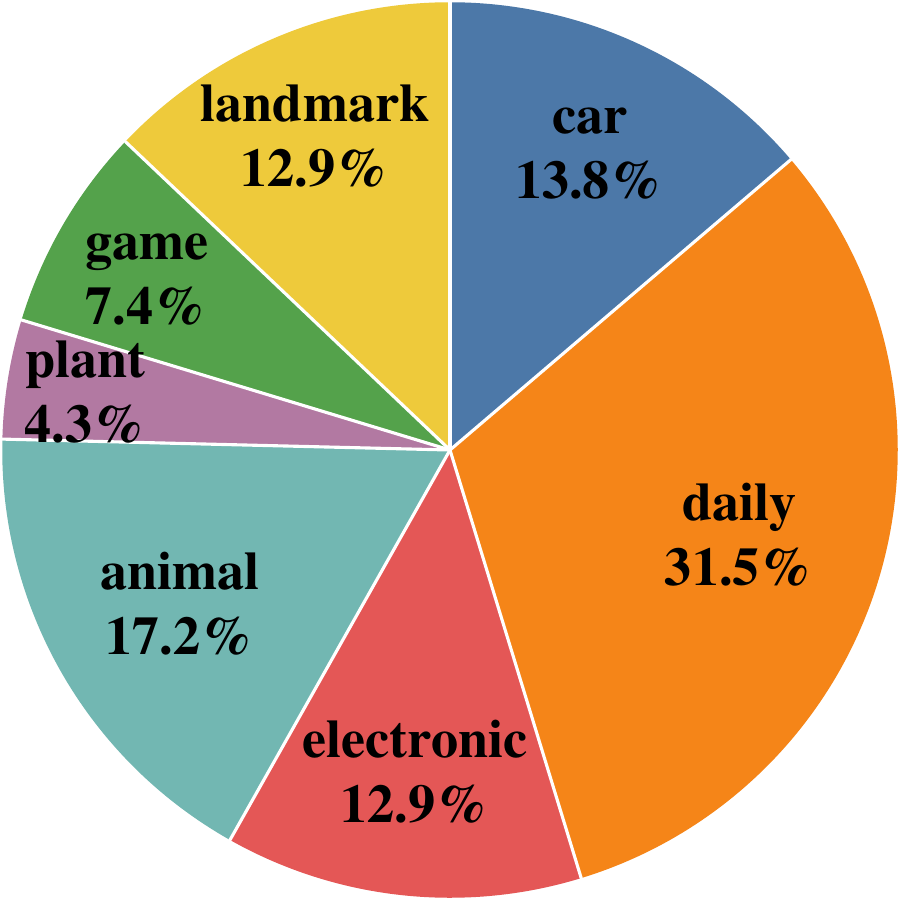}
        \caption{}
        \label{fig:domain_distribution}
    \end{subfigure}

    \vspace{-0.6em}
    \caption{
    Dataset statistics of \textsc{GranFact}.
        Panels (a) and (b) include only objects with multi-granularity
        category annotations.
        (a) Distribution of the maximum category-path depth per annotated object.
        (b) Distribution of multi-granularity annotated objects per image.
        (c) Domain distribution of images.
    }
    \label{fig:dataset_statistics}
    \vspace{-0.8em}
\end{figure*}

\textsc{GranFact} contains 581 images collected through web sourcing and real-world photography across 7 visual domains: daily objects, plants, animals, cars, electronics, landmarks, and games.
Each annotated entity is manually labeled with coarse-to-fine categories, quantity, and visual attributes, with examples shown in Figure~\ref{fig:dataset_example}.
Figure~\ref{fig:dataset_statistics} summarizes granularity depth, annotated objects per image, and domain distribution, with further details provided in Appendix~\ref{sec:appendix_benchmark_construction}.



Let $\mathcal{Z}$ denote the evaluation set. 
Each image is annotated with a set of GT entities $\mathcal{G}=\{g_j\}_{j=1}^{n}$. 
Each entity $g_j$ is labeled with a coarse-to-fine category $\mathcal{H}_j=(c_{j,1},\ldots,c_{j,L_j})$, a quantity capacity $d_j$, and visual attributes, where $c_{j,1}$ and $c_{j,L_j}$ denote the coarsest and finest categories, respectively.

\vspace{-2mm}
\subsection{Hierarchy-Aware Evaluation Algorithm}
\label{sec:evaluation_algorithm}
\vspace{-1mm}
We propose a hierarchy-aware evaluation pipeline for open-ended visual descriptions.
It first constructs a feasible set of prediction--GT assignments, then selects the final assignment with a reliability-prioritized, granularity-aware objective.



\vspace{-2mm}
\subsubsection{Response Parsing}
\label{sec:response_parsing}
Given an image and its corresponding MLLM response, we use an LLM to parse the free-form response into structured prediction entities.
The parsed prediction set is denoted as $\mathcal{P}=\{p_i\}_{i=1}^{m}$, where each prediction entity $p_i$ consists of a predicted category $q_i$, a predicted quantity $s_i$, and a set of non-quantitative attributes $\mathbf{u}_i$.
The full parsing details are provided in Appendix~\ref{sec:appendix_parser}.


\vspace{-2mm}
\subsubsection{Prediction--GT Assignment}
\label{sec:flow_formulation}


After parsing an open-ended response into structured prediction entities, evaluation requires assigning predicted entities to annotated GT entities.
This assignment is non-trivial because the number of predicted entities and GT entities may differ, while model outputs can appear at arbitrary semantic granularities and may include hallucinated entities.
We first augment the GT entity set with a dummy hallucination entity $g_{n+1}$, obtaining $\widetilde{\mathcal{G}}=\mathcal{G}\cup\{g_{n+1}\}$.
We then define an assignment matrix $\mathbf{X}\in\mathbb{R}_{\ge 0}^{m\times(n+1)}$ to represent the quantity assignment from prediction entities to the augmented GT entity set.
Unlike a binary matching matrix, each entry in $\mathbf{X}$ specifies how much of a predicted quantity is assigned to a GT entity. For example, a prediction with quantity larger than one may be distributed across several compatible GT entities.

We constrain that an assignment matrix $\boldsymbol{X}$ must assign each predicted quantity to GT entities or the hallucination entity, \textit{i.e.}, $\sum_{j=1}^{n+1} x_{ij}=s_i$. 
Moreover, we also constrain the total quantity assigned to each GT entity $g_j$ to be at most its annotated capacity $d_j$, \textit{i.e.},  $\sum_{i=1}^{m}x_{ij}\le d_j$.
The hallucination entity $g_{n+1}$ is left unconstrained, so unsupported or over-counted predictions can always be assigned to it and counted as hallucination.
Overall, the constraints on $\boldsymbol{X}$ can be summarized as 
\vspace{-2mm}
\begin{equation}
\small
\label{eq:assignment_feasible_set}
\begin{aligned}
\mathcal{X}
=
\Big\{
&\mathbf{X}\in\mathbb{R}_{\ge 0}^{m\times(n+1)}
\ \Big| 
 \sum_{j=1}^{n+1} x_{ij}=s_i,\quad 1\le i\le m, \\
& \sum_{i=1}^{m}x_{ij}\le d_j,\quad 1\le j\le n
\Big\}.
\end{aligned}
\end{equation}
\vspace{-2mm}
\subsubsection{Objective of Prediction--GT Assignment}
\label{sec:flow_objective}
\label{sec:granularity_matching}
\label{sec:flow_alignment}

With the assignment set $\mathcal{X}$, we select the final assignment matrix using a reliability-prioritized objective.
This objective is based on two assignment scores, \textit{i.e.}, a reliability score $M(\mathbf{X})$ and a granularity-aware score $M_{\mathrm{gran}}(\mathbf{X})$.  
$M(\mathbf{X})$ measures the degree to which predicted entities are assigned to category-compatible GT entities, where an assignment is considered category-compatible if the predicted category matches the GT category or one of its ancestor categories.
 $M_{\mathrm{gran}}(\mathbf{X})$ further accounts for category granularity and attribute consistency.
The final assignment matrix maximizes these two scores in lexicographic order, giving primary priority to $M(\mathbf{X})$ and using $M_{\mathrm{gran}}(\mathbf{X})$ to distinguish assignments with the same reliability score, which can be modeled as\footnote{
Optimal assignments may be non-unique but yield identical metric values; see Appendix~\ref{sec:appendix_welldefined_metrics}.
}
\vspace{-1mm}
\begin{equation}
\small
\label{eq:lexicographic_flow}
\begin{aligned}
\mathbf{X}^{\star}
&\in
\arg\max_{\mathbf{X}' \in \mathcal{X}}
\quad
M_{\mathrm{gran}}(\mathbf{X}') \\
\mathrm{s.t.}\quad
\mathbf{X}'
&\in
\arg\max_{\mathbf{X}\in\mathcal{X}}
M(\mathbf{X}) .
\end{aligned}
\end{equation}

To compute the two assignment-level scores $M(\cdot)$ and $M_{\mathrm{gran}}(\cdot)$, we define a granularity-aware assignment score $w_{ij}$ for each prediction entity $p_i$ and GT entity $g_j$.
$w_{ij}$ assigns positive scores to category-compatible assignments, with higher scores for finer category levels and higher attribute consistency. It assigns negative scores to incompatible assignments and neutral scores to assignments to the dummy GT entity. Formally, $w_{ij}$ is computed as
\vspace{-1mm}
\begin{equation}
\small
w_{ij}
=
\begin{cases}
\dfrac{\ell_{ij}+a_{ij}}{L_j+1},
& j\le n,\ \ell_{ij}>0, \\[5pt]
-1,
& j\le n,\ \ell_{ij}=0, \\[3pt]
0,
& j=n+1 .
\end{cases}
\label{eq:edge_weight}
\end{equation}
$\ell_{ij}\in\{0,\ldots,L_j\}$ is the category granularity level assigned to $(p_i,g_j)$, where $\ell_{ij}\leq 0$ indicates category incompatibility and a larger value indicates a finer supported category level. 
$a_{ij}\in[0,1]$ is the attribute consistency score between $p_i$ and $g_j$, computed over their non-quantitative attributes.
Using these scores, $M(\mathbf{X})$ and $M_{\mathrm{gran}}(\mathbf{X})$ are computed as
\vspace{-2mm}
\begin{equation}
\small
\begin{aligned}
M(\mathbf{X})
&=
\sum_{i=1}^{m}
\sum_{j=1}^{n}
\mathbb{I}[w_{ij}>0]x_{ij}, \\
M_{\mathrm{gran}}(\mathbf{X})
&=
\sum_{i=1}^{m}
\sum_{j=1}^{n+1}
w_{ij}x_{ij}.
\end{aligned}
\label{eq:flow_scores}
\end{equation}

We solve the lexicographic assignment in Eq.~(\ref{eq:lexicographic_flow}) by casting the quantity assignment as a minimum-cost flow problem on a bipartite network~\citep{ahuja1988network} and solving it with a successive shortest augmenting-path algorithm on the residual network~\citep{klein1967primal}. 
Appendix~\ref{sec:appendix_eval_details} provides additional details on the overall evaluation algorithm.


\vspace{-2mm}
\subsection{Evaluation Metrics}
\vspace{-1mm}
\label{sec:flow_metrics}
We compute metrics by aggregating quantities at two levels.
For each example $z$, we aggregate over entities to obtain $S_z=\sum_i s_i$, $D_z=\sum_j d_j$, $M_z=M(\mathbf{X}^{\star}_z)$, and $M_{\mathrm{gran},z}=M_{\mathrm{gran}}(\mathbf{X}^{\star}_z)$.
We omit the subscript $z$ for dataset-level sums over examples, e.g., $S=\sum_{z\in\mathcal{Z}}S_z$ and similarly for $D$, $M$, and $M_{\mathrm{gran}}$.
\textsc{GranFact} reports granularity-neutral metrics and granularity-aware metrics.

\vspace{-2mm}
\paragraph{Granularity-neutral metrics.}
Granularity-neutral metrics measure whether predictions are assigned to category-compatible GT entities, regardless of how fine-grained the descriptions are.
The precision, recall, and F1 are
\vspace{-2mm}
\begin{equation}
\small
\mathrm{P}=\frac{M}{S},\quad
\mathrm{R}=\frac{M}{D},\quad
\mathrm{F1}=
\frac{2\mathrm{P}\mathrm{R}}{\mathrm{P}+\mathrm{R}} .
\label{eq:normal_prf}
\end{equation}
We also report the Image Reliability Rate that is
\vspace{-2mm}
\begin{equation}
\small
    \mathrm{IR}
    =
    \frac{1}{|\mathcal{Z}|}
    \sum_{z\in\mathcal{Z}}
    \mathbb{I}[M_z=S_z],
    \label{eq:image_reliability}
\end{equation}
where an example is reliable only when all predicted quantities are assigned to category-compatible GT entities.

\vspace{-2mm}
\paragraph{Granularity-aware metrics.}
Granularity-aware metrics further account for the specificity and attribute consistency of reliable predictions.
The granularity-weighted precision, recall, and F1 are
\vspace{-2mm}
\begin{equation}
\small
\begin{gathered}
\mathrm{P}_{\mathrm{gran}}=\frac{M_{\mathrm{gran}}}{S},
\quad
\mathrm{R}_{\mathrm{gran}}=\frac{M_{\mathrm{gran}}}{D}, \\
\mathrm{F1}_{\mathrm{gran}}
=
\frac{2\mathrm{P}_{\mathrm{gran}}\mathrm{R}_{\mathrm{gran}}}
{\mathrm{P}_{\mathrm{gran}}+\mathrm{R}_{\mathrm{gran}}}.
\end{gathered}
\label{eq:gran_prf}
\end{equation}
Finally, we report the granularity-aware counterpart of image reliability:
\vspace{-2mm}
\begin{equation}
\small
    \mathrm{GIR}
=
\frac{1}{|\mathcal{Z}|}
\sum_{z\in\mathcal{Z}}
\mathbb{I}[M_z=S_z]
\cdot
\frac{M_{\mathrm{gran},z}}{M_z},
    \label{eq:gran_image_reliability}
\end{equation}
where the granularity term is set to zero when $M_z=0$.
$\mathrm{GIR}$ gives non-zero credit only to fully reliable responses and weights them by their average granularity-aware score.
As an additional diagnostic, we also report the average granularity score of reliable predictions
\vspace{-2mm}
\begin{equation}
\small
    \mathrm{G}_{\mathrm{avg}}
    =
    \frac{M_{\mathrm{gran}}}{M},
    \label{eq:g_avg}
\end{equation}
where $\mathrm{G}_{\mathrm{avg}}$ is set to $0$ when $M=0$.

\vspace{-2mm}
\section{Reliability-prioritized DPO}
\label{sec:method}
\vspace{-1mm}

We propose Reliability-prioritized DPO (RP-DPO), a preference optimization method for fine-grained visual description generation.
It consists of reliability-guided semantic rollback (RSR), preference construction, and metric-margin DPO.

\vspace{-2mm}
\subsection{Reliability-Guided Semantic Rollback}
\vspace{-1mm}
\label{sec:rsr}
\begin{figure}[ht]
    \centering
    \includegraphics[width=0.9\columnwidth]{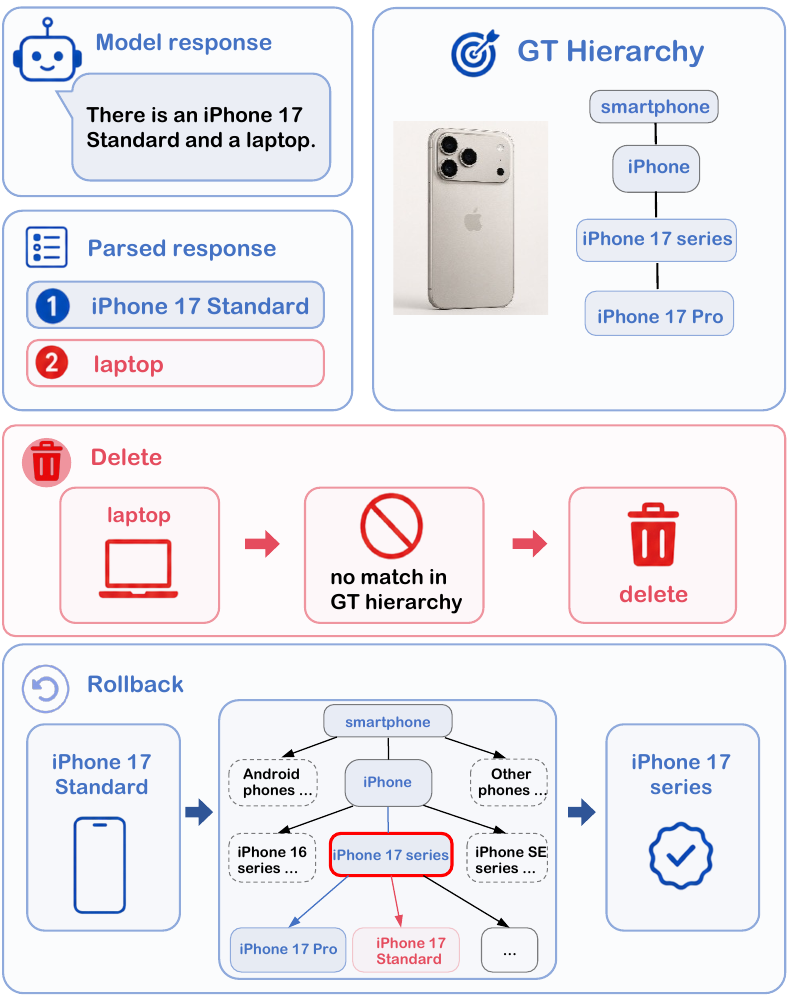}
    \caption{
    Illustration of RSR with \textsc{Delete} and \textsc{Rollback} operations.
    }
    \label{fig:rsr_example}
    \vspace{-3mm}
\end{figure}

Since an image often contains multiple objects, models may fail to produce correct descriptions for all objects in responses. To construct reliable positive samples for DPO, RSR converts erroneous responses into semantically valid ones by rolling incorrect predictions back to the coarse-grained ancestor category shared with GT objects. 

Specifically, we first parse the model response $y_{\mathrm{orig}}$ into prediction objects. For a hallucinated predicted object, we apply \textsc{Delete} or \textsc{Rollback} to correct it.
If a predicted object and its ancestor categories have no intersection with the hierarchical annotations of any GT entity,
RSR assigns a \textsc{Delete} to remove the predicted object.
If a coarser category of the predicted object matches an ancestor category of a GT object, RSR assigns a \textsc{Rollback} to replace it with the finest shared ancestor category.
Based on these two operations, we prompt an LLM with the original response and the corrected objects to produce a rectified response $y_{\mathrm{rect}}=\mathrm{RSR}(y_{\mathrm{orig}})$.
We illustrate an example in Figure~\ref{fig:rsr_example}.

\vspace{-2mm}
\subsection{Preference Construction}
\label{sec:pref_construction}
\vspace{-1mm}
%


We construct two complementary preference sets, $\mathcal{D}_{\mathrm{rel}}$ for reliability preferences and $\mathcal{D}_{\mathrm{gran}}$ for granularity preferences. For each image $I$, we sample $K$ candidate responses $\mathcal{Y}_{\mathrm{orig}}(I)=\{y_{\mathrm{orig}}^{(k)}\}_{k=1}^{K}$ and obtain $K$ rectified responses $\mathcal{Y}_{\mathrm{rect}}(I)=\{\mathrm{RSR}(y_{\mathrm{orig}}^{(k)})\}_{k=1}^{K}$.

Reliability preferences encourage models to avoid category-incompatible predictions by preferring RSR-rectified responses over their original versions.
We construct $\mathcal{D}_{\mathrm{rel}}=\{(I, y_{+}, y_{-})\}_{k}$ with $(y_{+}, y_{-})=(y_{\mathrm{rect}}^{(k)}, y_{\mathrm{orig}}^{(k)})$, where $y_{+}^{(k)}\sim\mathcal{Y}_{\mathrm{rect}}(I)$ and $y_{-}^{(k)}\sim\mathcal{Y}_{\mathrm{orig}}(I)$.



Granularity preferences encourage more informative responses in the RSR-rectified response pool.
We construct $\mathcal{D}_{\mathrm{gran}}=\{(I,y_{+},y_{-})\}_{k}$, where $y_{+},y_{-}\sim \mathcal{Y}_{\mathrm{rect}}(I)$ and 
$\mathrm{F1}_{\mathrm{gran}}(I,y_{+})>\mathrm{F1}_{\mathrm{gran}}(I,y_{-})+\tau$, with $\tau$ filtering out pairs with similar scores.


\vspace{-2mm}
\subsection{Metric-Margin DPO}
\label{sec:margin_dpo}
\vspace{-1mm}

We optimize the model with a metric-margin DPO objective,
which is computed as
\vspace{-2mm}
\begin{equation}
\small
    \mathcal{L}
    =
    -
    \mathbb{E}_{(I,y_{+},y_{-})\sim\mathcal{D}}
    \log
    \sigma
    \left(
    \Delta_\theta(I,y_{+},y_{-}) - m
    \right),
    \label{eq:metric_margin_dpo}
\end{equation}
where $ \Delta_\theta(I,y_{+},y_{-})$ denotes a reward gap, and $m$ denotes the margin.
We assign an instance-specific margin to each preference pair $(y_{w}, y_{l})$.  A larger performance gap $\delta\in[0,1]$ of preference pairs leads to a larger margin, which is modeled as $m = \gamma + \alpha \delta$, where $\gamma$ and $\alpha$ are hyper-parameters.

The reward gap is defined as $\Delta_\theta(I,y_{+},y_{-})=r_\theta(I,y_{+})-r_\theta(I,y_{-})$.
$r_\theta(\cdot)$ denotes implicit DPO reward, which is computed as
\vspace{-2mm}
\begin{equation}
\small
    r_\theta(I,y)
    =
    \beta
    \log
    \frac{\pi_\theta(y \mid I)}
         {\pi_{\mathrm{ref}}(y \mid I)} ,
\end{equation}
where $\pi_\theta$ is the policy model, $\pi_{\mathrm{ref}}$ is the frozen reference model, and $\beta$ is the  temperature.

The computation of $m$ depends on the preference set. For reliability preferences $(I,y_{+},y_{-})\in\mathcal{D}_{\mathrm{rel}}$, $m_{\mathrm{rel}} =  \gamma_{\mathrm{rel}} + \alpha_{\mathrm{rel}} \delta_{\mathrm{rel}}$. With
$\mathrm{P}$ in Eq.~\eqref{eq:normal_prf}, $\delta_{\mathrm{rel}}$ is
\vspace{-2mm}
\begin{equation}
\small
\begin{aligned}
    \delta_{\mathrm{rel}}
=
\frac{\mathrm{P}(y_{+})-\mathrm{P}(y_{-})}
{\max_{(I,y_{+},y_{-})\in\mathcal{D}_{\mathrm{rel}}}\left(\mathrm{P}(y_{+})-\mathrm{P}(y_{-})\right)},
\end{aligned}
\end{equation}
where the denominator normalizes the gap into $[0,1]$. For granularity preferences $(I,y_{+},y_{-})\in\mathcal{D}_{\mathrm{gran}}$, $m_{\mathrm{gran}} =  \gamma_{\mathrm{gran}} + \alpha_{\mathrm{gran}} \delta_{\mathrm{gran}}$. $\delta_{\mathrm{gran}}$ is
\vspace{-2mm}
\begin{equation}
\small
\begin{aligned}
    \delta_{\mathrm{gran}}
=
\frac{\mathrm{F1}_{\mathrm{gran}}(y_{+})-\mathrm{F1}_{\mathrm{gran}}(y_{-})}
{\max_{(I,y_{+},y_{-})\in\mathcal{D}_{\mathrm{gran}}}\left(\mathrm{F1}_{\mathrm{gran}}(y_{+})-\mathrm{F1}_{\mathrm{gran}}(y_{-})\right)},
\end{aligned}
\end{equation}
where $\mathrm{F1}_{\mathrm{gran}}$ is defined in Eq.~\eqref{eq:gran_prf}.
We set $\gamma_{\mathrm{rel}}$ and $\alpha_{\mathrm{rel}}$  larger than $\gamma_{\mathrm{gran}}$ and $\alpha_{\mathrm{gran}}$ for  
$\mathcal{D}_{\mathrm{gran}}$, so reliability preferences impose overall larger margins than granularity preferences. This encourages reliability-first optimization while favoring finer-grained responses when reliability is satisfied.

The loss of $\mathcal{D}_{\mathrm{rel}}$ and $\mathcal{D}_{\mathrm{gran}}$ is computed as 
\vspace{-2mm}
\begin{equation}
\small
   \begin{aligned}
        \mathcal{L}_{\mathrm{rel}}
    =
    -
    \mathbb{E}_{(I,y_{+},y_{-})\sim \mathcal{D}_{\mathrm{rel}} }
    \log
    \sigma
    \left(
    \Delta_\theta - m_{\mathrm{rel}}
    \right),\\
     \mathcal{L}_{\mathrm{gran}}
    =
    -
    \mathbb{E}_{(I,y_{+},y_{-})\sim \mathcal{D}_{\mathrm{gran}} }
    \log
    \sigma
    \left(
    \Delta_\theta - m_{\mathrm{gran}}
    \right).
   \end{aligned}
    \label{eq:metric_margin_dpo}
\end{equation}
The final objective is $\mathcal{L}_{\mathrm{total}}=\mathcal{L}_{\mathrm{rel}}+\lambda\mathcal{L}_{\mathrm{gran}}$, where $\lambda$ is a hyper-parameter. 

\vspace{-2mm}
\section{Experiments}
\label{sec:experiments}
\vspace{-1mm}
We conduct experiments to answer three questions:
(1) How do existing MLLMs perform on reliability-prioritized fine-grained generation under different prompting styles?
(2) Can our method improve both reliability and supported granularity across prompting styles?
(3) Which components of our method contribute to the improvement?

\vspace{-2mm}
\subsection{Experimental Setup}
\label{sec:exp_setup}

\paragraph{Benchmark and metrics.}
We evaluate all models on \textsc{GranFact} using the granularity-neutral and granularity-aware metrics defined in Section~\ref{sec:flow_metrics}.
The former measures reliability at any valid hierarchy level, while the latter further accounts for the specificity of reliable descriptions.

\vspace{-2mm}
\paragraph{Models and prompts.}
We evaluate a diverse set of open-source and closed-source MLLMs, including InstructBLIP-Vicuna-7B~\cite{dai2023instructblip}, InternVL3.5-8B~\cite{zhu2025internvl3}, Kimi-K2.6~\cite{moonshot2026kimik26}, Qwen2.5-VL-7B~\cite{bai2025qwen25vl}, Qwen3-VL-8B-Instruct~\cite{bai2025qwen3vl}, GLM-4.6V~\cite{hong2025glmv}, GPT-5.4~\cite{openai2026gpt54}, and Gemini-3.1-Flash-Lite~\cite{googledeepmind2026gemini31flashlite}.
For each model, we evaluate three prompt styles: \textit{conservative}, \textit{neutral}, and \textit{aggressive}, which induce different levels of generation specificity.
Implementation details, model versions, decoding settings, and full prompts are provided in Appendix~\ref{sec:appendix_exp_details}.

\vspace{-2mm}
\paragraph{Training data.}
For training, we construct a small auxiliary training set using category labels from the \textsc{GranFact} annotations to retrieve and synthesize additional single-object images.
This results in 525 structured training images, all of which are disjoint from the \textsc{GranFact} evaluation set.
More details are provided in Appendix~\ref{sec:appendix_training_data}.
\vspace{-2mm}
\subsection{Main Results on \textsc{GranFact}}
\label{sec:main_results}
\vspace{-1mm}
\begin{table*}[t]
\centering
\small
\setlength{\tabcolsep}{3.0pt}
\begin{tabular}{llccccccccc}
\toprule
\multirow{2}{*}{Model}
& \multirow{2}{*}{Prompt Style}
& \multicolumn{4}{c}{Granularity-neutral Metrics}
& \multicolumn{5}{c}{Granularity-aware Metrics} \\
\cmidrule(lr){3-6}
\cmidrule(lr){7-11}
&
& $\mathrm{IR}$
& $\mathrm{P}$
& $\mathrm{R}$
& $\mathrm{F1}$
& $\mathrm{GIR}$
& $\mathrm{P}_{\mathrm{gran}}$
& $\mathrm{R}_{\mathrm{gran}}$
& $\mathrm{F1}_{\mathrm{gran}}$
& $\mathrm{G}_{\mathrm{avg}}$ \\
\midrule

\multicolumn{11}{l}{\textit{Comparable-Scale Open-Weight Models}} \\
\midrule

\multirow{3}{*}{InstructBLIP-Vicuna-7B}
& Aggressive   & 0.1962 & 0.3897 & 0.3079 & 0.3440 & 0.0904 & 0.1627 & 0.1286 & 0.1437 & 0.4176 \\
& Neutral      & 0.3959 & 0.6125 & 0.2397 & 0.3446 & 0.1352 & 0.2192 & 0.0858 & 0.1233 & 0.3578 \\
& Conservative & 0.3890 & 0.6211 & 0.2330 & 0.3388 & 0.1323 & 0.2281 & 0.0855 & 0.1244 & 0.3672 \\
\midrule

\multirow{3}{*}{InternVL3.5-8B}
& Aggressive   & 0.3133 & 0.4817 & 0.5595 & 0.5177 & 0.1632 & 0.2421 & 0.2812 & 0.2602 & 0.5026 \\
& Neutral      & 0.3873 & 0.6246 & 0.5257 & 0.5709 & 0.1919 & 0.3035 & 0.2554 & 0.2774 & 0.4858 \\
& Conservative & 0.4062 & 0.6330 & 0.5574 & 0.5928 & 0.1816 & 0.2894 & 0.2548 & 0.2710 & 0.4571 \\
\midrule

\multirow{3}{*}{Qwen2.5-VL-7B}
& Aggressive   & 0.3838 & 0.6258 & 0.5551 & 0.5883 & 0.2126 & 0.3416 & 0.3029 & 0.3211 & 0.5458 \\
& Neutral      & 0.4923 & 0.6699 & 0.5130 & 0.5811 & 0.2624 & 0.3529 & 0.2702 & 0.3061 & 0.5268 \\
& Conservative & 0.4647 & 0.6701 & 0.5130 & 0.5811 & 0.2229 & 0.3229 & 0.2472 & 0.2800 & 0.4819 \\
\midrule

\multirow{3}{*}{Qwen3-VL-8B-Instruct}
& Aggressive   & 0.3718 & 0.6591 & 0.6469 & 0.6529 & 0.2450 & 0.3981 & 0.3907 & 0.3944 & 0.6040 \\
& Neutral      & 0.4251 & 0.6865 & 0.6130 & 0.6477 & 0.2635 & 0.4047 & 0.3613 & 0.3818 & 0.5895 \\
& Conservative & 0.4664 & 0.7198 & 0.6511 & 0.6838 & 0.2573 & 0.3857 & 0.3488 & 0.3663 & 0.5358 \\
\midrule

\multirow{3}{*}{\textbf{Ours (Qwen3-VL-8B)}}
& Aggressive   & 0.4017 & 0.6719 & 0.6480 & 0.6597 & 0.2655 & 0.4183 & \textbf{0.4034} & \textbf{0.4107} & \textbf{0.6225} \\
& Neutral      & 0.4897 & 0.7027 & 0.6163 & 0.6567 & \textbf{0.2938} & \textbf{0.4217} & 0.3698 & 0.3941 & 0.6001 \\
& Conservative & \textbf{0.4940} & \textbf{0.7242}  & \textbf{0.6514} & \textbf{0.6859} & 0.2742 & 0.3935 & 0.3539 & 0.3726 & 0.5433 \\
\midrule

\multicolumn{11}{l}{\textit{Frontier / Large-Scale Reference Models}} \\
\midrule

\multirow{3}{*}{Kimi-K2.6}
& Aggressive   & 0.4527 & 0.6682 & 0.5337 & 0.5934 & 0.2202 & 0.4166 & 0.3328 & 0.3700 & \textbf{0.6235} \\
& Neutral      & 0.4647 & 0.6844 & 0.4954 & 0.5747 & 0.2332 & 0.4193 & 0.3035 & 0.3521 & 0.6127 \\
& Conservative & 0.4923 & 0.7126 & 0.5357 & 0.6116 & 0.2105 & 0.3786 & 0.2846 & 0.3249 & 0.5313 \\
\midrule

\multirow{3}{*}{GLM-4.6V}
& Aggressive   & 0.4057 & 0.6751 & 0.7075 & 0.6910 & 0.2533 & 0.3881 & 0.4067 & 0.3972 & 0.5748 \\
& Neutral      & 0.4603 & 0.7093 & 0.6671 & 0.6876 & 0.2717 & 0.3881 & 0.3649 & 0.3762 & 0.5471 \\
& Conservative & 0.4922 & 0.7530 & 0.7128 & 0.7323 & 0.2654 & 0.3803 & 0.3600 & 0.3699 & 0.5051 \\
\midrule

\multirow{3}{*}{GPT-5.4}
& Aggressive   & 0.2754 & 0.5894 & 0.7008 & 0.6403 & 0.1684 & 0.3478 & 0.4136 & 0.3779 & 0.5901 \\
& Neutral      & 0.4983 & 0.7467 & 0.6318 & 0.6845 & 0.2752 & 0.3936 & 0.3331 & 0.3608 & 0.5271 \\
& Conservative & 0.4879 & 0.7612 & 0.6777 & 0.7170 & 0.2369 & 0.3616 & 0.3219 & 0.3406 & 0.4750 \\
\midrule

\multirow{3}{*}{Gemini-3.1-Flash-Lite}
& Aggressive   & 0.4672 & 0.7510 & 0.7292 & 0.7400 & 0.3096 & 0.4595 & \textbf{0.4461} & \textbf{0.4527} & 0.6118 \\
& Neutral      & 0.4836 & 0.7571 & 0.7019 & 0.7284 & \textbf{0.3141} & \textbf{0.4622} & 0.4285 & 0.4447 & 0.6105 \\
& Conservative & \textbf{0.5112} & \textbf{0.7722} & \textbf{0.7424} & \textbf{0.7570} & 0.3019 & 0.4265 & 0.4101 & 0.4182 & 0.5524 \\

\bottomrule
\end{tabular}
\caption{
Performance on \textsc{GranFact} under different prompt styles.
Models are grouped into comparable-scale open-weight models and frontier/large-scale reference models.
}
\label{tab:model_benchmark}
\end{table*}



Table~\ref{tab:model_benchmark} summarizes the main results on \textsc{GranFact}. We highlight four key observations below.
\paragraph{Reliability--granularity trade-off is prevalent.}
Aggressive prompting generally increases granularity but reduces reliability across models.
For example, compared with the conservative prompt, the aggressive prompt increases $\mathrm{G}_{\mathrm{avg}}$ for Qwen3-VL-8B-Instruct from $0.5358$ to $0.6040$, while decreasing $\mathrm{IR}$ from $0.4664$ to $0.3718$ and $\mathrm{F1}$ from $0.6838$ to $0.6529$.
This shows that prompting models to be more specific does not necessarily make their fine-grained descriptions reliable.
\vspace{-2mm}
\paragraph{Image-level reliability remains difficult.}
Across models, $\mathrm{IR}$ is much lower than precision, indicating that image-level reliability remains challenging in open-ended visual description.
For example, Gemini-3.1-Flash-Lite reaches a precision of $0.7722$ under the conservative prompt, but its $\mathrm{IR}$ is only $0.5112$.
This suggests that a full response can still contain unreliable claims even when much of its generated content is grounded.
\vspace{-2mm}
\paragraph{Frontier models are stronger in reliability yet fragile in fine-grained generation.}
Frontier and large-scale reference models generally achieve stronger reliability than comparable-scale open-weight models, with Gemini-3.1-Flash-Lite obtaining the best $\mathrm{IR}$, $\mathrm{P}$, and $\mathrm{F1}$ under the conservative prompt.
However, this reliability advantage weakens when the model is pushed toward finer-grained generation: for Gemini-3.1-Flash-Lite, $\mathrm{G}_{\mathrm{avg}}$ increases from $0.5524$ to $0.6118$ under the aggressive prompt, while $\mathrm{IR}$ drops from $0.5112$ to $0.4672$.
This shows that even stronger MLLMs remain vulnerable when generating fine-grained descriptions.
\vspace{-2mm}
\paragraph{Our method improves fine-grained generation without sacrificing reliability.}
Compared with Qwen3-VL-8B-Instruct, the proposed RP-DPO  method improves both reliability and granularity under the same prompt styles.
Under the aggressive prompt, it improves $\mathrm{IR}$ from $0.3718$ to $0.4017$ and $\mathrm{F1}_{\mathrm{gran}}$ from $0.3944$ to $0.4107$, while also increasing $\mathrm{G}_{\mathrm{avg}}$ from $0.6040$ to $0.6225$.
This indicates that our method can improve fine-grained generation without making the model either overly conservative or prone to unreliable fine-grained claims. Moreover, our method achieves comparable performance to the largest-scale model on the $G_{\mathrm{avg}}$ metric, further demonstrating the superiority of our method.

\vspace{-2mm}
\subsection{Ablation Study}
\label{sec:ablation}
\vspace{-1mm}
We conduct ablations under the aggressive prompt setting and compare the full method with variants that remove RSR, reliability preferences, granularity preferences, or metric-margin optimization.

\begin{table}[ht]
\centering
\small
\setlength{\tabcolsep}{2.2pt}
\resizebox{\linewidth}{!}{
\begin{tabular}{l r@{\,}c r@{\,}c r@{\,}c r@{\,}c r@{\,}c}
\toprule
Method
& \multicolumn{2}{c}{$\mathrm{IR}$}
& \multicolumn{2}{c}{$\mathrm{F1}$}
& \multicolumn{2}{c}{$\mathrm{GIR}$}
& \multicolumn{2}{c}{$\mathrm{F1}_{\mathrm{gran}}$}
& \multicolumn{2}{c}{$\mathrm{G}_{\mathrm{avg}}$} \\
\midrule
Baseline
& 0.3718 & 
& 0.6529 & 
& 0.2450 & 
& 0.3944 & 
& 0.6040 &  \\
Full
& 0.4017 & \up
& 0.6597 & \up
& \textbf{0.2655} & \up
& \textbf{0.4107} & \up
& 0.6225 & \up \\
w/o RSR
& 0.3374 & \down
& 0.6389 & \down
& 0.2317 & \down
& 0.3963 & \up
& 0.6203 & \up \\
w/o Rel.
& 0.2845 & \down
& 0.6102 & \down
& 0.1938 & \down
& 0.3923 & \down
& 0.6429 & \up \\
w/o Gran.
& 0.4223 & \up
& 0.6691 & \up
& 0.2602 & \up
& 0.3870 & \down
& 0.5784 & \down \\
w/o Margin
& 0.3754 & \up
& 0.6337 & \down
& 0.2441 & \down
& 0.3881 & \down
& 0.6125 & \up \\
\bottomrule
\end{tabular}
}
\caption{
Ablation study under the aggressive prompt setting.
Arrows compare each result with the baseline.
Bold highlights the best results on $\mathrm{GIR}$ and $\mathrm{F1}_{\mathrm{gran}}$, which jointly account for reliability and granularity.
}
\label{tab:ablation}
\end{table}


As shown in Table~\ref{tab:ablation}, the reliability-prioritized components are crucial for preserving factual reliability.
Specifically, removing RSR, reliability preferences, or metric-margin lowers $\mathrm{F1}$ and $\mathrm{GIR}$, although these variants sometimes yield higher $\mathrm{G}_{\mathrm{avg}}$, indicating more aggressive but less reliable fine-grained descriptions.
In contrast, removing granularity preferences improves $\mathrm{IR}$ and $\mathrm{F1}$ but substantially lowers $\mathrm{F1}_{\mathrm{gran}}$ and $\mathrm{G}_{\mathrm{avg}}$, suggesting that the model becomes overly conservative.
Overall, the full method achieves the best $\mathrm{GIR}$ and $\mathrm{F1}_{\mathrm{gran}}$, showing that reliability-oriented components suppress unsupported specificity while granularity-oriented preferences prevent the model from collapsing to coarse descriptions.

\vspace{-2mm}
\section{Conclusion}
\label{sec:conclusion}
\vspace{-1mm}
In this paper, we study reliability-prioritized fine-grained generation in MLLMs, where models should generate the finest visual descriptions that remain reliably supported by the image. We introduce \textsc{GranFact}, a granularity-aware benchmark with expert-verified multi-object images with hierarchical coarse-to-fine annotations, together with a hierarchy-aware evaluation algorithm that measures both correctness and supported granularity. The proposed RP-DPO method prioritizes reliable predictions while encouraging finer-grained descriptions when reliability is ensured. Experiments on \textsc{GranFact} demonstrate that existing MLLMs still struggle to produce descriptions that are both reliable and fine-grained. Experimental results also demonstrate that the proposed RP-DPO method can improve fine-grained generation while preserving reliability.

\section*{Limitations}
\label{sec:limitations}

A limitation of our work is that the evaluation set is relatively small and covers a limited visual domains. Future work will expand the scale and diversity of \textsc{GranFact}.



\bibliography{custom}

\clearpage
\appendix

\section{Theoretical Analysis of Semantic Granularity}
\label{app:semantic_granularity}
\label{sec:appendix_theoretical_analysis}

\subsection{Semantic Hierarchy and Cut Sets}

Let $\mathscr{S}=(\mathcal{U},\mathcal{E})$ be a rooted semantic tree.
Each node $u\in\mathcal{U}$ represents a semantic category.
We write $u\preceq_{\mathscr{S}} v$ if $u=v$ or $u$ is a descendant of $v$,
or equivalently if $u$ is a refinement of $v$.
Let $\Lambda$ denote the set of leaf nodes.
For each node $u$, let $\Lambda(u)\subseteq\Lambda$ be the set of leaves
under $u$.
Thus,
\begin{equation}
    u\preceq_{\mathscr{S}} v
    \quad \Rightarrow \quad
    \Lambda(u)\subseteq \Lambda(v).
\end{equation}

A semantic granularity is represented by a cut set $\Omega\subseteq\mathcal{U}$.
A valid cut set satisfies
\begin{equation}
\label{eq:appendix_cut_set}
\small
\begin{aligned}
    \bigcup_{u\in\Omega}\Lambda(u) &= \Lambda, \\
    \Lambda(u)\cap\Lambda(v) &= \emptyset,
    \quad \forall u\neq v\in\Omega .
\end{aligned}
\end{equation}
That is, nodes in $\Omega$ form a partition of the leaf-level semantic space.

For two cut sets $\Omega_f$ and $\Omega_c$, we say that $\Omega_f$ is
finer than $\Omega_c$, denoted by $\Omega_f\sqsubseteq\Omega_c$, if
for every $u\in\Omega_f$, there exists $v\in\Omega_c$ such that
$u\preceq_{\mathscr{S}} v$.
Because $\Omega_c$ is a cut set, this node $v$ is unique.
Otherwise, two distinct nodes in $\Omega_c$ would have overlapping leaf
sets, contradicting Eq.~(\ref{eq:appendix_cut_set}).
We therefore define the coarsening map
\begin{equation}
\label{eq:appendix_coarsening_map}
\small
\kappa_{f\to c}(u)=v,
\quad
u\in\Omega_f,\ v\in\Omega_c,\ u\preceq_{\mathscr{S}} v .
\end{equation}

\subsection{Semantic Correctness and Bayes Risk}

We consider a task distribution over image--semantic pairs
$(I,\xi^\star)$, where $I\in\mathsf{Img}$ is an image and
$\xi^\star\in\Lambda$ is the finest true semantic state.
Let
\begin{equation}
    T=\phi(I)\in\mathsf{Tr}
\end{equation}
denote the visual representation obtained from the model's visual encoder.
In the analysis below, the model makes semantic predictions based on the
same visual representation $T$, while the target granularity is specified
by the cut set $\Omega$.
A prediction $u\in\Omega$ is semantically correct if its semantic scope
covers the true leaf:
\begin{equation}
    \xi^\star \preceq_{\mathscr{S}} u .
\end{equation}
Equivalently, $\xi^\star\in\Lambda(u)$.

For a cut set $\Omega$, the model can be modeled as a function
\begin{equation}
    h_{\Omega}:\mathsf{Tr}\to\Omega .
\end{equation}
Its semantic risk is defined as
\begin{equation}
\label{eq:appendix_semantic_risk}
\small
\varepsilon(h_{\Omega})
=
\Pr\!\left[
\xi^\star\npreceq_{\mathscr{S}} h_{\Omega}(T)
\right],
\end{equation}
where the probability is taken over the task distribution and the induced
visual representation $T=\phi(I)$.
The Bayes-optimal semantic risk under cut $\Omega$ is
\begin{equation}
\label{eq:appendix_bayes_risk}
\small
\varepsilon^\star(\Omega)
=
\inf_{h_{\Omega}:\mathsf{Tr}\to\Omega}
\varepsilon(h_{\Omega}).
\end{equation}

We also define the cut-induced semantic label associated with $\Omega$.
For each cut set $\Omega$, let $Y_\Omega\in\Omega$ be the unique node
satisfying
\begin{equation}
\label{eq:appendix_cut_induced_label}
    \xi^\star\preceq_{\mathscr{S}}Y_\Omega .
\end{equation}
The uniqueness follows from the fact that $\Omega$ forms a partition of
the leaf-level semantic space.
For two cut sets $\Omega_f\sqsubseteq\Omega_c$, we write
\begin{equation}
\label{eq:appendix_yf_yc}
    Y_f=Y_{\Omega_f},
    \qquad
    Y_c=Y_{\Omega_c}.
\end{equation}
By definition of the coarsening map, we have
\begin{equation}
\label{eq:appendix_yc_kappa_yf}
    Y_c=\kappa_{f\to c}(Y_f).
\end{equation}

\subsection{Proof of Monotonicity}

We restate the proposition below.

\begin{proposition}[Monotonicity of Bayes-Optimal Semantic Risk]
\label{prop:appendix_semantic_risk_monotonicity}
Let $\Omega_f$ and $\Omega_c$ be two cut sets of $\mathscr{S}$, where
$\Omega_f\sqsubseteq\Omega_c$.
Then
\begin{equation}
\small
    \varepsilon^\star(\Omega_c)
    \le
    \varepsilon^\star(\Omega_f).
\end{equation}
\end{proposition}

\begin{proof}
Take any predictor $h_f:\mathsf{Tr}\to\Omega_f$.
Using the coarsening map in Eq.~(\ref{eq:appendix_coarsening_map}), we construct a coarse-grained predictor
\begin{equation}
\label{eq:appendix_coarse_predictor}
\small
    h_c(T)=\kappa_{f\to c}(h_f(T)).
\end{equation}
By definition of $\kappa_{f\to c}$,
\begin{equation}
    h_f(T)\preceq_{\mathscr{S}} h_c(T).
\end{equation}
Therefore, for every image--semantic pair $(I,\xi^\star)$ and its visual
representation $T=\phi(I)$,
\begin{equation}
\small
\xi^\star\preceq_{\mathscr{S}} h_f(T)
\quad\Rightarrow\quad
\xi^\star\preceq_{\mathscr{S}} h_c(T).
\end{equation}
In words, whenever the fine-grained prediction is correct, its coarsened
prediction is also correct.
Equivalently, the error event of $h_c$ is contained in the error event
of $h_f$:
\begin{equation}
\label{eq:appendix_error_inclusion}
\small
\left\{
\xi^\star\npreceq_{\mathscr{S}} h_c(T)
\right\}
\subseteq
\left\{
\xi^\star\npreceq_{\mathscr{S}} h_f(T)
\right\}.
\end{equation}
Taking probabilities gives
\begin{equation}
    \varepsilon(h_c)\le \varepsilon(h_f).
\end{equation}
Since $h_c$ is a valid predictor on $\Omega_c$, we have
\begin{equation}
\small
\varepsilon^\star(\Omega_c)
\le
\varepsilon(h_c)
\le
\varepsilon(h_f).
\end{equation}
This holds for any predictor $h_f:\mathsf{Tr}\to\Omega_f$.
Taking the infimum over all such $h_f$ yields
\begin{equation}
    \varepsilon^\star(\Omega_c)
    \le
    \varepsilon^\star(\Omega_f).
\end{equation}
\end{proof}

\subsection{Information-Theoretic Refinement Analysis}

The proposition above establishes a monotonicity result under semantic
0-1 risk.
We further quantify the amount of additional information required by
semantic refinement.

For $\Omega_f\sqsubseteq\Omega_c$, we define the coarse-to-fine
refinement information deficit as
\begin{equation}
\label{eq:appendix_refinement_deficit}
\small
\eta_{f\mid c}
=
H(Y_f\mid T,Y_c).
\end{equation}
Equivalently, using the definition of conditional mutual information,
\begin{equation}
\label{eq:appendix_refinement_deficit_mi}
\small
\eta_{f\mid c}
=
H(Y_f\mid Y_c)
-
\operatorname{MI}(Y_f;T\mid Y_c).
\end{equation}
Here, $H(Y_f\mid Y_c)$ measures the semantic refinement entropy required
to move from the coarse cut to the fine cut, while
$\operatorname{MI}(Y_f;T\mid Y_c)$ measures the refinement information
provided by the visual representation $T$.
Thus, $\eta_{f\mid c}$ captures the part of the fine-grained semantic
information not resolved by the model's visual evidence.

\begin{proposition}[Additivity and Monotonicity of the Refinement Information Deficit]
\label{prop:appendix_refinement_deficit_monotonicity}
Let $\Omega_{f_2}$, $\Omega_{f_1}$, and $\Omega_c$ be three nested cut
sets satisfying
$\Omega_{f_2}\sqsubseteq\Omega_{f_1}\sqsubseteq\Omega_c$.
Then
\begin{equation}
\label{eq:appendix_refinement_deficit_additivity}
\small
\eta_{f_2\mid c}
=
\eta_{f_2\mid f_1}
+
\eta_{f_1\mid c}
\ge
\eta_{f_1\mid c}.
\end{equation}
Therefore, the refinement information deficit is non-decreasing as the
target granularity becomes finer and can be used to quantify the degree
of semantic refinement.
\end{proposition}

\begin{proof}
Since $Y_{f_1}$ is a deterministic coarsening of $Y_{f_2}$ and $Y_c$ is
a deterministic coarsening of $Y_{f_1}$, the chain rule of conditional
entropy gives
\begin{equation}
\small
\begin{aligned}
H(Y_{f_2}\mid T,Y_c)
&=
H(Y_{f_1}\mid T,Y_c)
+
H(Y_{f_2}\mid T,Y_{f_1},Y_c)
\\
&=
H(Y_{f_1}\mid T,Y_c)
+
H(Y_{f_2}\mid T,Y_{f_1}).
\end{aligned}
\end{equation}
This is exactly
$\eta_{f_2\mid c}=\eta_{f_1\mid c}+\eta_{f_2\mid f_1}$.
The inequality follows from the non-negativity of conditional entropy.
\end{proof}

We next define a log-score uncertainty risk.
For a predictive distribution $\widehat P_\Omega(\cdot\mid T)$ over
a cut set $\Omega$, let
\begin{equation}
\label{eq:appendix_log_score_risk}
\small
\mathcal R_{\mathrm{ls}}(\widehat P_\Omega)
=
\mathbb E\left[
-\log \widehat P_\Omega(Y_\Omega\mid T)
\right].
\end{equation}
The Bayes-optimal log-score uncertainty risk under cut $\Omega$ is
\begin{equation}
\label{eq:appendix_log_score_bayes}
\small
\mathcal R_{\mathrm{ls}}^\star(\Omega)
=
\inf_{\widehat P_\Omega}
\mathcal R_{\mathrm{ls}}(\widehat P_\Omega).
\end{equation}

\begin{proposition}[Bayes-Optimal Uncertainty-Risk Gap]
\label{prop:appendix_bayes_uncertainty_gap}
Let $\Omega_f$ and $\Omega_c$ be two cut sets of $\mathscr S$, where
$\Omega_f\sqsubseteq\Omega_c$.
Then
\begin{equation}
\label{eq:appendix_uncertainty_gap}
\small
\mathcal R_{\mathrm{ls}}^\star(\Omega_f)
-
\mathcal R_{\mathrm{ls}}^\star(\Omega_c)
=
\eta_{f\mid c}.
\end{equation}
\end{proposition}

\begin{proof}
The logarithmic scoring rule is proper, so the infimum in
Eq.~(\ref{eq:appendix_log_score_bayes}) is achieved by the true posterior
distribution $P(Y_\Omega\mid T)$.
Therefore,
\begin{equation}
\label{eq:appendix_log_score_entropy}
\small
\mathcal R_{\mathrm{ls}}^\star(\Omega)
=
H(Y_\Omega\mid T).
\end{equation}
Applying this to $\Omega_f$ and $\Omega_c$ gives
\begin{equation}
\small
\mathcal R_{\mathrm{ls}}^\star(\Omega_f)
-
\mathcal R_{\mathrm{ls}}^\star(\Omega_c)
=
H(Y_f\mid T)-H(Y_c\mid T).
\end{equation}
Since $Y_c=\kappa_{f\to c}(Y_f)$, $Y_c$ is a deterministic coarsening of
$Y_f$.
By the chain rule of conditional entropy,
\begin{equation}
\small
H(Y_f\mid T)
=
H(Y_c\mid T)+H(Y_f\mid T,Y_c).
\end{equation}
Thus,
\begin{equation}
\small
\mathcal R_{\mathrm{ls}}^\star(\Omega_f)
-
\mathcal R_{\mathrm{ls}}^\star(\Omega_c)
=
H(Y_f\mid T,Y_c)
=
\eta_{f\mid c}.
\end{equation}
\end{proof}

This proposition gives a quantitative characterization of the
coarse-to-fine trade-off under log-score uncertainty risk.
The intrinsic cost of moving from $\Omega_c$ to $\Omega_f$ is exactly the
refinement information deficit: the semantic information needed for
fine-grained refinement that is not resolved by the visual representation
available to the model.

\subsection{Actual Model Decomposition}

The proposition above characterizes the Bayes-optimal uncertainty-risk
gap.
An actual model, however, may deviate from the Bayes posterior at each
granularity.
We therefore analyze how the model's predictive distributions affect the
coarse-to-fine risk gap.

Let $\widehat P_f(\cdot\mid T)$ and $\widehat P_c(\cdot\mid T)$ be the
model's predictive distributions over $\Omega_f$ and $\Omega_c$,
respectively.
We define the coarse distribution induced by the fine-grained prediction
as
\begin{equation}
\label{eq:appendix_coarsened_fine_distribution}
\small
\widehat P_{f\to c}(y_c\mid T)
=
\sum_{y_f\in\Omega_f:\kappa_{f\to c}(y_f)=y_c}
\widehat P_f(y_f\mid T),
\qquad
y_c\in\Omega_c .
\end{equation}
This distribution corresponds to first predicting at the fine granularity
and then discarding the refinement details through the coarsening map.

For notational convenience, define the log-score risk of the induced
coarse distribution as
\begin{equation}
\label{eq:appendix_coarsened_fine_risk}
\small
\mathcal R_{\mathrm{ls}}(\widehat P_{f\to c})
=
\mathbb E\left[
-\log \widehat P_{f\to c}(Y_c\mid T)
\right].
\end{equation}
If the model were strictly hierarchy-consistent, the directly elicited
coarse distribution $\widehat P_c(\cdot\mid T)$ would coincide with the
induced distribution $\widehat P_{f\to c}(\cdot\mid T)$.
We do not require this equality.
Instead, we define the signed hierarchy-readout violation margin as
\begin{equation}
\label{eq:appendix_hierarchy_violation_margin}
\small
\mu_{f\mid c}
=
\mathcal R_{\mathrm{ls}}(\widehat P_c)
-
\mathcal R_{\mathrm{ls}}(\widehat P_{f\to c}).
\end{equation}
A positive $\mu_{f\mid c}$ means that directly eliciting the coarse
prediction yields higher log-score risk than coarsening the model's own
fine-grained prediction.

\begin{assumption}[Average Hierarchy-Readout Stability]
\label{assumption:appendix_average_hierarchy_stability}
We assume average hierarchy-readout stability over the evaluated
comparable granularity pairs:
\begin{equation}
\label{eq:appendix_average_hierarchy_stability}
\small
\mathbb E_{(f,c)}\!\left[\mu_{f\mid c}\right]\le 0,
\end{equation}
where $\mathbb E_{(f,c)}$ denotes the average over the evaluated pairs
$(\Omega_f,\Omega_c)$ with $\Omega_f\sqsubseteq\Omega_c$.
This condition allows individual hierarchy violations, but requires that
directly eliciting coarse-grained predictions is not worse on average
than eliciting fine-grained predictions and then coarsening them.
\end{assumption}

\begin{theorem}[Uncertainty-Risk Gap for an Actual Model]
\label{thm:appendix_actual_uncertainty_gap}
Under Assumption~\ref{assumption:appendix_average_hierarchy_stability},
\begin{equation}
\label{eq:appendix_expected_actual_gap}
\small
\mathbb E_{(f,c)}
\left[
\mathcal R_{\mathrm{ls}}(\widehat P_f)
-
\mathcal R_{\mathrm{ls}}(\widehat P_c)
\right]
\ge
\mathbb E_{(f,c)}
\left[
\eta_{f\mid c}
\right].
\end{equation}
Therefore, if
$\mathbb E_{(f,c)}[\eta_{f\mid c}]>0$, the actual fine-grained prediction
has larger log-score uncertainty risk than the coarse-grained prediction
on average.
\end{theorem}

\begin{proof}
We first insert the induced coarse distribution $\widehat P_{f\to c}$:
\begin{equation}
\label{eq:appendix_actual_gap_insert}
\small
\begin{aligned}
&
\mathcal R_{\mathrm{ls}}(\widehat P_f)
-
\mathcal R_{\mathrm{ls}}(\widehat P_c)
\\
=&\
\left[
\mathcal R_{\mathrm{ls}}(\widehat P_f)
-
\mathcal R_{\mathrm{ls}}(\widehat P_{f\to c})
\right]
-
\left[
\mathcal R_{\mathrm{ls}}(\widehat P_c)
-
\mathcal R_{\mathrm{ls}}(\widehat P_{f\to c})
\right]
\\
=&\
\left[
\mathcal R_{\mathrm{ls}}(\widehat P_f)
-
\mathcal R_{\mathrm{ls}}(\widehat P_{f\to c})
\right]
-
\mu_{f\mid c}.
\end{aligned}
\end{equation}

By the cross-entropy decomposition,
\begin{equation}
\label{eq:appendix_fine_to_coarsened_fine_gap}
\small
\begin{aligned}
&
\mathcal R_{\mathrm{ls}}(\widehat P_f)
-
\mathcal R_{\mathrm{ls}}(\widehat P_{f\to c})
\\
=&\
H(Y_f\mid T)-H(Y_c\mid T)
\\
&+
\mathbb E_T
D_{\mathrm{KL}}
\!\left(
P(Y_f\mid T)
\Vert
\widehat P_f(\cdot\mid T)
\right)
\\
&-
\mathbb E_T
D_{\mathrm{KL}}
\!\left(
P(Y_c\mid T)
\Vert
\widehat P_{f\to c}(\cdot\mid T)
\right)
\\
=&\
\eta_{f\mid c}
+
\lambda_{f\mid c},
\end{aligned}
\end{equation}
where
\begin{equation}
\label{eq:appendix_contraction_gain}
\small
\begin{aligned}
\lambda_{f\mid c}
=&\
\mathbb E_T
D_{\mathrm{KL}}
\!\left(
P(Y_f\mid T)
\Vert
\widehat P_f(\cdot\mid T)
\right)
\\
&-
\mathbb E_T
D_{\mathrm{KL}}
\!\left(
P(Y_c\mid T)
\Vert
\widehat P_{f\to c}(\cdot\mid T)
\right).
\end{aligned}
\end{equation}
Since $Y_c=\kappa_{f\to c}(Y_f)$ and
$\widehat P_{f\to c}$ is obtained by applying the same coarsening map to
$\widehat P_f$, the contraction property of KL divergence under
coarsening gives
\begin{equation}
\label{eq:appendix_lambda_nonnegative}
\small
\lambda_{f\mid c}\ge 0 .
\end{equation}
Combining Eq.~(\ref{eq:appendix_actual_gap_insert}) and
Eq.~(\ref{eq:appendix_fine_to_coarsened_fine_gap}), we obtain
\begin{equation}
\label{eq:appendix_actual_gap_mu}
\small
\mathcal R_{\mathrm{ls}}(\widehat P_f)
-
\mathcal R_{\mathrm{ls}}(\widehat P_c)
=
\eta_{f\mid c}
+
\lambda_{f\mid c}
-
\mu_{f\mid c}.
\end{equation}
Taking the average over the evaluated comparable granularity pairs gives
\begin{equation}
\small
\begin{aligned}
&
\mathbb E_{(f,c)}
\left[
\mathcal R_{\mathrm{ls}}(\widehat P_f)
-
\mathcal R_{\mathrm{ls}}(\widehat P_c)
\right]
\\
=&\
\mathbb E_{(f,c)}
\left[
\eta_{f\mid c}
+
\lambda_{f\mid c}
-
\mu_{f\mid c}
\right].
\end{aligned}
\end{equation}
Using $\lambda_{f\mid c}\ge0$ and the average hierarchy-readout stability
assumption in Eq.~(\ref{eq:appendix_average_hierarchy_stability}), we
have
\begin{equation}
\small
\mathbb E_{(f,c)}
\left[
\mathcal R_{\mathrm{ls}}(\widehat P_f)
-
\mathcal R_{\mathrm{ls}}(\widehat P_c)
\right]
\ge
\mathbb E_{(f,c)}
\left[
\eta_{f\mid c}
\right].
\end{equation}
Thus, under the stated assumption, a positive average refinement
information deficit implies a positive average coarse-to-fine
uncertainty-risk gap.
\end{proof}

\subsection{Implications}

The semantic-risk result shows that moving to a finer cut cannot decrease
the Bayes-optimal 0-1 risk. The information-theoretic analysis further
shows that $\eta_{f\mid c}$ is additive and non-decreasing along nested
semantic refinements, and is exactly equal to the Bayes-optimal
coarse-to-fine log-score risk gap. Therefore, $\eta_{f\mid c}$ quantifies
the unresolved uncertainty introduced by semantic refinement.

For an actual model, the coarse-to-fine risk gap decomposes as
$\eta_{f\mid c}+\lambda_{f\mid c}-\mu_{f\mid c}$.
Since $\lambda_{f\mid c}\ge0$, average hierarchy-readout stability implies
that the expected actual risk gap is lower bounded by the expected
refinement information deficit.

These results motivate reliability-prioritized fine-grained generation:
models should provide additional specificity only when the visual evidence
supports the corresponding refinement, and otherwise back off to a
reliable coarser description.
The empirical analysis in Section~\ref{sec:analysis_prompt_tradeoff}
complements this result by showing that larger granularity gains are more
frequently accompanied by reliability drops in actual MLLM generations.

The analysis above concerns a single hierarchical category claim.
\textsc{GranFact} extends the same reliability-prioritized principle to
open-ended descriptions involving multiple objects, quantities,
attributes, omissions, and hallucinated entities.

\section{Details on Benchmark Construction}
\label{sec:appendix_benchmark_construction}

\subsection{Domain Selection Rationale}
\label{sec:appendix_domain_selection}

The seven domains in \textsc{GranFact} are selected to balance category
coverage, the availability of meaningful semantic hierarchies, and the
visual grounding of fine-grained distinctions.

First, the selected domains provide broad coverage of both natural and
human-made visual concepts. Animals and plants are retained as
representative fine-grained natural categories, while daily objects, cars,
electronics, games, and landmarks extend the benchmark to a wider range of
real-world scenarios. In particular, cars, electronics, games, and
landmarks complement existing fine-grained datasets and evaluation
resources, which commonly place greater emphasis on biological categories.

Second, these domains support semantically meaningful coarse-to-fine
category paths. For example, a car can be progressively described from a
general vehicle category to its brand, series, and specific model. An
electronic device can similarly be described from a general device
category to its product type, brand, product line, and specific model.
Such hierarchical structures enable the benchmark to evaluate whether a
model produces a reliable description at an appropriate level of semantic
granularity.

Third, fine-grained distinctions in these domains can often be supported
by visible evidence. Relevant cues include object shape, logos, component
layouts, color, material, and structural design. These cues are
recognizable to human annotators but remain challenging for MLLMs,
particularly in realistic multi-object scenes. To avoid relying on hidden
metadata or forcing overly specific labels, annotators assign each object
the finest category level that is reliably supported by the visible
evidence. This annotation principle is consistent with the benchmark goal
of rewarding fine-grained descriptions only when their specificity is
visually justified.

\subsection{Comparison with Existing Benchmarks}
\label{sec:appendix_benchmark_comparison}

\begin{table*}[t]
\centering
\small
\resizebox{\textwidth}{!}{%
\begin{tabular}{lrrl}
\toprule
Benchmark
& \# Images
& \# Object-level Units
& Annotation Granularity \\
\midrule
MMHal-Bench~\cite{sun2024aligning}
& 96
& 96
& Object-focused evaluation instances \\
POPE~\cite{li2023evaluating}
& 500
& 3,000
& Object-existence questions \\
FINER-CompreCap~\cite{xiao2026finer}
& 560
& 3,505
& 	Objects and attributes \\
\textsc{GranFact}
& \textbf{581}
& \textbf{5,093}
& \textbf{Objects, attributes, quantity, and coarse-to-fine categories} \\
\bottomrule
\end{tabular}%
}
\caption{
Comparison with representative hallucination and fine-grained MLLM
benchmarks. For question-based benchmarks, object-level units refer to
object-related questions; for densely annotated benchmarks, they refer
to annotated objects.
}
\label{tab:appendix_benchmark_comparison}
\end{table*}

Although \textsc{GranFact} contains 581 images, its image-level scale is
comparable to that of representative hallucination and fine-grained MLLM
benchmarks. Moreover, its images contain 5,093 annotated objects in total.
Among them, 2,302 objects are equipped with coarse-to-fine category paths
and serve as the primary targets of the granularity-aware evaluation. The
remaining single-granularity objects serve as a whitelist, preventing
predictions of visually present non-target objects from being incorrectly
treated as hallucinations.

The numbers in the third column provide an approximate object-level scale
comparison rather than strictly identical evaluation units. In
question-based benchmarks such as POPE, an evaluation unit corresponds to
an object-related question, whereas in densely annotated benchmarks such
as \textsc{GranFact}, it corresponds to an annotated object.

Beyond its object-level coverage, \textsc{GranFact} provides structured
coarse-to-fine category paths, quantity annotations, and visually
identifiable attributes for its multi-granularity evaluation objects.
Consequently, its scale and annotation effort are not fully reflected by
the number of images alone. The benchmark represents a balance among image
coverage, object density, annotation granularity, and annotation quality.

\subsection{Dataset Statistics}
\label{sec:appendix_dataset_statistics}

This subsection provides additional statistics of the \textsc{GranFact}
evaluation set, supplementing the summary visualizations in
Figure~\ref{fig:dataset_statistics}. The benchmark contains 581 images
from seven visual domains and 5,093 annotated objects in total. Among
them, 2,302 objects have coarse-to-fine category annotations and serve as
the primary targets of our granularity-aware evaluation. The remaining
objects have single-granularity annotations and are used as whitelist
objects, ensuring that predictions referring to these visually present
objects are not incorrectly treated as hallucinations.
Table~\ref{tab:appendix_domain_distribution} reports the exact number of
images in each domain.

\begin{table}[h]
\centering
\small
\begin{tabular}{lrr}
\toprule
Domain & \# Images & Percentage \\
\midrule
Daily objects & 183 & 31.5\% \\
Animals       & 100 & 17.2\% \\
Cars          & 80  & 13.8\% \\
Electronics   & 75  & 12.9\% \\
Landmarks     & 75  & 12.9\% \\
Games         & 43  & 7.4\%  \\
Plants        & 25  & 4.3\%  \\
\midrule
Total         & 581 & 100.0\% \\
\bottomrule
\end{tabular}
\caption{
Domain distribution of the \textsc{GranFact} evaluation set.
}
\label{tab:appendix_domain_distribution}
\end{table}

We further summarize the statistics of the multi-granularity objects in
Table~\ref{tab:appendix_entity_statistics}. Here, each multi-granularity
object corresponds to a GT entry $g_j$ in the structured annotation of an
image. These statistics describe the number of multi-granularity objects
per image and provide a more detailed view of the multi-object structure
of the evaluation set. Single-granularity whitelist objects are not
included in these statistics.

\begin{table}[h]
\centering
\small
\begin{tabular}{lr}
\toprule
Statistic & Value \\
\midrule
Total multi-granularity objects & 2302 \\
Average multi-granularity objects per image & 3.96 \\
Median multi-granularity objects per image & 4 \\
Minimum multi-granularity objects per image & 1 \\
Maximum multi-granularity objects per image & 14 \\
Images with 1 multi-granularity object & 124 \\
Images with 2 multi-granularity objects & 116 \\
Images with 3 or more multi-granularity objects & 341 \\
\bottomrule
\end{tabular}
\caption{
Statistics of multi-granularity objects in \textsc{GranFact}.
}
\label{tab:appendix_entity_statistics}
\end{table}

Table~\ref{tab:appendix_depth_distribution} reports the distribution of
category-path depths. For each multi-granularity object $g_j$, the depth
is defined as the length $L_j$ of its coarse-to-fine category path
$\mathcal{H}_j=(c_{j,1},\ldots,c_{j,L_j})$. A larger depth indicates that
the object is annotated with a more detailed semantic hierarchy.

\begin{table}[h]
\centering
\small
\begin{tabular}{lrr}
\toprule
Category-path depth $L_j$ & \# Objects & Percentage \\
\midrule
2 & 7 & 0.3\% \\
3 & 175 & 7.6\% \\
4 & 522 & 22.7\% \\
5 & 880 & 38.2\% \\
6 & 581 & 25.2\% \\
7 & 123 & 5.3\% \\
8 & 12 & 0.5\% \\
9 & 2 & 0.1\% \\
\midrule
Total & 2302 & 100.0\% \\
\bottomrule
\end{tabular}
\caption{
Distribution of category-path depths over multi-granularity objects.
}
\label{tab:appendix_depth_distribution}
\end{table}

In addition to category labels, \textsc{GranFact} includes visual
attributes for multi-granularity objects when such attributes are visually
identifiable and useful for evaluation.
Table~\ref{tab:appendix_attribute_statistics} summarizes the attribute
annotations. Single-granularity whitelist objects are not included in
these attribute statistics.

\begin{table}[h]
\centering
\small
\begin{tabular}{lr}
\toprule
Statistic & Value \\
\midrule
Total attribute annotations & 13144 \\
Average attributes per annotated entity & 5.71 \\
\bottomrule
\end{tabular}
\caption{
Attribute annotation statistics of \textsc{GranFact}.
}
\label{tab:appendix_attribute_statistics}
\end{table}

\subsection{Image Collection, Filtering, and Annotation Protocol}
\label{sec:appendix_image_collection}
\label{sec:appendix_annotation_protocol}

Since the goal of \textsc{GranFact} is to evaluate
reliability-prioritized fine-grained generation, we collected and
annotated images with an emphasis on fine-grained visual understanding.
We recruited 10 domain-aware annotators to participate in the collection
and annotation process. The annotators were familiar with at least one of
the covered visual domains, such as daily objects, animals, plants,
cars, electronics, landmarks, or games. The process was conducted
iteratively: candidate images were first collected according to the
benchmark goal, then preliminarily annotated with entity-level coarse-to-fine categories and visual attributes, and finally reviewed and filtered before
inclusion in the final evaluation set.

During collection, annotators followed several practical guidelines to
ensure that the images were suitable for evaluating fine-grained visual
description. First, each retained image should contain at least one
visually identifiable entity that can be associated with a
coarse-to-fine category path. Second, when available, annotators
preferred images with diverse entity compositions, visually related
entities, or realistic visual contexts, since such cases provide useful
open-ended evaluation settings. Images centered on a single entity were
also retained when the entity could be annotated at a sufficiently
fine-grained level with reasonable confidence. Third, for each annotated
entity, its finest category label should be reasonably supported by the
image and, when needed, by common visual knowledge or reference
information available to the annotators. Images were excluded when the
intended fine-grained identities of key entities were too uncertain to
annotate reliably.

The initial collection contained approximately 700 candidate
image--annotation pairs. Each candidate pair consisted of an image and
preliminary structured annotations. For each annotated entity, the
preliminary annotation included a coarse-to-fine category path and a set
of visual attributes. The category path records the semantic granularity
of the entity from a coarse category to the finest category that could be
reasonably identified. The attributes describe visible non-quantitative
properties such as color, material, posture, spatial position, shape, or
other domain-specific visual cues when applicable.

We manually reviewed the candidate samples and revised or removed samples
with clear annotation issues. The review focused on whether the annotated
entities were visually present, whether their  coarse-to-fine category paths were
appropriate, whether the finest category labels were sufficiently
supported, and whether the visual attributes were consistent with the
image.

We filtered out candidate samples according to the following criteria:
\begin{itemize}
    \item \textbf{Low visual quality:} images that were blurry,
    low-resolution, over-compressed, or otherwise difficult to inspect.
    \item \textbf{Insufficient visual evidence:} images where the
    intended fine-grained category of a key entity could not be
    determined with reasonable confidence.
    \item \textbf{Ambiguous labels:} images where multiple fine-grained
    labels appeared plausible for a key entity or where a stable
    category path was difficult to define.
    \item \textbf{Severe occlusion or truncation:} images where key
    entities were heavily occluded, cropped, or too small to support
    reliable annotation.
    \item \textbf{Duplicate or near-duplicate samples:} visually
    redundant images that contributed little additional diversity.
    \item \textbf{Unsuitable scenes:} images that did not contain
    identifiable fine-grained entities or were not well aligned with the
    intended open-ended visual description setting.
\end{itemize}

When inconsistencies were identified, the annotations were revised when
a reliable correction was clear; otherwise, the sample was removed from
the benchmark. After this filtering and review process, we retained 581
image--annotation pairs as the final \textsc{GranFact} evaluation set.
This process aims to improve annotation reliability while preserving the
fine-grained and visually diverse characteristics needed for evaluating
open-ended model responses.

\subsection{Annotator Instructions, Consent, and Ethical Considerations}
\label{sec:appendix_annotator_details}

This subsection provides additional details about the human annotation
process used for constructing \textsc{GranFact}. The annotation task was
limited to image collection, image filtering, and object-level visual
annotation for research purposes. It did not involve interventions with
human subjects or the collection of sensitive personal information from
annotators.

\paragraph{Annotator recruitment.}
We recruited 10 domain-aware annotators to participate in image
collection, filtering, and annotation. Annotators were selected based on
their familiarity with at least one of the covered visual domains,
including daily objects, animals, plants, cars, electronics, landmarks,
and games. We did not use an external crowdsourcing platform. The
annotators were instructed to work on domains for which they had
sufficient familiarity, and uncertain cases were reviewed or filtered out
during the quality-control stage.

\paragraph{Annotation instructions.}
Annotators were given the following instructions.

\begin{quote}
You will help construct a research benchmark for evaluating
fine-grained visual description in multimodal large language models.
For each candidate image, first determine whether the image is suitable
for fine-grained visual description evaluation. A suitable image should
contain at least one visually identifiable entity that can be associated
with a coarse-to-fine category path. Prefer images with diverse object
compositions, visually related entities, or realistic visual contexts.
Images centered on a single object may also be retained if the object can
be annotated at a sufficiently fine-grained level with reasonable
confidence.

For each retained image, annotate the visible target entities. For each
entity, provide: (1) a coarse-to-fine category path, from a broad
category to the finest category that can be reasonably supported; (2) the
quantity of visible instances; and (3) visible non-quantitative
attributes, such as color, material, shape, pose, spatial position, or
other domain-specific visual cues when applicable.

Do not guess fine-grained identities that are not supported by the image.
If the finest category of an entity cannot be determined with reasonable
confidence, annotate the entity only up to the finest reliable level, or
mark the sample for review. When needed, use common visual knowledge or
reference information to verify fine-grained categories, but do not add
labels that remain visually ambiguous. Coarse but reliable annotations
are preferred over unsupported fine-grained labels.

Exclude candidate images if they are blurry, low-resolution,
over-compressed, severely occluded, heavily cropped, or otherwise
difficult to inspect. Also exclude images where the intended
fine-grained category is too uncertain, where multiple fine-grained
labels are plausible, where a stable  coarse-to-fine category path cannot be defined, or
where the scene is not suitable for open-ended visual description
evaluation. Duplicate or near-duplicate images should be removed.

The collected annotations will be used for constructing a research
benchmark and for reporting aggregate statistics and examples in the
paper. The task focuses on object-level visual annotation and should not
require annotators to provide demographic, sensitive, or private personal
information.
\end{quote}

\paragraph{Review and quality control.}
Candidate image--annotation pairs were manually reviewed before inclusion
in the final evaluation set. The review focused on whether annotated
entities were visually present, whether their  coarse-to-fine category paths were
appropriate, whether the finest category labels were sufficiently
supported, and whether the visual attributes were consistent with the
image. Samples with unclear visual evidence, ambiguous fine-grained
labels, severe occlusion or truncation, low visual quality, or duplicate
content were revised or removed.


\paragraph{Consent and data use.}
Annotators were informed that the annotations would be used to construct
a research benchmark, evaluate multimodal models, and report aggregate
statistics and qualitative examples in the paper. The benchmark focuses
on objects and scenes rather than human subjects. During image selection,
we avoided images centered on identifiable people or sensitive personal
information. For real-world photographs, images were collected by the
research team or used with permission. For web-sourced images,
redistribution will follow the corresponding source licenses or terms of
use. When redistribution rights are unclear, we will release annotations,
metadata, and evaluation code as appropriate rather than redistributing
the original images.

\paragraph{Ethics review.}
The study did not undergo formal ethics board review. The annotation
task involved low-risk object-level image annotation, did not involve
interventions with human subjects, and did not require the collection of
sensitive personal information. We nevertheless provided annotators with
task instructions, informed them of the research use of the annotations,
and filtered images to avoid sensitive or privacy-invasive content.

\paragraph{Annotator compensation and crediting.}
The annotation was conducted as part of the research project by domain-aware annotators rather than through an external crowdsourcing marketplace.
Annotators were compensated or credited according to the project arrangement.
No annotator was asked to provide sensitive personal information as part of the annotation task.

\subsection{More Qualitative Examples}
\label{sec:appendix_annotation_examples}

We provide additional qualitative examples of \textsc{GranFact}
annotations in Figure~\ref{fig:appendix_dataset_examples_1},
Figure~\ref{fig:appendix_dataset_examples_2}, and
Figure~\ref{fig:appendix_dataset_examples_3}. These examples complement
the qualitative examples shown in the main paper and cover the remaining
domains, including landmarks, cars, daily objects, plants, and games.

\begin{figure*}[t]
    \centering
    \includegraphics[width=\textwidth]{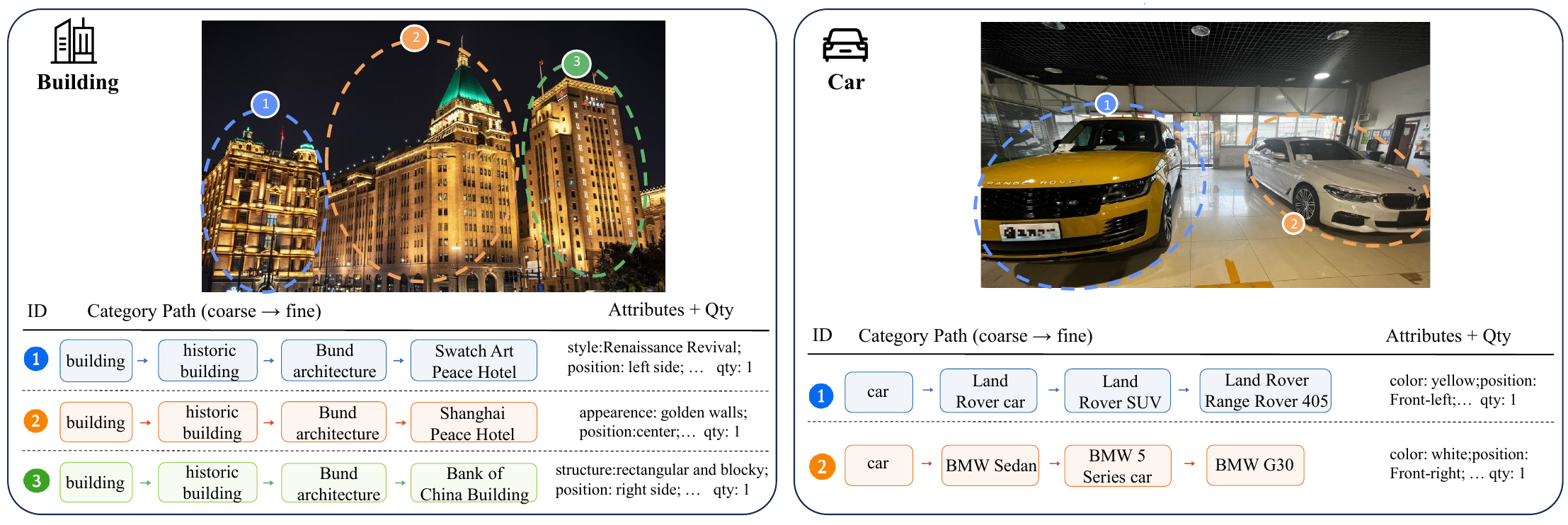}
    \caption{
    Additional qualitative examples from the landmark and car domains.
    }
    \label{fig:appendix_dataset_examples_1}
\end{figure*}

\begin{figure*}[t]
    \centering
    \includegraphics[width=\textwidth]{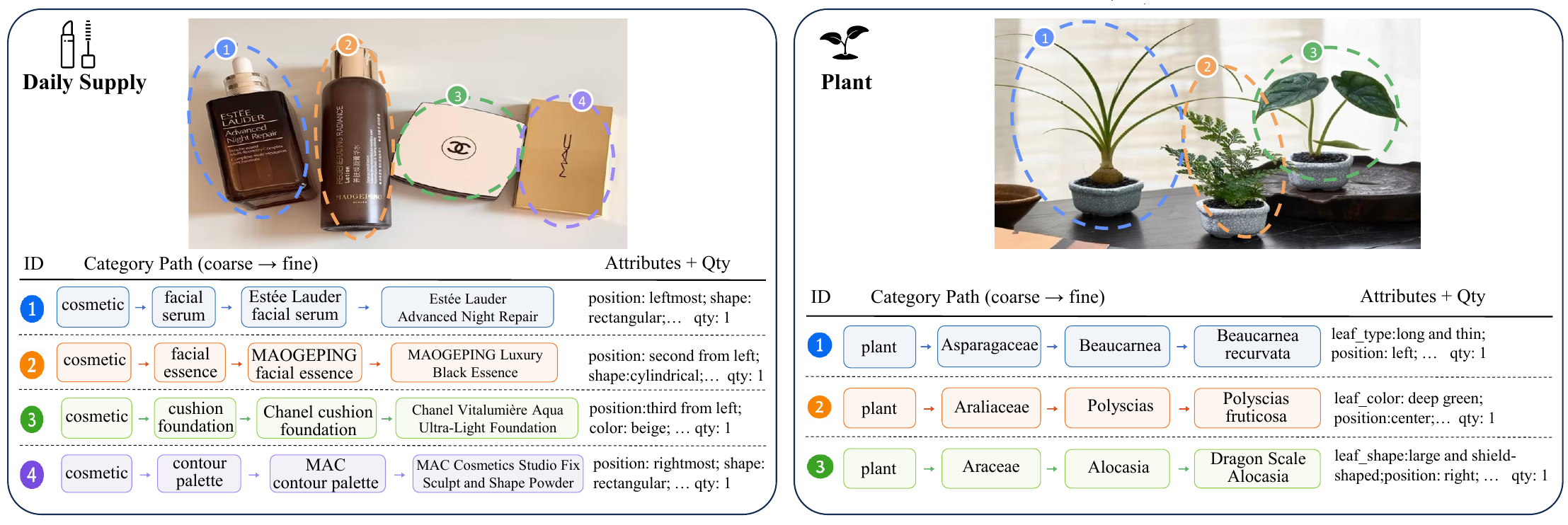}
    \caption{
    Additional qualitative examples from the daily-object and plant
    domains.
    }
    \label{fig:appendix_dataset_examples_2}
\end{figure*}

\begin{figure}[t]
    \centering
    \includegraphics[width=0.5\textwidth]{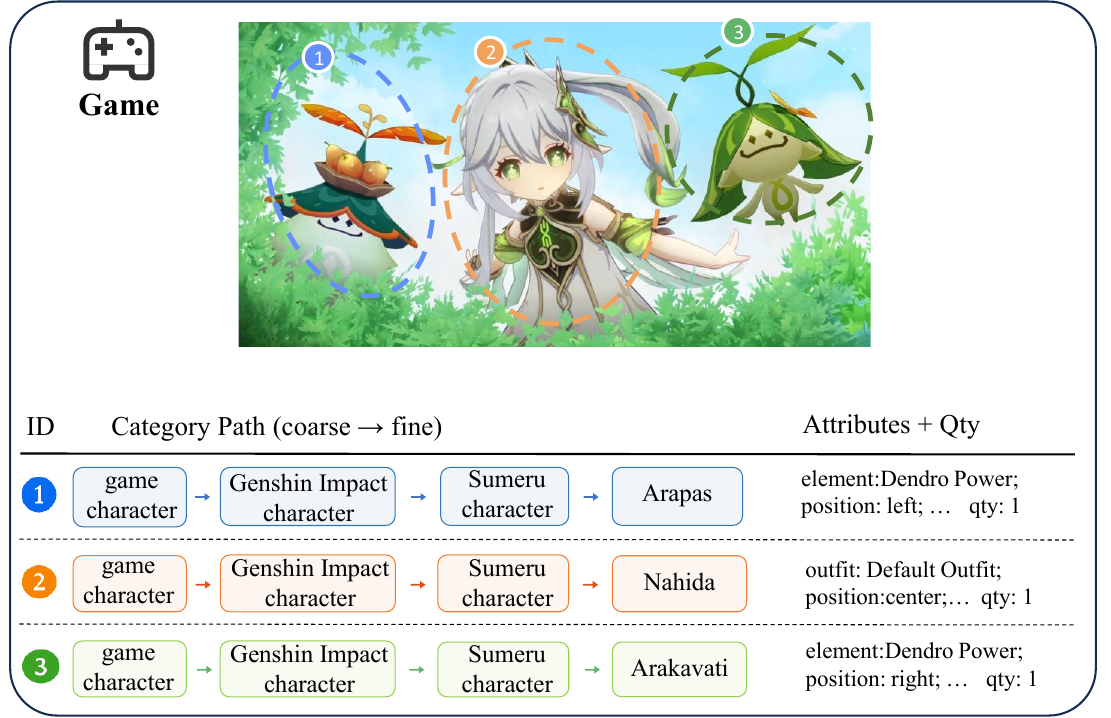}
    \caption{
    Additional qualitative example from the game domain.
    }
    \label{fig:appendix_dataset_examples_3}
\end{figure}

\section{Evaluation Details}
\label{sec:appendix_eval_details}

\subsection{Details of Response Parsing}
\label{sec:appendix_parser}

\textsc{GranFact} evaluates open-ended model responses rather than
closed-set predictions. Before applying the hierarchy-aware assignment
algorithm, we first convert each free-form response into a set of
structured prediction entities. Given an MLLM response, we use an LLM
parser to extract the primary physical objects mentioned in the response,
together with their category descriptions and visual attributes.

A rule-based parser is brittle in this setting because open-ended
responses often express object information across multiple clauses or
sentences. For example, a model may introduce an object in one sentence
and describe its color, position, quantity, or fine-grained identity
later in the response. Responses may also contain coordinated noun
phrases, implicit quantities, uncertain category claims, or background
objects that should not be treated as primary predictions. We therefore
use an LLM parser to leverage its language-understanding ability while
constraining its output with a fixed extraction prompt.

In our implementation, we use Qwen3.5-27B as the parser model. The parser
is instructed to extract only in-scope primary physical objects, keep
quantity and position as attributes, avoid extracting generic background
or low-information context objects, and output a JSON array following the
specified format. The resulting prediction entities are then used as the
input to the prediction--GT assignment algorithm described in
Section~\ref{sec:flow_formulation}. The core parsing prompt is shown in
Figure~\ref{fig:appendix_response_parsing_prompt}. In implementation, we
prepend a short domain-specific scope instruction to restrict extraction
to the target visual domain, while keeping the same output format and
general extraction rules across domains.

\begin{figure*}[t]
\centering
\begin{PromptCardNoBreak}
{\scriptsize
\promptlbl{Response Parsing Prompt}

\vspace{3pt}

You are an expert extraction assistant. Extract all in-scope primary
physical objects from the text into exactly one JSON array.

\vspace{4pt}
\promptlbl{Scope}

Extract primary physical objects that belong to the target visual domain.
Exclude generic background, environment, surface, lighting, UI elements,
and low-information context objects unless the response itself focuses on
them. Minor accessories, holders, connectors, supports, or object parts
should usually be kept as attributes or omitted, unless they are described
as main objects.

\vspace{4pt}
\promptlbl{Extraction Rules}

\begin{itemize}[leftmargin=1.4em,itemsep=0pt,topsep=1pt]
    \item Output exactly one JSON array and no extra text.
    \item Each extracted object must contain \texttt{finest\_category}
    and \texttt{attributes}.
    \item Use the most specific category explicitly supported by the
    response. Every word in \texttt{finest\_category} must be grounded in
    the text.
    \item Use an English category name for \texttt{finest\_category}. Do
    not copy non-English OCR text, packaging text, labels, or abnormal
    OCR-noise sequences into \texttt{finest\_category}.
    \item Keep visual details such as color, material, finish, texture,
    pattern, size, shape, position, orientation, state, condition, and
    part counts in \texttt{attributes}, unless they are part of an
    identity-bearing category name.
    \item Product brand, model, series, or version may stay in
    \texttt{finest\_category} only when they are explicit and
    identity-bearing.
    \item Put quantity and position inside \texttt{attributes}, not as
    top-level fields. Store quantity in \texttt{attributes.number} and
    explicit position or order in \texttt{attributes.position}.
    \item Use an exact integer string for a specific quantity. Use
    \texttt{"1"} for one countable object instance, and use \texttt{"2"}
    for a pair or complete paired set, such as earbuds, shoes, socks,
    gloves, or chopsticks.
    \item Use \texttt{"uncertain"} for ranged, approximate,
    plural-but-unspecified, or genuinely unclear quantities, including
    sets or collections whose item count is not clear.
    \item Do not double count the same physical objects at both group and
    member levels. If a counted group is partially described by more
    specific members, extract the members separately and set the group
    quantity to the remaining unspecified count.
    \item Do not merge objects with different explicit labels, series,
    specifications, capacities, or form factors.
    \item If an object has a labeled part, component, or accessory, do not
    use that part, component, or accessory name as the category of the
    whole object.
\end{itemize}

\vspace{4pt}
\promptlbl{Uncertainty Rules}

If the response gives a certain category but an uncertain subtype or
model, keep only the certain category. If the response says
``possibly X or Y'', use the most specific common super-category supported
by both options. If all category wording is uncertain, infer the most
specific category still supported by the response.

\vspace{4pt}
\promptlbl{Output Format}

\begin{verbatim}
[
  {
    "finest_category": "...",
    "attributes": {
      "number": "...",
      "position": "...",
      "attribute_name": "attribute_value"
    }
  }
]
\end{verbatim}
}
\end{PromptCardNoBreak}
\caption{
Core prompt template used for parsing open-ended MLLM responses into
structured prediction entities. In implementation, this core prompt is
combined with a short domain-specific scope instruction, such as
restricting extraction to electronic devices, cars, buildings, daily-use
consumer goods, or game-related entities.
}
\label{fig:appendix_response_parsing_prompt}
\end{figure*}

\subsection{Granularity-Aware Assignment Score Computation}
\label{sec:appendix_score_computation}
\label{sec:appendix_category_matching_prompt}

This subsection provides additional details on how we compute the
granularity-aware assignment score $w_{ij}$ in Eq.~(\ref{eq:edge_weight}).
For each parsed prediction entity $p_i$ and each GT entity $g_j$, the
score is determined by two components: the matched category granularity
level $\ell_{ij}$ and the attribute consistency score $a_{ij}$. Both
components are obtained with an LLM judge.

\paragraph{Category granularity level.}
Let $q_i$ denote the predicted category of $p_i$, and let
$\mathcal{H}_j=(c_{j,1},\ldots,c_{j,L_j})$ denote the coarse-to-fine
category path of GT entity $g_j$. The goal is to determine the finest
level in $\mathcal{H}_j$ that is reliably supported by the prediction
category $q_i$. We denote this level as
$\ell_{ij}\in\{0,\ldots,L_j\}$, where $\ell_{ij}=0$ indicates that
$p_i$ is category-incompatible with $g_j$, and $\ell_{ij}>0$ indicates
that $p_i$ can be matched to the $\ell_{ij}$-th node in the GT category
path.

To obtain $\ell_{ij}$, we use an LLM judge to perform pairwise
category-level matching between the predicted category $q_i$ and candidate
category nodes in $\mathcal{H}_j$. The judge returns one of three
decisions: \texttt{is\_a}, \texttt{may\_refer\_to}, or
\texttt{cannot\_refer\_to}. The decision \texttt{is\_a} means that
$q_i$ is the target category itself, a subtype or instance of the target,
or a synonym of the target. The decision \texttt{may\_refer\_to} means
that $q_i$ is too generic or underspecified but could still refer to the
target. The decision \texttt{cannot\_refer\_to} means that $q_i$ is
incompatible with the target due to a different domain, brand, family,
generation, species, function, or another mutually exclusive identity.

We traverse the category path $\mathcal{H}_j$ from the finest node
$c_{j,L_j}$ to the coarsest node $c_{j,1}$. If the judge returns
\texttt{is\_a} for node $c_{j,k}$, we stop the traversal and set
$\ell_{ij}=k$, since the prediction is supported at that category level.
If the judge returns \texttt{may\_refer\_to}, we continue the traversal
toward coarser levels, because the prediction may be compatible with the
GT entity but is not specific enough to support the current target node.
If the judge returns \texttt{cannot\_refer\_to}, we stop the traversal
and set $\ell_{ij}=0$, treating the prediction as incompatible with the
entire GT category path. In implementation, pairwise category judgments
are cached to avoid redundant LLM calls for repeated prediction--target
category pairs.

This three-way procedure distinguishes underspecification from
incompatibility. For example, if the GT entity is annotated as
``smartphone'' $\rightarrow$ ``iPhone'' $\rightarrow$
``iPhone 15 Pro Max'', a prediction such as ``iPhone'' may match the
coarser GT node ``iPhone'' but not the finest node. In contrast, a
prediction such as ``iPhone 15 Pro'' should not receive credit merely
because it shares the ancestor ``iPhone'', since it asserts a mutually
exclusive fine-grained model. In such cases, the prediction is treated as
category-incompatible with that GT entity. More generally,
$\ell_{ij}$ reflects the finest GT-supported semantic scope that the
prediction can reliably claim: a coarse but correct prediction receives a
positive but lower matched level, a correct fine-grained prediction
receives a higher matched level, and an incompatible prediction receives
$\ell_{ij}=0$.

The category-matching prompt is shown in
Figure~\ref{fig:appendix_category_matching_prompt}.

\begin{figure*}[t]
\centering
\begin{PromptCardNoBreak}
{\scriptsize
\promptlbl{Pairwise Category-Level Matching Judge Prompt}

\vspace{3pt}

You are determining the relationship between a prediction category and a
target category. Given a \texttt{Prediction} and a \texttt{Target}, decide
whether the prediction is the same as, a subtype of, incompatible with,
or may refer to the target under the category hierarchy.

\vspace{4pt}
\promptlbl{Input}

\begin{itemize}[leftmargin=1.4em,itemsep=0pt,topsep=1pt]
    \item \textbf{Prediction:} \texttt{\{predicted\_category\}}
    \item \textbf{Target:} \texttt{\{target\_label\}}
\end{itemize}

\vspace{4pt}
\promptlbl{Target Format}

The target may contain parenthetical annotations, such as
\texttt{fragrance (bottle)} or \texttt{eye cream (jar)}. The parenthetical
content provides contextual information rather than a strict requirement.
Focus on the main category and use the parenthetical content only as a
hint. For example, \texttt{fragrance} and \texttt{perfume} should be
treated as equivalent categories.

\vspace{4pt}
\promptlbl{Three-Way Decision}

\begin{itemize}[leftmargin=1.4em,itemsep=0pt,topsep=1pt]
    \item \texttt{is\_a}: the prediction is the target, or is clearly a
    subtype or instance of the target. For example,
    \texttt{iPhone 15 Pro} is \texttt{is\_a} \texttt{iPhone} and
    \texttt{smartphone}; \texttt{perfume} is \texttt{is\_a}
    \texttt{fragrance (bottle)}.
    \item \texttt{cannot\_refer\_to}: the prediction cannot refer to the
    target because they are from different domains, brands, families,
    generations, species, functions, or mutually exclusive identities. For
    example, \texttt{iPhone 15 Pro} is \texttt{cannot\_refer\_to}
    \texttt{Samsung Galaxy}, and \texttt{perfume} is
    \texttt{cannot\_refer\_to} \texttt{eye cream (jar)}.
    \item \texttt{may\_refer\_to}: the prediction may refer to the target
    but is too generic, underspecified, or semantically overlapping without
    a strict is-a relation. For example, \texttt{smartphone} is
    \texttt{may\_refer\_to} \texttt{iPhone 15 Pro}, and
    \texttt{beverage} is \texttt{may\_refer\_to} \texttt{Coca Cola}.
\end{itemize}

\vspace{4pt}
\promptlbl{Decision Priority}

\begin{itemize}[leftmargin=1.4em,itemsep=0pt,topsep=1pt]
    \item If you can confidently determine that the prediction is
    incompatible with the target, return \texttt{cannot\_refer\_to}.
    \item If the prediction is the same as the target, a clear subtype of
    the target, or a synonym of the target, return \texttt{is\_a}.
    \item If the prediction is too generic, underspecified, or uncertain
    but could still refer to the target, return \texttt{may\_refer\_to}.
\end{itemize}

\vspace{4pt}
\promptlbl{Output}

Return only a JSON object in the following format:
\begin{verbatim}
{
  "decision": "is_a",
  "reason": "brief explanation"
}
\end{verbatim}

The value of \texttt{decision} must be one of
\texttt{is\_a}, \texttt{cannot\_refer\_to}, or
\texttt{may\_refer\_to}.
}
\end{PromptCardNoBreak}
\caption{
LLM judge prompt for pairwise category-level matching. In implementation,
the judge is applied between a predicted category and candidate target
categories from the GT hierarchy, and the resulting
\texttt{is\_a}/\texttt{cannot\_refer\_to}/\texttt{may\_refer\_to}
decisions are used to determine the matched category granularity level.
Domain-specific examples are appended for electronics, daily products,
cars, animals, plants, buildings, and game entities.
}
\label{fig:appendix_category_matching_prompt}
\end{figure*}

\paragraph{Attribute consistency score.}
For category-compatible pairs with $\ell_{ij}>0$, we further compute an
attribute consistency score $a_{ij}\in[0,1]$. This score measures the
verifiable truthfulness of predicted non-quantitative attributes and the
recall of GT non-quantitative attributes with respect to the GT entity
$g_j$. The attributes include properties such as color, material,
texture, shape, posture, spatial position, visible state, brand-related
visual cues, and other domain-specific visual details when applicable.
Quantity is not evaluated as part of $a_{ij}$; it is handled separately
by the assignment constraints through the predicted quantity $s_i$ and GT
capacity $d_j$.

We provide the LLM judge with one predicted object and all
category-compatible GT candidates. Category legality has already been
decided upstream, so the judge is instructed not to use category labels,
category depth, global assignment, other predictions, or quantity
allocation. For each candidate $g_j$, the judge identifies: (1) predicted
attribute facts that are clearly supported by the GT attributes; (2)
predicted attribute facts that are clearly contradicted by the GT
attributes; and (3) GT attribute facts that are explicitly recalled or
covered by the predicted attributes. Predicted attribute facts that are
neither clearly supported nor clearly contradicted by the GT attributes
are treated as unverifiable and are excluded from the attribute precision
denominator. GT facts that are not clearly covered by the prediction are
treated as unrecalled.

Let $\mathcal{A}_{i}^{\mathrm{pred}}$ denote the set of non-quantitative
attribute facts predicted for $p_i$, and let
$\mathcal{A}_{j}^{\mathrm{gt}}$ denote the set of non-quantitative GT
attribute facts of $g_j$. Let
$\mathcal{C}_{ij}^{\mathrm{pred}}\subseteq
\mathcal{A}_{i}^{\mathrm{pred}}$ denote the predicted attribute facts
that are judged to be correct, and let
$\mathcal{W}_{ij}^{\mathrm{pred}}\subseteq
\mathcal{A}_{i}^{\mathrm{pred}}$ denote the predicted attribute facts
that are judged to be wrong or contradicted. Let
$\mathcal{R}_{ij}^{\mathrm{gt}}\subseteq
\mathcal{A}_{j}^{\mathrm{gt}}$ denote the GT attribute facts that are
recalled by the prediction. We compute
\begin{equation}
\small
\mathrm{P}_{ij}^{\mathrm{attr}}
=
\frac{
|\mathcal{C}_{ij}^{\mathrm{pred}}|
}{
|\mathcal{C}_{ij}^{\mathrm{pred}}|
+
|\mathcal{W}_{ij}^{\mathrm{pred}}|
},
\end{equation}
\begin{equation}
\small
\mathrm{R}_{ij}^{\mathrm{attr}}
=
\frac{
|\mathcal{R}_{ij}^{\mathrm{gt}}|
}{
|\mathcal{A}_{j}^{\mathrm{gt}}|
},
\end{equation}
and set
\begin{equation}
\small
a_{ij}
=
\frac{
2\mathrm{P}_{ij}^{\mathrm{attr}}
\mathrm{R}_{ij}^{\mathrm{attr}}
}{
\mathrm{P}_{ij}^{\mathrm{attr}}
+
\mathrm{R}_{ij}^{\mathrm{attr}}
}.
\end{equation}
When both $p_i$ and $g_j$ have no non-quantitative attributes, we set
$a_{ij}=1$, so that the edge score reflects a pure category-level match.
If
$|\mathcal{C}_{ij}^{\mathrm{pred}}|+
|\mathcal{W}_{ij}^{\mathrm{pred}}|=0$
in cases where attribute precision is otherwise needed, we set
$\mathrm{P}_{ij}^{\mathrm{attr}}=0$. If
$|\mathcal{A}_{j}^{\mathrm{gt}}|=0$
in cases where attribute recall is otherwise needed, we set
$\mathrm{R}_{ij}^{\mathrm{attr}}=0$. If
$\mathrm{P}_{ij}^{\mathrm{attr}}+\mathrm{R}_{ij}^{\mathrm{attr}}=0$, we
set $a_{ij}=0$.

For category-incompatible pairs with $\ell_{ij}=0$, the attribute score
does not affect the final assignment score, since the edge is already
treated as incompatible at the category level. The LLM judge rubric for
attribute scoring is shown in
Figure~\ref{fig:appendix_attribute_scoring_prompt}.

\begin{figure*}[t]
\centering
\begin{PromptCardNoBreak}
{\small
\promptlbl{Attribute Truthfulness and Recall Judge Prompt}

\vspace{3pt}

You are a pair-level attribute truthfulness and recall judge. Given one
predicted object's non-quantity attribute facts and several
category-legal GT candidates' non-quantity attribute facts, judge
attributes only.

\vspace{5pt}
\promptlbl{Input}

\begin{itemize}[leftmargin=1.5em,itemsep=1pt,topsep=2pt]
    \item \textbf{Predicted object:} \texttt{\{predicted\_object\}}
    \item \textbf{Category-legal GT candidates:}
    \texttt{\{gt\_candidates\}}
\end{itemize}

\vspace{5pt}
\promptlbl{Decision rules}

\begin{itemize}[leftmargin=1.5em,itemsep=1pt,topsep=2pt]
    \item Category legality has already been decided upstream. Do not use
    category labels, category depth, global assignment, other predictions,
    or quantity allocation.
    \item For each candidate GT, identify predicted attribute facts that
    are clearly consistent with or supported by the GT attributes.
    \item Identify predicted attribute facts that are clearly inconsistent
    with or contradicted by the GT attributes.
    \item Identify GT attribute facts that are explicitly recalled or
    covered by the predicted attributes.
    \item Be literal and conservative. Do not infer hidden facts.
    \item If a predicted attribute is not clearly correct or clearly wrong
    from the GT attributes, leave it out of both
    \texttt{correct\_pred\_fact\_indices} and
    \texttt{wrong\_pred\_fact\_indices}; it will be treated as
    uncertain or unverifiable.
    \item If a GT attribute is not clearly covered by the predicted
    attributes, leave it out of \texttt{recalled\_gt\_fact\_indices}; it
    will be treated as unrecalled.
    \item The same predicted fact index should not appear in both
    \texttt{correct\_pred\_fact\_indices} and
    \texttt{wrong\_pred\_fact\_indices}.
    \item Return one result for every input \texttt{candidate\_id}. Do not
    omit candidates; use empty arrays when no facts match.
    \item Keep each reason brief; do not include internal reasoning or
    self-correction.
\end{itemize}

\vspace{5pt}
\promptlbl{Output}

Return only one JSON object with the following structure:
\begin{verbatim}
{
  "candidates": {
    "g3": {
      "correct_pred_fact_indices": [0, 2],
      "wrong_pred_fact_indices": [1],
      "recalled_gt_fact_indices": [0, 2],
      "reason": "brief explanation"
    },
    "g5": {
      "correct_pred_fact_indices": [],
      "wrong_pred_fact_indices": [0],
      "recalled_gt_fact_indices": [],
      "reason": "brief explanation"
    }
  }
}
\end{verbatim}
}
\end{PromptCardNoBreak}
\caption{
LLM judge prompt for evaluating attribute truthfulness and recall between
a predicted entity and category-legal GT candidates.
}
\label{fig:appendix_attribute_scoring_prompt}
\end{figure*}

\paragraph{Assignment score.}
Given $\ell_{ij}$ and $a_{ij}$, we compute the assignment score
$w_{ij}$ as in Eq.~(\ref{eq:edge_weight}). For a real GT entity
$j\le n$, a category-compatible pair receives
\begin{equation}
\small
w_{ij}
=
\frac{\ell_{ij}+a_{ij}}{L_j+1},
\quad \ell_{ij}>0.
\end{equation}
This normalization keeps the score in a comparable range across GT
entities with different category-path depths. It also ensures that
matching a finer category level receives a higher score than matching a
coarser level, while attribute consistency can further distinguish
assignments at the same matched category level.

If $p_i$ is category-incompatible with a real GT entity $g_j$, we set
$w_{ij}=-1$. This negative score discourages the optimizer from assigning
unsupported predictions to real GT entities. For the dummy hallucination
entity $g_{n+1}$, we set $w_{i,n+1}=0$. Therefore, unsupported prediction
quantity can be assigned to the dummy entity instead of being forced to
match a real GT entity. This design makes hallucinated or over-specific
content explicitly count as unsupported while avoiding spurious matches
to real objects.

\subsection{Capacity-Constrained Assignment Solver}
\label{sec:appendix_assignment_solver}

This subsection provides implementation details for solving the
lexicographic assignment problem in Eq.~\ref{eq:lexicographic_flow}.
Given the prediction entities, GT entities, and the granularity-aware
assignment scores, we instantiate the prediction--GT alignment as a
capacity-constrained bipartite flow problem.

\paragraph{Graph construction.}
For each parsed prediction entity $p_i$, we create a prediction node with
supply $s_i$. For each GT entity $g_j$, we create a GT node with capacity
$d_j$. A directed edge is added from $p_i$ to $g_j$ only when the pair is
category-compatible, i.e., when $\ell_{ij}>0$. The edge capacity is set
to
\begin{equation}
\small
    u_{ij} = \min(s_i,d_j),
\end{equation}
so that the assigned quantity on this edge cannot exceed either the
prediction supply or the GT capacity. We then add a source node connected
to all prediction nodes and a sink node connected from all GT nodes. The
source-to-prediction edge has capacity $s_i$, and the GT-to-sink edge has
capacity $d_j$.

In Section~\ref{sec:flow_formulation}, we introduce a dummy hallucination
entity $g_{n+1}$ to present the assignment problem in a unified form:
every predicted quantity is either assigned to a real GT entity or to the
dummy entity. In implementation, we use an equivalent real-GT-only graph
and do not explicitly instantiate the dummy node. Predicted quantity that
cannot be assigned to any category-compatible real GT entity remains
unmatched. This unmatched residual can be used for diagnostic accounting
of unsupported or over-counted predictions, while the matched quantities
on real GT edges determine the assignment scores used by the main
metrics.

\paragraph{Lexicographic objective.}
The assignment objective is lexicographic: it first maximizes the total
category-compatible matched quantity $M(\mathbf{X})$, and then, among
assignments with the same matched quantity, maximizes the
granularity-aware score $M_{\mathrm{gran}}(\mathbf{X})$. We implement
this priority using a standard large-constant transformation in the
min-cost-flow objective.

For each compatible edge $(p_i,g_j)$, we assign a reward consisting of
two parts: a large base reward for assigning one unit of predicted
quantity to a real GT entity, and a smaller granularity-aware reward
based on the edge score $w_{ij}$. Since the solver minimizes cost, this
reward is represented as a negative edge cost:
\begin{equation}
\small
    \mathrm{cost}_{ij}
    =
    -\left(B + \tilde{w}_{ij}\right),
\end{equation}
where $\tilde{w}_{ij}$ is an integer-scaled version of $w_{ij}$ and
$B$ is a large constant.

The constant $B$ is chosen so that the base reward for one additional
matched quantity is larger than any possible difference in total
granularity-aware utility among assignments with the same or fewer
matched quantities. Therefore, the solver always prefers an assignment
with larger category-compatible matched quantity. Only when two
assignments match the same total quantity does the second term
$\tilde{w}_{ij}$ determine which assignment is preferred. This realizes
the lexicographic objective in Eq.~\ref{eq:lexicographic_flow} within a
single min-cost-flow problem.

\paragraph{Successive shortest augmenting path.}
We solve the resulting min-cost-flow problem with a successive shortest
augmenting-path algorithm on the residual network. Starting from zero
flow, the algorithm repeatedly finds the shortest source-to-sink path
with negative total cost in the residual graph. If such a path exists, it
augments as much flow as possible along that path, updates the residual
capacities, and continues. The procedure stops when no negative-cost
source-to-sink path remains. At this point, no further assignment can
increase the lexicographic objective.

The returned flow specifies the assigned quantity $x_{ij}$ for each
activated prediction--GT edge. For each GT entity, the incoming flow
determines how much of its annotated quantity is covered by model
predictions. For each prediction entity, any quantity not assigned to a
real GT entity does not contribute to $M_z$ or $M_{\mathrm{gran},z}$; it
only remains unmatched while the total predicted quantity is still
included in $S_z$. Thus, unmatched prediction quantity affects the final
metrics by reducing the fraction of predicted content that is grounded,
rather than by adding a separate score term. The quantities
$S_z$, $D_z$, $M_z$, and $M_{\mathrm{gran},z}$ are then used for metric
computation.

\paragraph{Validity checks.}
After solving, we verify that the produced assignment satisfies the
capacity constraints. Specifically, the assigned quantity on each edge
must not exceed its edge capacity, the total outgoing assigned quantity
of a prediction entity must not exceed its predicted supply, and the
total incoming assigned quantity of a GT entity must not exceed its GT
capacity. These checks ensure that the resulting assignment is a feasible
capacity-constrained alignment before metric computation.

\subsection{Metric Computation and Well-Definedness}
\label{sec:appendix_metric_details}
\label{sec:appendix_welldefined_metrics}

This subsection provides additional details on the computation of the
metrics introduced in Section~\ref{sec:flow_metrics}. We first describe
how example-level quantities are aggregated at the dataset level, and
then explain why the reported metrics are well-defined even when the
optimal assignment matrix in Eq.~(\ref{eq:lexicographic_flow}) is not
unique.

\paragraph{Dataset-level aggregation.}
For each evaluation example $z\in\mathcal{Z}$, let $S_z$ denote the
total predicted quantity parsed from the model response, and let $D_z$
denote the total annotated GT quantity. After solving the
prediction--GT assignment problem, we obtain the granularity-neutral
score $M_z$ and the granularity-aware matched score
$M_{\mathrm{gran},z}$. Dataset-level precision, recall, and F1 are
computed by micro-averaging these quantities over all examples:
\begin{equation}
\small
\begin{aligned}
S &= \sum_{z\in\mathcal{Z}} S_z,
&
D &= \sum_{z\in\mathcal{Z}} D_z, \\
M &= \sum_{z\in\mathcal{Z}} M_z,
&
M_{\mathrm{gran}} &= \sum_{z\in\mathcal{Z}} M_{\mathrm{gran},z}.
\end{aligned}
\end{equation}
The metrics in Section~\ref{sec:flow_metrics} are then computed from
these aggregated quantities. This micro-averaged computation treats each
predicted or annotated quantity unit consistently across the dataset,
rather than first computing per-image scores and then averaging them.

\paragraph{Granularity-neutral metrics.}
Granularity-neutral metrics evaluate whether the predicted entities can
be grounded in category-compatible GT entities, regardless of how
fine-grained the grounded descriptions are. In particular, $M$ counts the
amount of predicted quantity assigned to category-compatible GT entities.
Thus, precision measures the fraction of generated quantity that is
grounded, while recall measures the fraction of GT quantity that is
covered by grounded predictions:
\begin{equation}
\small
\mathrm{P}=\frac{M}{S},\quad
\mathrm{R}=\frac{M}{D},\quad
\mathrm{F1}=
\frac{2\mathrm{P}\mathrm{R}}{\mathrm{P}+\mathrm{R}} .
\end{equation}
When the denominator of F1 is zero, we set F1 to zero. In addition to
these quantity-level metrics, we report the Image Reliability Rate
$\mathrm{IR}$, which is a stricter image-level metric. An example
contributes positively to $\mathrm{IR}$ only when all predicted quantity
in that response is assigned to category-compatible GT entities:
\begin{equation}
\small
    \mathrm{IR}
    =
    \frac{1}{|\mathcal{Z}|}
    \sum_{z\in\mathcal{Z}}
    \mathbb{I}[M_z=S_z].
\end{equation}
Therefore, $\mathrm{P}$, $\mathrm{R}$, and $\mathrm{F1}$ measure
quantity-level grounding, while $\mathrm{IR}$ measures whether a complete
open-ended response is free from unsupported predicted quantity.

\paragraph{Granularity-aware metrics.}
Granularity-aware metrics further account for the specificity and
attribute consistency of grounded predictions. Instead of counting each
category-compatible assignment equally, $M_{\mathrm{gran}}$ weights
grounded assignments by their matched category level and attribute
consistency. A prediction receives a higher score when it is matched to a
finer supported category level and has more consistent non-quantitative
attributes. Based on $M_{\mathrm{gran}}$, we compute granularity-weighted
precision, recall, and F1:
\begin{equation}
\small
\begin{gathered}
\mathrm{P}_{\mathrm{gran}}=\frac{M_{\mathrm{gran}}}{S},
\quad
\mathrm{R}_{\mathrm{gran}}=\frac{M_{\mathrm{gran}}}{D}, \\
\mathrm{F1}_{\mathrm{gran}}
=
\frac{2\mathrm{P}_{\mathrm{gran}}\mathrm{R}_{\mathrm{gran}}}
{\mathrm{P}_{\mathrm{gran}}+\mathrm{R}_{\mathrm{gran}}}.
\end{gathered}
\end{equation}
Again, when the denominator of $\mathrm{F1}_{\mathrm{gran}}$ is zero, we
set $\mathrm{F1}_{\mathrm{gran}}$ to zero.

We also report the granularity-aware counterpart of image reliability:
\begin{equation}
\small
    \mathrm{GIR}
    =
    \frac{1}{|\mathcal{Z}|}
    \sum_{z\in\mathcal{Z}}
    \mathbb{I}[M_z=S_z]
    \cdot
    \frac{M_{\mathrm{gran},z}}{M_z},
\end{equation}
where the granularity term is set to zero when $M_z=0$.
This metric gives non-zero credit only to responses that are fully
reliable under the granularity-neutral criterion, and further weights
such responses by their average granularity-aware score. Finally, we
report
\begin{equation}
    \mathrm{G}_{\mathrm{avg}}
    =
    \frac{M_{\mathrm{gran}}}{M},
\end{equation}
where $\mathrm{G}_{\mathrm{avg}}$ is set to zero when $M=0$.
This diagnostic metric measures the average granularity-aware score among
grounded predictions. It helps distinguish a model that is reliable but
mostly produces coarse descriptions from one that is reliable and also
more specific.

\paragraph{Well-definedness under non-unique assignments.}
The lexicographic optimization in Eq.~(\ref{eq:lexicographic_flow}) may
admit multiple optimal assignment matrices. This non-uniqueness can occur
when several GT entities receive exactly the same assignment scores, or when
multiple assignments lead to the same reliability and granularity-aware
objectives. However, the reported metrics do not depend on which optimal
assignment matrix is selected.

To see this, consider an example $z$ and let
$\mathcal{X}^{\star}_z$ denote the set of lexicographically optimal
assignment matrices for this example. By definition, every
$\mathbf{X}^{\star}\in\mathcal{X}^{\star}_z$ first achieves the same
maximal value of the primary objective $M(\mathbf{X})$. Therefore,
$M_z$ is invariant across all optimal assignments. Among the assignments
that maximize $M(\mathbf{X})$, the second stage maximizes
$M_{\mathrm{gran}}(\mathbf{X})$. Thus, every lexicographically optimal
assignment also achieves the same maximal value of
$M_{\mathrm{gran}}(\mathbf{X})$, making $M_{\mathrm{gran},z}$ invariant
as well.

The remaining quantities, $S_z$ and $D_z$, are determined directly by
the parsed model response and the GT annotations, and do not depend on
the assignment matrix. Consequently, the example-level scalar quantities
used by the metrics, i.e., $S_z$, $D_z$, $M_z$, and
$M_{\mathrm{gran},z}$, are uniquely determined even if the concrete
optimal assignment matrix is not unique. Since all reported metrics are
functions only of these scalar quantities, the benchmark scores are
well-defined. Non-uniqueness may affect which tied assignment is chosen
for qualitative visualization, but it does not affect any reported
quantitative metric.

\subsection{Uncertain Quantity Handling}
\label{sec:appendix_uncertainty}

Some images contain objects whose exact quantities are difficult to
determine. For example, products on a crowded shelf may be heavily
occluded, partially visible, or arranged in a way that makes exact
counting unreliable. Model responses may also use vague quantity
expressions such as ``several'', ``many'', or ``a group of''. To avoid
treating all such cases as exact counting errors, we explicitly mark
uncertain quantities and handle them separately during assignment and
metric computation.

\paragraph{Quantity normalization.}
For both parsed predictions and GT annotations, each entity is associated
with a quantity type. If an entity has an explicit count, we treat it as
a numeric quantity. If no count is specified, we use a default quantity
of one. If the quantity is expressed as uncertain, approximate, ranged,
or plural-but-unspecified, we mark its quantity type as uncertain.
The uncertain label means that the entity should not be interpreted as
having a fixed exact count, even though a temporary unit value can be
used for the initial assignment step.

\paragraph{Initial assignment.}
The capacity-constrained solver requires finite supplies and capacities.
Therefore, in the initial flow problem, uncertain quantities are
temporarily instantiated with unit supply or capacity, while their
uncertain quantity type is retained. This gives the solver a conservative
initial assignment without committing to a precise count. The standard
capacity constraints and granularity-aware edge utilities are then
applied as described in Appendix~\ref{sec:appendix_assignment_solver}.

\paragraph{Post-assignment repair.}
After the initial assignment, we apply a limited repair step for
uncertain quantities. The repair is augment-only: it does not re-solve
the full assignment problem, but only expands uncertain quantities along
existing category-compatible edges.

First, if a prediction has remaining unmatched quantity and it is
connected to a GT entity whose quantity is uncertain, the uncertain GT
capacity may be expanded to absorb the residual prediction quantity.
Intuitively, when the GT annotation indicates that the exact number is
uncertain, extra compatible predicted instances should not necessarily be
treated as over-counting errors.

Second, if a GT entity has remaining uncovered quantity and it is
connected to a prediction whose quantity is uncertain, the uncertain
prediction supply may be expanded to cover the residual GT quantity.
This reflects the fact that a vague prediction such as ``several
bottles'' can cover more than one compatible GT instance when the exact
predicted count is not specified.

When multiple compatible edges are available for such repair, we prefer
edges with higher granularity-aware utility and stronger attribute
agreement, using a deterministic tie-breaking rule. This keeps the repair
consistent with the assignment objective while avoiding arbitrary
allocation of uncertain quantities.

\paragraph{Credit discount for uncertain predictions.}
Uncertain quantities on the prediction side are handled conservatively in
metric computation. If a prediction has uncertain quantity and is matched
to a GT entity with a certain quantity, we apply a partial credit
discount to the matched true-positive quantity. If both the prediction
and the GT entity have uncertain quantities, no discount is applied.

Formally, for an assigned prediction--GT edge $(p_i,g_j)$ with assigned
quantity $x_{ij}$, we define an effective matched quantity
\begin{equation}
\small
    \widehat{x}_{ij}
    =
    \rho_{ij} x_{ij},
\end{equation}
where we set $\rho_{ij}=\lambda_{\mathrm{unc}}$ only when $p_i$ has uncertain
quantity and $g_j$ has certain quantity; otherwise, $\rho_{ij}=1$.
Thus,
\begin{equation}
\small
    \rho_{ij}
    =
    \begin{cases}
    \lambda_{\mathrm{unc}}, & \mathrm{unc}(p_i)\wedge \neg\mathrm{unc}(g_j), \\
    1, & \text{otherwise},
    \end{cases}
\end{equation}
where $\mathrm{unc}(\cdot)$ indicates uncertain quantity status. We set
$\lambda_{\mathrm{unc}}=0.5$ in our implementation. Thus, an
uncertain prediction receives partial credit when matched to a GT entity
with a known exact count, because the prediction is category-compatible
but does not make an exact quantity commitment. In contrast, when both
sides have uncertain quantities, the match is not penalized for quantity
uncertainty.

The effective matched quantity $\widehat{x}_{ij}$ is used when
accumulating matched credit for the reliability-oriented metrics. The
total predicted quantity $S_z$ still counts the prediction-side quantity,
so uncertain predictions do not receive free precision. They only receive
reduced matched credit when their quantity uncertainty is less specific
than the GT quantity.

\paragraph{Metric accounting.}
Uncertain quantity handling does not introduce a separate metric. Instead,
it affects the quantities used by the existing metrics in
Section~\ref{sec:flow_metrics}. Matched real-GT assignments contribute to
the matched quantities according to their effective credit. Prediction
quantity that remains unmatched after assignment and repair does not
contribute to $M_z$ or $M_{\mathrm{gran},z}$, while the corresponding
predicted quantity is still included in $S_z$. Therefore, unmatched
prediction quantity lowers precision by reducing the fraction of
predicted content that is grounded, rather than by adding a separate
penalty term.

Similarly, uncovered GT quantity affects recall through $D_z$ and the
matched quantity. If an uncertain prediction is expanded to cover
compatible GT residual quantity, the additional matched quantity is
credited according to the discount rule above. This allows vague quantity
predictions to receive appropriate credit when they are compatible with
the GT, while still penalizing them when the GT provides a more precise
count.

\paragraph{Uncertain quantities outside the solver.}
Some entities do not enter the capacity-constrained solver. For example,
a prediction may be rejected before assignment because its category is
incompatible with all GT entities, or a GT entity may have no compatible
prediction. When such entities have uncertain quantities, we estimate
their effective quantities using the average quantity of comparable
entities with known counts in the same evaluation instance when
available, and fall back to a unit quantity otherwise. This convention
prevents uncertain quantities outside the solver from being arbitrarily
treated as exact large counts, while still allowing them to contribute to
the precision or recall denominator in a controlled way.

\section{Detailed Experimental Settings}
\label{sec:appendix_exp_details}

\subsection{Evaluated Models}
\label{sec:appendix_models}

We evaluate both open-source and closed-source MLLMs to cover a broad
range of model families, model scales, and deployment settings.
Table~\ref{tab:appendix_evaluated_models} summarizes the evaluated
models and how they are accessed. For locally evaluated models, we load
the corresponding HuggingFace checkpoints. For API-based models, we use
the API-accessible versions available at the time of evaluation. All
models are evaluated with the same set of images, prompt styles, and
evaluation pipeline.

\begin{table*}[t]
\centering
\small
\begin{tabular}{lll}
\toprule
Model & Access & Size \\
\midrule
InstructBLIP-Vicuna-7B~\cite{dai2023instructblip}
& Local HF checkpoint & 7B \\

InternVL3.5-8B~\cite{zhu2025internvl3}
& Local HF checkpoint & 8B \\

Qwen2.5-VL-7B~\cite{bai2025qwen25vl}
& Local HF checkpoint & 7B \\

Qwen3-VL-8B-Instruct~\cite{bai2025qwen3vl}
& Local HF checkpoint & 8B \\

Kimi-K2.6~\cite{moonshot2026kimik26}
& API & -- \\

GLM-4.6V~\cite{hong2025glmv}
& API & -- \\

GPT-5.4~\cite{openai2026gpt54}
& API & -- \\

Gemini-3.1-Flash-Lite~\cite{googledeepmind2026gemini31flashlite}
& API & -- \\
\bottomrule
\end{tabular}
\caption{
Evaluated MLLMs. The access column indicates whether each model is
evaluated by loading local HuggingFace weights or through an API. For
API-accessed models, we leave the size unspecified.
}
\label{tab:appendix_evaluated_models}
\end{table*}

For locally evaluated models, we follow the official model-specific
conversation template and image preprocessing pipeline when available.
For API-based models, we submit the same image and text prompt through
the corresponding API interface. We do not apply model-specific prompt
engineering beyond adapting the input format required by each interface.
The three evaluation prompt styles are kept semantically identical across
models and are provided in Appendix~\ref{sec:appendix_prompts}. Decoding
and inference settings are described in
Appendix~\ref{sec:appendix_decoding}.

\subsection{Decoding and Inference Settings}
\label{sec:appendix_decoding}

For each image--prompt pair, we generate one response from each model.
All prompt styles use the same decoding configuration within each model,
so that differences among conservative, neutral, and aggressive prompts
are induced by the prompt instructions rather than by decoding changes.

For locally evaluated models, we use deterministic decoding with sampling
disabled, temperature set to $0$, and a maximum generation length of
2048 new tokens. Since sampling is disabled, sampling-specific parameters
such as top-$p$ do not affect generation. For API-based models, we use
the closest available deterministic setting provided by the corresponding
API and set the maximum output length to 2048 tokens when the option is
available.

For locally evaluated models, we use the official conversation template
and default image preprocessing when available. For API-based models, we
use the corresponding API interface and only adapt the input format as
required by the provider. All generated responses are processed by the
same response parsing, assignment, and metric computation pipeline
described in Appendix~\ref{sec:appendix_eval_details}.

\paragraph{Reporting.}
Unless otherwise stated, all reported results are aggregate metrics computed over the full \textsc{GranFact} evaluation set; model performance is reported from a single evaluation run under the specified prompt style and decoding settings.
We do not report error bars, standard deviations, or confidence intervals across multiple random seeds or repeated runs.
To reduce stochastic variation, we use deterministic decoding whenever available, including disabling sampling and setting temperature to $0$ for locally evaluated models.

\subsection{Evaluation Prompts}
\label{sec:appendix_prompts}

We evaluate each model under three prompt styles to study how different
levels of requested specificity affect reliability and granularity. The
neutral prompt asks for a standard image description, the conservative
prompt encourages factual description without guessing, and the
aggressive prompt encourages more detailed generation. All prompts require
English responses and are applied consistently across models.

\begin{table}[t]
\centering
\small
\begin{tabularx}{\columnwidth}{lX}
\toprule
Prompt style & Prompt \\
\midrule
Neutral &
Describe the image. Respond in English only. \\

Conservative &
Describe the image factually and refrain from guessing unobserved details.
Respond in English only. \\

Aggressive &
Describe the image in as much detail as possible. Respond in English only. \\
\bottomrule
\end{tabularx}
\caption{
Evaluation prompts used for the three prompt styles.
}
\label{tab:appendix_eval_prompts}
\end{table}

\subsection{Auxiliary Training Data Construction}
\label{sec:appendix_training_data}

To train our alignment method, we construct a small auxiliary training
set that is separate from the \textsc{GranFact} evaluation set. The goal
is to provide fine-grained, structured training examples without using
the evaluation images. We focus on five domains: daily objects, animals,
plants, electronics, and cars.

We first collect textual category labels from these domains and use them
as image search queries. For each category label, we retrieve five
candidate images from the web. The candidate images are first screened by
an MLLM, Qwen3.5-27B, to remove clearly mismatched or visually unsuitable
images. They are then reviewed by human annotators, who select the image
that best matches the target category when a suitable candidate is
available. If none of the five retrieved images is suitable for a
category, that category is discarded. Table~\ref{tab:appendix_training_label_filtering}
summarizes the number of category labels before and after this filtering
process.

\begin{table}[t]
\centering
\small
\begin{tabular}{lrr}
\toprule
Domain & Initial labels & Retained labels \\
\midrule
Daily objects & 600 & 496 \\
Animals       & 150 & 104 \\
Plants        & 100 & 63  \\
Electronics   & 250 & 208 \\
Cars          & 150 & 104 \\
\midrule
Total         & 1250 & 975 \\
\bottomrule
\end{tabular}
\caption{
Category-label filtering statistics for the auxiliary training set.
}
\label{tab:appendix_training_label_filtering}
\end{table}

After filtering, we randomly combine several retained categories within
each domain to form image-generation specifications. For daily objects,
animals, plants, and electronics, each generated image contains categories
sampled from a range of 3 to 5. For cars, we use a smaller range of 2 to
4 categories, since fine-grained car categories often require more visual
space to remain identifiable. We then use Gemini-3.1-Flash-Image-Preview
to synthesize images from these category combinations. After generation,
we use GPT-5.4 to produce structured annotations for the synthesized
images, including entity-level coarse-to-fine category path and visual attributes.
The generated image--annotation pairs are then checked for validity.

The final auxiliary training set contains 525 images and 1986 annotated
object instances. Table~\ref{tab:appendix_training_domain_stats}
summarizes the per-domain statistics.

\begin{table}[t]
\centering
\small
\resizebox{\columnwidth}{!}{
\begin{tabular}{lrrrr}
\toprule
Domain & Images & Labels & Objects & Avg. obj/img \\
\midrule
Daily objects & 250 & 496 & 990 & 3.96 \\
Animals       & 50  & 104 & 194 & 3.88 \\
Plants        & 50  & 63  & 204 & 4.08 \\
Cars          & 75  & 104 & 195 & 2.60 \\
Electronics   & 100 & 208 & 403 & 4.03 \\
\midrule
Total         & 525 & 975 & 1986 & 3.78 \\
\bottomrule
\end{tabular}
}
\caption{
Domain-level statistics of the auxiliary training set.
}
\label{tab:appendix_training_domain_stats}
\end{table}

The auxiliary training set is constructed using textual category labels
from the \textsc{GranFact} domains, while all training images are kept
image-disjoint from the \textsc{GranFact} evaluation set.

Figure~\ref{fig:appendix_training_statistics} visualizes the domain
distribution, object-count distribution, and category-depth distribution
of the auxiliary training set.

\begin{figure*}[htbp]
\centering
\begin{minipage}{0.32\textwidth}
    \centering
    \includegraphics[width=\linewidth]{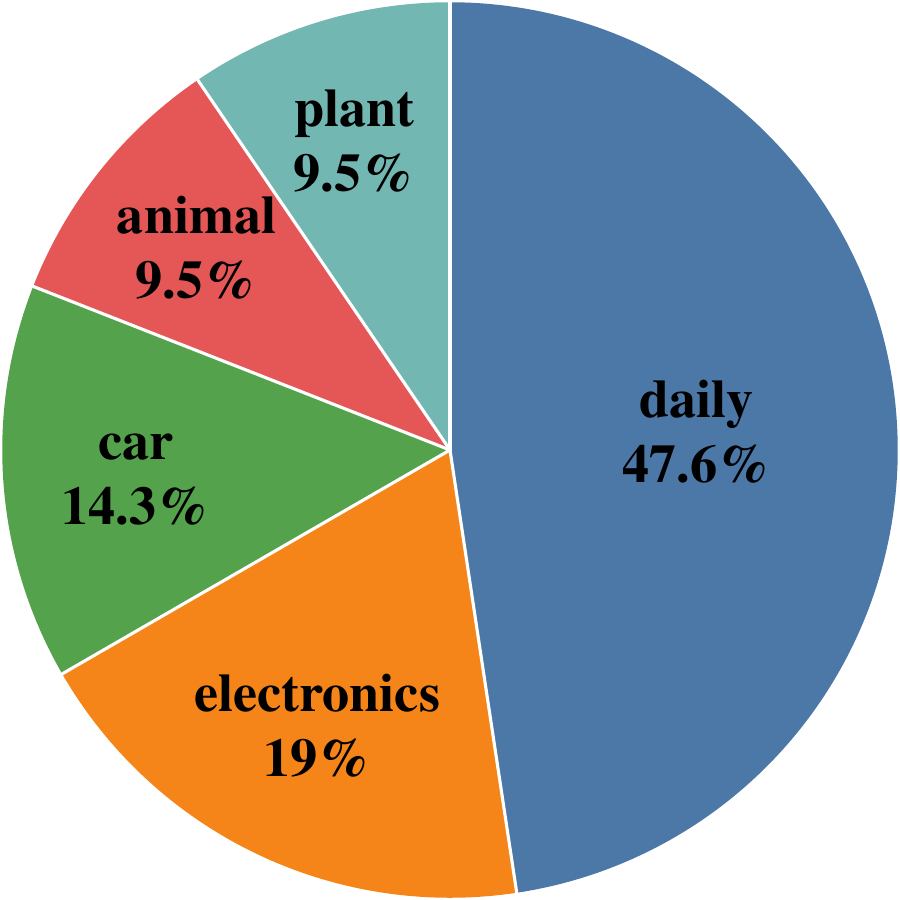}
\end{minipage}
\hfill
\begin{minipage}{0.32\textwidth}
    \centering
    \includegraphics[width=\linewidth]{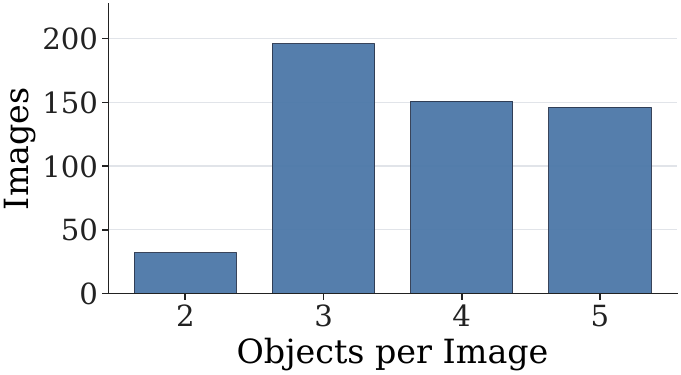}
\end{minipage}
\hfill
\begin{minipage}{0.32\textwidth}
    \centering
    \includegraphics[width=\linewidth]{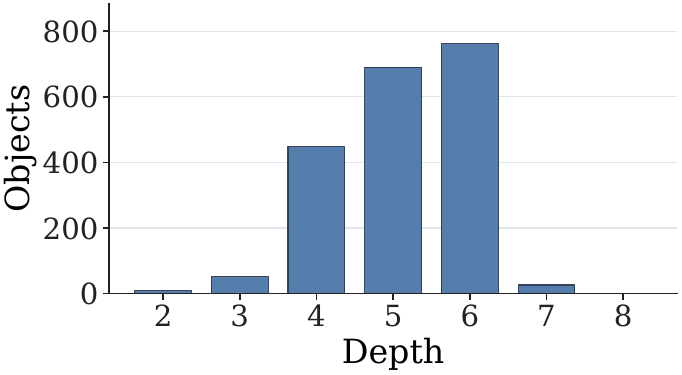}
\end{minipage}
\caption{
Statistics of the auxiliary training set. Left: domain distribution.
Middle: number of annotated objects per image. Right: distribution of
category-path depths.
}
\label{fig:appendix_training_statistics}
\end{figure*}

Overall, the auxiliary training set provides a compact source of
structured fine-grained supervision for preference construction. Its
scale is deliberately smaller than typical large-scale instruction-tuning
datasets, reflecting a low-resource setting where only a limited number
of fine-grained structured annotations are available.

\subsection{Reliability-Guided Semantic Rollback Prompts}
\label{sec:appendix_rsr_modifier_prompts}
\label{sec:appendix_fsr_modifier_prompts}

Reliability-Guided Semantic Rollback (RSR) is used to construct
reliability-oriented preference pairs. Given an original model response
and the diagnostic results from the evaluation pipeline, RSR revises
unsupported or overly specific category claims while preserving the
supported content of the response as much as possible.

RSR uses the image-specific GT category hierarchy to decide whether a
predicted category should be kept, deleted, or rolled back. For each
predicted category, we provide the LLM judge with the GT hierarchy of the
current image. The judge first determines the deepest GT-supported node
that the prediction can match. It then checks whether the prediction is
consistent with the child nodes under that matched node. This second
check is important because an over-specific prediction may share a coarse
ancestor with the GT hierarchy while still making a mutually exclusive
fine-grained claim.

Based on this hierarchy-aware judgment, RSR generates edit commands using
the following rules. If a predicted category cannot match any node in the
GT hierarchy, the corresponding claim is marked for deletion. If the
prediction matches a non-leaf node but contains a fine-grained claim that
conflicts with the child nodes under that node, the claim is rolled back
to the matched node. Otherwise, the prediction is kept unchanged. The
resulting edit commands are then applied to the original response by a
modifier model.

\begin{figure*}[t]
\centering
\begin{PromptCardNoBreak}
{\small
\promptlbl{RSR Edit Command Prompt}

\vspace{3pt}

\textbf{Role.}
You are a hierarchy-aware reliability judge. Given predicted category
claims from a model response and the GT category hierarchy of the current
image, decide whether each claim should be kept, deleted, or rolled back.

\vspace{5pt}

\textbf{Input.}
The input contains: (1) predicted category texts extracted from the model
response, and (2) the image-specific GT category hierarchy.

\vspace{5pt}

\textbf{Decision rules.}
For each predicted category, find the deepest GT-supported hierarchy node
that the prediction can match. If no GT-supported node can be matched,
mark the claim as \texttt{delete}. If the prediction matches a non-leaf
node but makes a fine-grained claim that conflicts with the child nodes
under that node, mark the claim as \texttt{rollback} and roll it back to
the matched node. Otherwise, mark the claim as \texttt{keep}. Missing
specificity is not a conflict: a coarse but correct prediction should be
kept. Ordinary visual modifiers such as color, material, size, pose, or
position should not be treated as category conflicts unless they are part
of an identity-bearing category name.

\vspace{5pt}

\textbf{Output.}
Return only a JSON array. Each item should correspond to one predicted
category claim:
\begin{verbatim}
[
  {
    "predicted_category": "...",
    "action": "keep | delete | rollback",
    "matched_node": "... or null",
    "rollback_to": "... or null",
    "reason": "brief explanation"
  }
]
\end{verbatim}

For \texttt{keep}, set \texttt{rollback\_to} to \texttt{null}. For
\texttt{delete}, set both \texttt{matched\_node} and
\texttt{rollback\_to} to \texttt{null}. For \texttt{rollback}, set
\texttt{rollback\_to} to the supported GT hierarchy node that should
replace the unsupported fine-grained claim.
}
\end{PromptCardNoBreak}
\caption{
Prompt used for generating edit commands in Reliability-Guided Semantic
Rollback.
}
\label{fig:appendix_rsr_edit_prompt}
\end{figure*}


The modifier receives the original response and the edit instructions, and directly returns the revised English description. It is instructed to apply only the specified edits, preserve unrelated content, and avoid introducing new visual claims.

\begin{figure*}[t]
\centering
\begin{PromptCardNoBreak}
{\small
\promptlbl{RSR Modifier Prompt}

\vspace{3pt}

\textbf{Role.}
You are a careful revision assistant for image descriptions. Your task is
to minimally edit an existing English image description according to
explicit object-level edit instructions.

\vspace{5pt}

\textbf{Input.}
The input is a JSON object containing the original description and a list
of edit instructions:
\begin{verbatim}
{
  "original_description": "...",
  "edit_instructions": [
    {
      "action": "rollback | delete",
      "source_category": "...",
      "target_category": "... or null"
    }
  ]
}
\end{verbatim}

\vspace{5pt}

\textbf{Instructions.}
For \texttt{action=rollback}, rewrite mentions of
\texttt{source\_category} as \texttt{target\_category} and remove
incompatible identity-specific details. For \texttt{action=delete},
remove sentences or clauses describing the hallucinated
\texttt{source\_category}. If a hallucinated object is intertwined with a
sentence about real objects, edit only the hallucinated part and keep the
real-object content. Preserve all unrelated content as much as possible.
Do not add new objects, attributes, counts, or visual details. Keep the
final description fluent, coherent, and in English.

\vspace{5pt}

\textbf{Output.}
Return only the final revised English description, with no explanation,
no markdown, and no quotes around the whole answer. Do not mention the
edit instructions or the words \texttt{hallucination},
\texttt{source\_category}, or \texttt{target\_category} in a meta way.
}
\end{PromptCardNoBreak}
\caption{
Prompt used for applying edit instructions and producing the RSR-revised
response.
}
\label{fig:appendix_rsr_modifier_prompt}
\end{figure*}

When RSR produces a revised response, we pair it with the original
response to form a reliability preference pair. The revised response is
treated as preferred because it preserves the supported content of the
original response while reducing unsupported or overly specific category
claims.

\subsection{Preference Pair and Margin Construction}
\label{sec:appendix_preference_data}

This subsection provides a brief recap of how the training preference
pairs are instantiated, together with illustrative examples. As
described in Section~\ref{sec:pref_construction}, we construct two
complementary preference sets: reliability preferences
$\mathcal{D}_{\mathrm{rel}}$ and granularity preferences
$\mathcal{D}_{\mathrm{gran}}$.

For each training image $I$, we sample $K$ candidate responses
$\mathcal{Y}_{\mathrm{orig}}(I)=\{y_{\mathrm{orig}}^{(k)}\}_{k=1}^{K}$
from the base model, and then apply Reliability-Guided Semantic Rollback
(RSR) to obtain the rectified response pool
$\mathcal{Y}_{\mathrm{rect}}(I)=\{\mathrm{RSR}(y_{\mathrm{orig}}^{(k)})\}_{k=1}^{K}$.
All preference pairs are constructed within the same image, so that the
comparison reflects response quality rather than image difficulty.

\paragraph{Reliability preferences.}
Reliability preferences compare an original response with its
RSR-rectified version. When RSR modifies a response by deleting an
unsupported category claim or rolling back an over-specific claim to a
supported coarser category, we form a preference pair
$(I,y_{+},y_{-})=(I,y_{\mathrm{rect}}^{(k)},y_{\mathrm{orig}}^{(k)})$.
The rectified response is treated as preferred because it preserves the
supported content of the original response while reducing unsupported
claims. For such pairs, the margin is computed from the improvement in
granularity-neutral precision, following
Section~\ref{sec:margin_dpo}.

\paragraph{Granularity preferences.}
Granularity preferences are constructed within the rectified response
pool of the same image. Specifically, we compare pairs
$(y_{+},y_{-})$ drawn from $\mathcal{Y}_{\mathrm{rect}}(I)$ and retain a
pair only when
$\mathrm{F1}_{\mathrm{gran}}(I,y_{+})
>
\mathrm{F1}_{\mathrm{gran}}(I,y_{-})+\tau$.
This encourages the model to prefer more informative responses after
reliability has already been improved by RSR. For these pairs, the
margin is computed from the normalized
$\mathrm{F1}_{\mathrm{gran}}$ gap.

\paragraph{Margins.}
Each preference pair is assigned an instance-specific margin
$m=\gamma+\alpha\delta$, where $\delta\in[0,1]$ is the normalized metric
gap of that pair. The gap definition depends on the preference type:
reliability pairs use the improvement in $\mathrm{P}$, while
granularity pairs use the improvement in
$\mathrm{F1}_{\mathrm{gran}}$. In practice, we use larger margin
hyperparameters for reliability preferences than for granularity
preferences, so that the overall training objective remains
reliability-first while still rewarding supported fine-grained
descriptions.

Figure~\ref{fig:appendix_reliability_preference_pair} shows an example
of a reliability preference pair, including the original response, the
RSR-rectified response, their metric scores, and the resulting margin.
Figure~\ref{fig:appendix_granularity_preference_pair} shows an example
of a granularity preference pair, where both responses are reliable, but
the preferred one receives a higher granularity-aware score and thus a
positive margin.

\begin{figure*}[t]
    \centering
    \includegraphics[width=\textwidth]{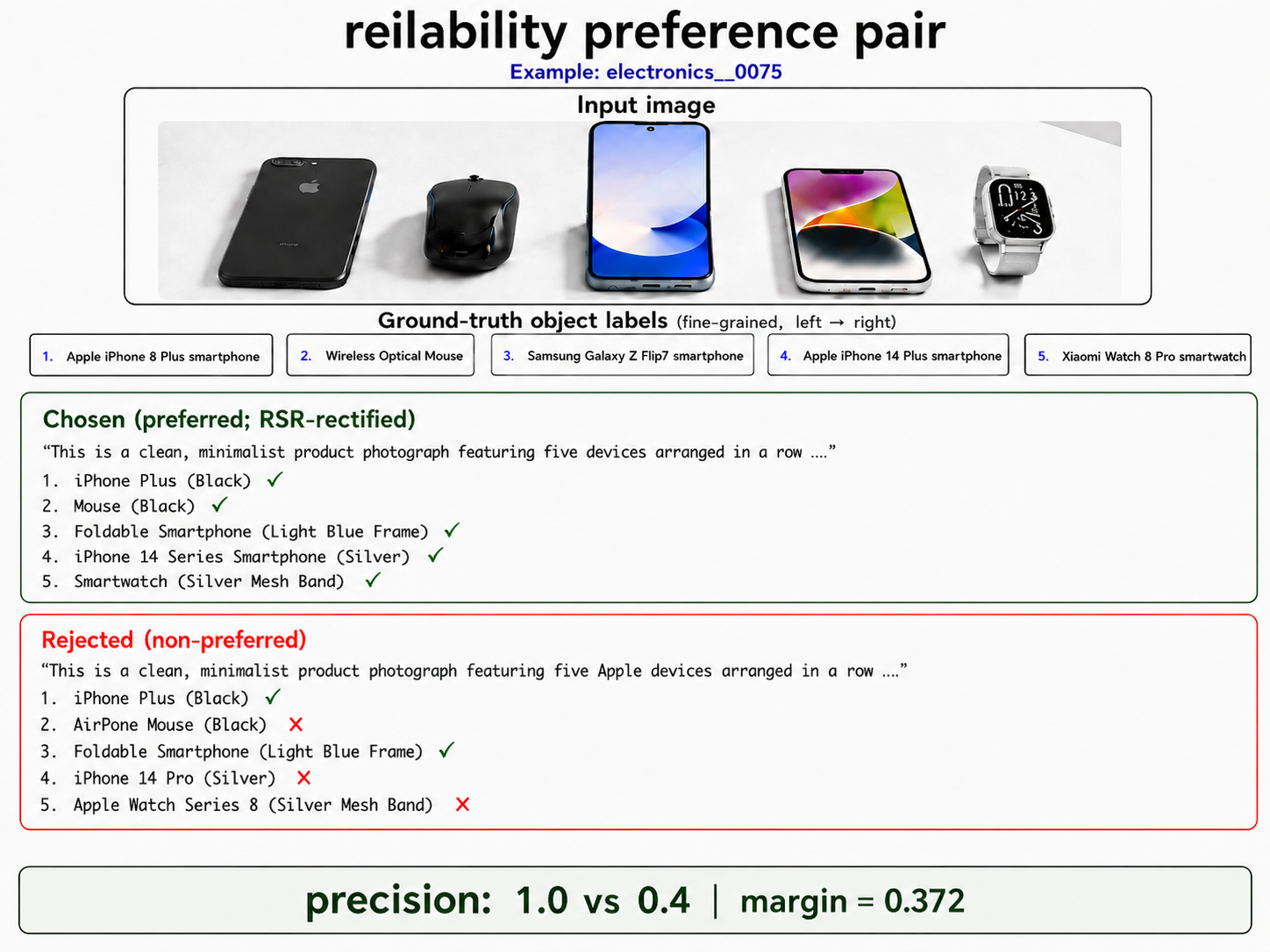}
    \caption{
    Example of a reliability preference pair.
    The preferred response is the RSR-rectified response, while the
    dispreferred response is the original response.
    The figure also shows the corresponding evaluation scores and the
    margin computed from the precision gap.
    }
    \label{fig:appendix_reliability_preference_pair}
\end{figure*}

\begin{figure*}[t]
    \centering
    \includegraphics[width=\textwidth]{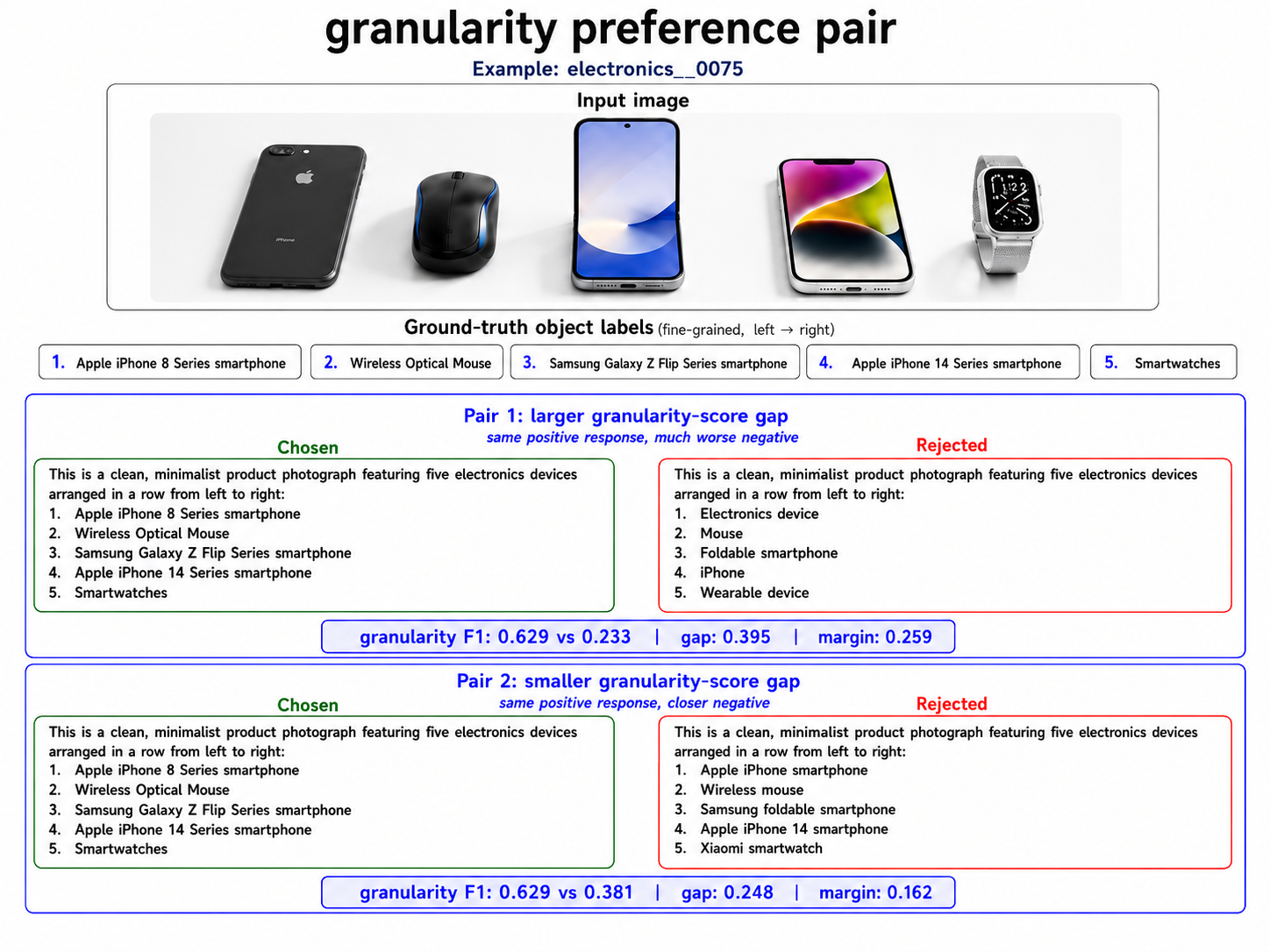}
    \caption{
    Example of a granularity preference pair.
    Both responses are drawn from the RSR-rectified response pool of the
    same image. The preferred response is selected because it achieves a
    higher granularity-aware score, and the figure shows the
    corresponding metric gap and margin.
    }
    \label{fig:appendix_granularity_preference_pair}
\end{figure*}

\subsection{Training Hyperparameters}
\label{sec:appendix_training_hparams}

For the model budget, RP-DPO is trained on Qwen3-VL-8B-Instruct, which has approximately 8B parameters. 
Training is conducted with LoRA, so only a small subset of adapter parameters is updated while the base model weights remain frozen.
All training runs are performed on 8 NVIDIA GeForce RTX 4090D GPUs. 
The full RP-DPO training run takes approximately \textbf{6} wall-clock hours, corresponding to approximately \textbf{48}$\times$ GPU-hours.

We train our method starting from Qwen3-VL-8B-Instruct as the base policy
model. The reference model in DPO is initialized from the same checkpoint
and kept frozen unless otherwise specified. Across response extraction,
assignment-score computation, and RSR, we use Qwen3.5-27B as the LLM
judge with temperature set to $0$ for reproducibility.

We use MS-Swift as the training framework and implement the RP-DPO
objective with our custom \texttt{RPDPOTrainer}. The main training
settings are summarized in
Table~\ref{tab:appendix_training_hparams}. Unless otherwise specified,
we use the final checkpoint after training for evaluation.

\begin{table}[t]
\centering
\small
\resizebox{\linewidth}{!}{%
\begin{tabular}{ll}
\toprule
Hyperparameter & Value \\
\midrule
Base policy model & Qwen3-VL-8B-Instruct \\
Reference model & Frozen initial policy \\
LLM judge & Qwen3.5-27B \\
Judge temperature & 0 \\
Training framework & MS-Swift \\
Trainer & Custom \texttt{RPDPOTrainer} \\
\midrule
$\lambda$ & 0.5 \\
$\gamma_{\mathrm{rel}}$ & 0.1 \\
$\alpha_{\mathrm{rel}}$ & 0.5 \\
$\gamma_{\mathrm{gran}}$ & 0 \\
$\alpha_{\mathrm{gran}}$ & 0.3 \\
DPO $\beta$ & 0.1\\
Pair filtering threshold $\tau$ & 0.002 \\
\midrule
Learning rate & $1\times 10^{-4}$ \\
Per-device train batch size & 1 \\
Gradient accumulation steps & 16 \\
Training steps & 33 \\
Precision & bf16 \\
Model size & 8B parameters \\
Hardware & 8 $\times$ NVIDIA GeForce RTX 4090D GPUs \\
GPU memory & 24GB per GPU \\
Training wall-clock time & \textbf{6} hours \\
Training compute budget & \textbf{48} GPU-hours \\
\midrule
LoRA rank & 32 \\
LoRA alpha & 32 \\
LoRA target modules & -all-linear \\
\bottomrule
\end{tabular}%
}
\caption{
Training settings for RP-DPO.
}
\label{tab:appendix_training_hparams}
\end{table}

\section{Additional Experimental Results}
\label{sec:appendix_additional_results}

This section provides additional experiments on comparison methods,
training stability, evaluation reliability, human-perceived generation
quality, and generalization.

\subsection{Comparison with Additional Baselines}
\label{sec:appendix_additional_baselines}

We additionally compare RP-DPO with two representative methods:
Visual Contrastive Decoding (VCD)~\cite{leng2024mitigating}, a
training-free hallucination mitigation method, and Multi-Objective Direct
Preference Optimization (MODPO)~\cite{zhou2024beyond}, a preference
optimization method for multiple objectives.
For a controlled comparison, MODPO uses the same base model, rollout data,
and training budget as RP-DPO, while VCD is applied to the same base model
at inference time.
All methods are evaluated using Qwen3-VL-8B under the aggressive prompt
style and the same evaluation pipeline.

\begin{table*}[t]
\centering
\small
\resizebox{\textwidth}{!}{%
\begin{tabular}{lccccccccc}
\toprule
Method &
GIR &
$P_{\mathrm{gran}}$ &
$R_{\mathrm{gran}}$ &
$F1_{\mathrm{gran}}$ &
$G_{\mathrm{avg}}$ &
IR &
P &
R &
F1 \\
\midrule
Prompt-based baseline
& 0.2450
& 0.3981
& 0.3907
& 0.3944
& 0.6040
& 0.3718
& 0.6591
& 0.6469
& 0.6529 \\

VCD
& 0.2431
& 0.3971
& 0.3889
& 0.3930
& 0.6010
& 0.3775
& 0.6607
& 0.6471
& 0.6538 \\

MODPO
& 0.2425
& 0.3926
& 0.3785
& 0.3854
& 0.5848
& 0.3879
& 0.6713
& 0.6473
& 0.6591 \\

RP-DPO
& \textbf{0.2655}
& \textbf{0.4183}
& \textbf{0.4034}
& \textbf{0.4107}
& \textbf{0.6225}
& \textbf{0.4017}
& \textbf{0.6719}
& \textbf{0.6480}
& \textbf{0.6597} \\
\bottomrule
\end{tabular}%
}
\caption{
Comparison with hallucination mitigation and preference optimization
baselines. All methods use Qwen3-VL-8B under the aggressive prompt style.
}
\label{tab:appendix_additional_baselines}
\end{table*}

VCD slightly improves IR over the prompt-based baseline, from 0.3718 to
0.3775, while reducing $G_{\mathrm{avg}}$ from 0.6040 to 0.6010.
MODPO further improves IR to 0.3879, but reduces $G_{\mathrm{avg}}$ more
substantially to 0.5848.
In contrast, RP-DPO improves both reliability and supported granularity,
achieving an IR of 0.4017 and a $G_{\mathrm{avg}}$ of 0.6225.
It also obtains the highest values across the other reported metrics.
These results indicate that RP-DPO improves reliability without encouraging
systematically coarser descriptions.

\subsection{Robustness across Training Seeds}
\label{sec:appendix_seed_robustness}

To assess training variability, we repeat RP-DPO training with three
different random seeds using Qwen3-VL-8B and evaluate the resulting
checkpoints under the aggressive prompt style.
Table~\ref{tab:appendix_seed_robustness} reports the mean and standard
deviation across the three training runs.
The prompt-based baseline does not involve preference training and is
reported as a fixed reference.

\begin{table}[t]
\centering
\small
\resizebox{\columnwidth}{!}{%
\begin{tabular}{lcccc}
\toprule
Method &
IR &
GIR &
$F1_{\mathrm{gran}}$ &
$G_{\mathrm{avg}}$ \\
\midrule
Prompt-based baseline
& 0.3718
& 0.2450
& 0.3944
& 0.6040 \\

RP-DPO
& $0.4016 \pm 0.0062$
& $0.2672 \pm 0.0015$
& $0.4099 \pm 0.0007$
& $0.6210 \pm 0.0033$ \\
\bottomrule
\end{tabular}%
}
\caption{
Results across three RP-DPO training runs with different random seeds.
RP-DPO results are reported as mean $\pm$ standard deviation.
}
\label{tab:appendix_seed_robustness}
\end{table}

Across the three runs, RP-DPO consistently outperforms the prompt-based
baseline on IR, GIR, $F1_{\mathrm{gran}}$, and $G_{\mathrm{avg}}$, while
exhibiting small standard deviations.
These results show that the observed improvements are stable across
different training seeds.

\subsection{Reliability of the Evaluation Pipeline}
\label{sec:appendix_evaluation_reliability}

We assess the reliability of the proposed evaluation pipeline from two
complementary perspectives.
First, human experts check the intermediate evaluation outputs produced by
the original judge model.
Second, we re-run the evaluation pipeline using two additional LLM judges
and examine whether the resulting metrics and model rankings remain
consistent.

\subsubsection{Human Check of Evaluation Outputs}
\label{sec:appendix_human_check}

We randomly sample 100 evaluation examples and present the intermediate
outputs produced by the original judge model, Qwen3.5-27B, to 10 human
experts.
The human check covers object extraction, prediction-to-ground-truth
matching, granularity assignment, and attribute matching.
Each case is assigned an integer score from 0 to 5 according to
Table~\ref{tab:appendix_human_check_rubric}.

\begin{table*}[t]
\centering
\small
\begin{tabularx}{\textwidth}{cX}
\toprule
Score & Criterion \\
\midrule
0 &
The judge fails to extract the target objects from the model response. \\

1 &
The judge extracts only part of the target objects, with major missing or
spurious objects. \\

2 &
The judge extracts most target objects correctly, but the
prediction-to-GT matching contains clear errors. \\

3 &
The judge extracts target objects and matches them to GT objects reasonably
well, but there are noticeable errors in granularity or attribute matching. \\

4 &
The judge correctly extracts target objects and matches them to GT objects,
with only minor errors in granularity or attributes. \\

5 &
The judge produces fully correct object extraction, prediction-to-GT
matching, granularity assignment, and attribute matching. \\
\bottomrule
\end{tabularx}
\caption{
Rubric for the human check of intermediate evaluation outputs.
}
\label{tab:appendix_human_check_rubric}
\end{table*}

As shown in Table~\ref{tab:appendix_human_check}, the judge model achieves
a mean score of 4.37.
Moreover, 93\% of the ratings are no lower than 4, and 99\% are no lower
than 3.
This indicates that the intermediate evaluation outputs are generally
considered correct or acceptable by human experts, with only minor
discrepancies in a small fraction of cases.

\begin{table}[t]
\centering
\small
\resizebox{\columnwidth}{!}{%
\begin{tabular}{ccc}
\toprule
Mean score &
High-quality rate ($\geq 4$) &
Acceptable rate ($\geq 3$) \\
\midrule
4.37 & 0.93 & 0.99 \\
\bottomrule
\end{tabular}%
}
\caption{
Human check of the Qwen3.5-27B judge on 100 randomly sampled evaluation
examples by 10 experts.
}
\label{tab:appendix_human_check}
\end{table}

\subsubsection{Robustness across Different LLM Judges}
\label{sec:appendix_cross_judge}

We additionally use Qwen3.6-Flash and DeepSeek-V4-Flash to re-run the
evaluation pipeline on the same 100 randomly sampled examples.
We evaluate InternVL3.5-8B, Qwen3-VL-8B, and Qwen3-VL-8B trained with
RP-DPO.
For conciseness, we report IR, GIR, $G_{\mathrm{avg}}$, P, and R, since
the remaining metrics can be derived from these values.
For each alternative judge, the values in parentheses in
Table~\ref{tab:appendix_cross_judge} indicate the relative change from
Qwen3.5-27B, and the final column reports the mean relative change across
the five metrics.

\begin{table*}[t]
\centering
\scriptsize
\resizebox{\textwidth}{!}{%
\begin{tabular}{llcccccc}
\toprule
Judge &
Evaluated model &
IR &
GIR &
$G_{\mathrm{avg}}$ &
P &
R &
Mean rel. change \\
\midrule

Qwen3.5-27B
& InternVL3.5-8B
& 0.3090
& 0.1612
& 0.5008
& 0.4804
& 0.5573
& -- \\

Qwen3.5-27B
& Qwen3-VL-8B
& 0.3646
& 0.2429
& 0.6006
& 0.6561
& 0.6454
& -- \\

Qwen3.5-27B
& Qwen3-VL-8B (RP-DPO)
& 0.3958
& 0.2631
& 0.6194
& 0.6689
& 0.6473
& -- \\

\midrule

Qwen3.6-Flash
& InternVL3.5-8B
& 0.3056 $(-1.10\%)$
& 0.1585 $(-1.67\%)$
& 0.4980 $(-0.56\%)$
& 0.4765 $(-0.81\%)$
& 0.5538 $(-0.63\%)$
& $-0.95\%$ \\

Qwen3.6-Flash
& Qwen3-VL-8B
& 0.3599 $(-1.29\%)$
& 0.2395 $(-1.40\%)$
& 0.5983 $(-0.38\%)$
& 0.6514 $(-0.72\%)$
& 0.6424 $(-0.46\%)$
& $-0.85\%$ \\

Qwen3.6-Flash
& Qwen3-VL-8B (RP-DPO)
& 0.3911 $(-1.19\%)$
& 0.2590 $(-1.56\%)$
& 0.6166 $(-0.45\%)$
& 0.6645 $(-0.66\%)$
& 0.6439 $(-0.53\%)$
& $-0.88\%$ \\

\midrule

DeepSeek-V4-Flash
& InternVL3.5-8B
& 0.3022 $(-2.20\%)$
& 0.1563 $(-3.04\%)$
& 0.4965 $(-0.86\%)$
& 0.4741 $(-1.31\%)$
& 0.5511 $(-1.11\%)$
& $-1.70\%$ \\

DeepSeek-V4-Flash
& Qwen3-VL-8B
& 0.3571 $(-2.06\%)$
& 0.2350 $(-3.25\%)$
& 0.5959 $(-0.78\%)$
& 0.6468 $(-1.42\%)$
& 0.6371 $(-1.29\%)$
& $-1.76\%$ \\

DeepSeek-V4-Flash
& Qwen3-VL-8B (RP-DPO)
& 0.3866 $(-2.32\%)$
& 0.2540 $(-3.46\%)$
& 0.6138 $(-0.90\%)$
& 0.6588 $(-1.51\%)$
& 0.6394 $(-1.22\%)$
& $-1.88\%$ \\
\bottomrule
\end{tabular}%
}
\caption{
Cross-judge robustness on 100 randomly sampled examples from
\textsc{GranFact}.
Values in parentheses report the relative change from the original
Qwen3.5-27B judge.
}
\label{tab:appendix_cross_judge}
\end{table*}

The results produced by the two alternative judges remain close to those
of Qwen3.5-27B.
For every evaluated model, the magnitude of the mean relative change
across the five metrics remains below 2\%.
Moreover, the three judges produce the same relative ranking among the
evaluated models.
These results indicate that the main evaluation conclusions are robust to
the choice of LLM judge.

\subsection{Human Evaluation of Generated Descriptions}
\label{sec:appendix_human_generation_eval}

We further conduct a human evaluation of the descriptions generated by
the prompt-based Qwen3-VL-8B and RP-DPO.
We randomly sample 100 images and use the cached descriptions generated
under the same evaluation setting.
Three annotators independently score each description from 1 to 5 in terms
of factual correctness, informativeness, fluency, and overall quality.
The evaluation rubric is provided in
Table~\ref{tab:appendix_human_generation_rubric}.

\begin{table*}[t]
\centering
\small
\begin{tabularx}{\textwidth}{lX}
\toprule
Dimension & Anchor descriptions \\
\midrule

Factual correctness &
\textbf{1:} Contains major hallucinations or contradicts the image.
\textbf{3:} Mostly correct but contains minor unsupported details.
\textbf{5:} Fully supported by the image. \\

Informativeness &
\textbf{1:} Overly generic or misses salient visual content.
\textbf{3:} Covers the main content with moderate detail.
\textbf{5:} Provides rich and appropriate details without unnecessary
speculation. \\

Fluency &
\textbf{1:} Disfluent or difficult to understand.
\textbf{3:} Generally understandable with minor language issues.
\textbf{5:} Fluent, natural, and coherent. \\

Overall quality &
\textbf{1:} Poor overall description.
\textbf{3:} Acceptable description.
\textbf{5:} High-quality description that is accurate, informative, and
fluent. \\
\bottomrule
\end{tabularx}
\caption{
Human-evaluation rubric.
Annotators assign an integer score from 1 to 5 for each dimension.
Scores 2 and 4 represent intermediate quality between the adjacent
anchors.
}
\label{tab:appendix_human_generation_rubric}
\end{table*}

\begin{table}[t]
\centering
\small
\resizebox{\columnwidth}{!}{%
\begin{tabular}{lcccc}
\toprule
Method &
Factual correctness &
Informativeness &
Fluency &
Overall quality \\
\midrule
Prompt-based baseline
& 3.42
& 2.93
& \textbf{4.83}
& 3.74 \\

RP-DPO
& \textbf{3.84}
& \textbf{3.20}
& \textbf{4.83}
& \textbf{3.93} \\
\bottomrule
\end{tabular}%
}
\caption{
Human evaluation of cached prompt-based baseline and RP-DPO descriptions
on 100 randomly sampled images.
Each description is independently assessed by three annotators on a
five-point scale.
}
\label{tab:appendix_human_generation_eval}
\end{table}

As shown in Table~\ref{tab:appendix_human_generation_eval}, RP-DPO obtains
higher mean scores for factual correctness, informativeness, and overall
quality, while maintaining the same fluency score as the prompt-based
baseline.
These results indicate that the improvements measured by the automatic
evaluation pipeline are also reflected in human judgments of the generated
descriptions.

\subsection{Generalization Analysis}
\label{sec:appendix_generalization}

We examine generalization from two perspectives.
First, we evaluate RP-DPO separately on semantic classes that are covered
and not covered by the auxiliary training data.
Second, we quantify the visual distribution gap between the synthetic
training images and the real-world \textsc{GranFact} evaluation images.

\subsubsection{Seen- and Unseen-Class Performance}
\label{sec:appendix_seen_unseen}

The auxiliary training data cover five classes: cars, animals, plants,
daily objects, and electronics.
We define images from these classes as the seen-class subset.
The game and landmark classes are not used for auxiliary training and form
the unseen-class subset.
Table~\ref{tab:appendix_seen_unseen} reports the results of the
prompt-based Qwen3-VL-8B and RP-DPO on both subsets under the aggressive
prompt style.

\begin{table*}[t]
\centering
\small
\resizebox{\textwidth}{!}{%
\begin{tabular}{llccccccccc}
\toprule
Subset &
Method &
GIR &
$P_{\mathrm{gran}}$ &
$R_{\mathrm{gran}}$ &
$F1_{\mathrm{gran}}$ &
$G_{\mathrm{avg}}$ &
IR &
P &
R &
F1 \\
\midrule

Seen-class
& Prompt-based baseline
& 0.2346
& 0.4149
& 0.4228
& 0.4188
& 0.6079
& 0.3564
& 0.6825
& 0.6955
& 0.6889 \\

Seen-class
& RP-DPO
& \textbf{0.2548}
& \textbf{0.4353}
& \textbf{0.4351}
& \textbf{0.4352}
& \textbf{0.6250}
& \textbf{0.3887}
& \textbf{0.6965}
& \textbf{0.6962}
& \textbf{0.6963} \\

\midrule

Unseen-class
& Prompt-based baseline
& 0.2857
& 0.3238
& 0.2733
& 0.2964
& 0.5826
& 0.4322
& 0.5557
& 0.4691
& 0.5088 \\

Unseen-class
& RP-DPO
& \textbf{0.3075}
& \textbf{0.3437}
& \textbf{0.2874}
& \textbf{0.3130}
& \textbf{0.6091}
& \textbf{0.4525}
& \textbf{0.5643}
& \textbf{0.4718}
& \textbf{0.5139} \\
\bottomrule
\end{tabular}%
}
\caption{
Generalization results on seen- and unseen-class subsets using
Qwen3-VL-8B under the aggressive prompt style.
}
\label{tab:appendix_seen_unseen}
\end{table*}

RP-DPO improves all reported metrics on both the seen- and unseen-class
subsets.
In particular, the consistent improvements on games and landmarks indicate
that the benefits of RP-DPO are not restricted to the classes covered by
the auxiliary preference-training data.

\subsubsection{Visual Distribution Gap}
\label{sec:appendix_visual_distribution}

The seen-class subset shares semantic classes with the auxiliary training
data, but it is not necessarily visually in-distribution.
The auxiliary training images are synthetic, whereas the
\textsc{GranFact} evaluation images are real-world images with different
visual styles and scene compositions.
We therefore quantify the visual distribution gap using the Fréchet
distance between CLIP image-embedding distributions and the average
nearest-neighbor similarity between the evaluation images and the
auxiliary training set.

As an internal reference, we first compare two splits of the synthetic
training images.
We then compare the complete synthetic training set with
\textsc{GranFact}-All, \textsc{GranFact}-Seen, and
\textsc{GranFact}-Unseen.
Results are reported as mean $\pm$ standard deviation over 100 random
subsamples.

\begin{table}[t]
\centering
\small
\resizebox{\columnwidth}{!}{%
\begin{tabular}{lcc}
\toprule
Image-set comparison &
Fréchet CLIP distance &
Avg. nearest-neighbor similarity \\
\midrule
Training split A vs. Training split B
& $78.89 \pm 1.42$
& $0.7646 \pm 0.0031$ \\

Training set vs. \textsc{GranFact}-All
& $188.97 \pm 4.48$
& $0.5329 \pm 0.0078$ \\

Training set vs. \textsc{GranFact}-Seen
& $167.79 \pm 4.21$
& $0.5912 \pm 0.0076$ \\

Training set vs. \textsc{GranFact}-Unseen
& $318.67 \pm 2.22$
& $0.3210 \pm 0.0047$ \\
\bottomrule
\end{tabular}%
}
\caption{
Visual distribution comparison between the synthetic auxiliary training
images and the real-world \textsc{GranFact} evaluation images.
A larger Fréchet CLIP distance and a lower average nearest-neighbor
similarity indicate a greater visual distribution gap.
}
\label{tab:appendix_visual_distribution}
\end{table}

Compared with the internal training-split comparison, all
\textsc{GranFact} subsets have substantially larger Fréchet CLIP distances
and lower nearest-neighbor similarities.
The same pattern holds for the seen-class subset, indicating that sharing
semantic classes with the auxiliary training data does not make the
evaluation images visually in-distribution.
Together with the unseen-class results, these findings support the
generalization of RP-DPO beyond both the training-covered classes and the
synthetic visual distribution used for preference construction.

\subsection{Qualitative Examples and Error Analysis}
\label{sec:appendix_qualitative_examples}

Figure~\ref{fig:qualitative_response_cards} presents the original responses generated by Gemini-3.1-Flash under conservative and aggressive prompts, and Figure~\ref{fig:gpt_assisted_case_analysis} further visualizes how our evaluator scores them.
For each response, we first parse object-level predictions and then compute an optimal global matching to the ground-truth (GT) objects.
For a sample \(z\), we report both the matching size \(M_z\), i.e., the number of matched prediction--GT pairs under the optimal assignment, and the total granularity score \(M_{z,\mathrm{gran}}\), i.e., the sum of the edge-level granularity scores over all matched pairs.

Under the conservative prompt, the parsed predictions are
\textit{smartphone},
\textit{Samsung smartphone},
\textit{Samsung smartphone}, and
\textit{foldable smartphone}.
All four predictions can be matched to GT objects, giving \(M_z=4\).
However, these matches remain at relatively coarse semantic levels, resulting in a lower total granularity score \(M_{z,\mathrm{gran}}=1.6123\).

Under the aggressive prompt, the parsed predictions become
\textit{Galaxy Z Fold smartphone},
\textit{Samsung smartphone},
\textit{Galaxy S24 Ultra smartphone}, and
\textit{Galaxy Z Flip smartphone}.
The first and fourth predictions are more specific and receive substantially higher granularity scores, which increases the total granularity score to \(M_{z,\mathrm{gran}}=2.0459\).
However, the third prediction is an unsupported fine-grained claim: it identifies the true \textit{Samsung Galaxy S23 Ultra smartphone} as a \textit{Galaxy S24 Ultra smartphone}.
As a result, this prediction remains unmatched, the corresponding GT object is missed, and the matching size drops to \(M_z=3\).

This example illustrates the reliability--granularity trade-off captured by our evaluation.
Aggressive prompting can improve semantic specificity for matched objects, but it may also introduce unsupported fine-grained predictions that reduce object coverage under global matching.
By reporting both \(M_z\) and \(M_{z,\mathrm{gran}}\), our evaluator makes this trade-off explicit at the sample level.

\begin{figure*}[t]
\centering
\includegraphics[width=0.68\textwidth]{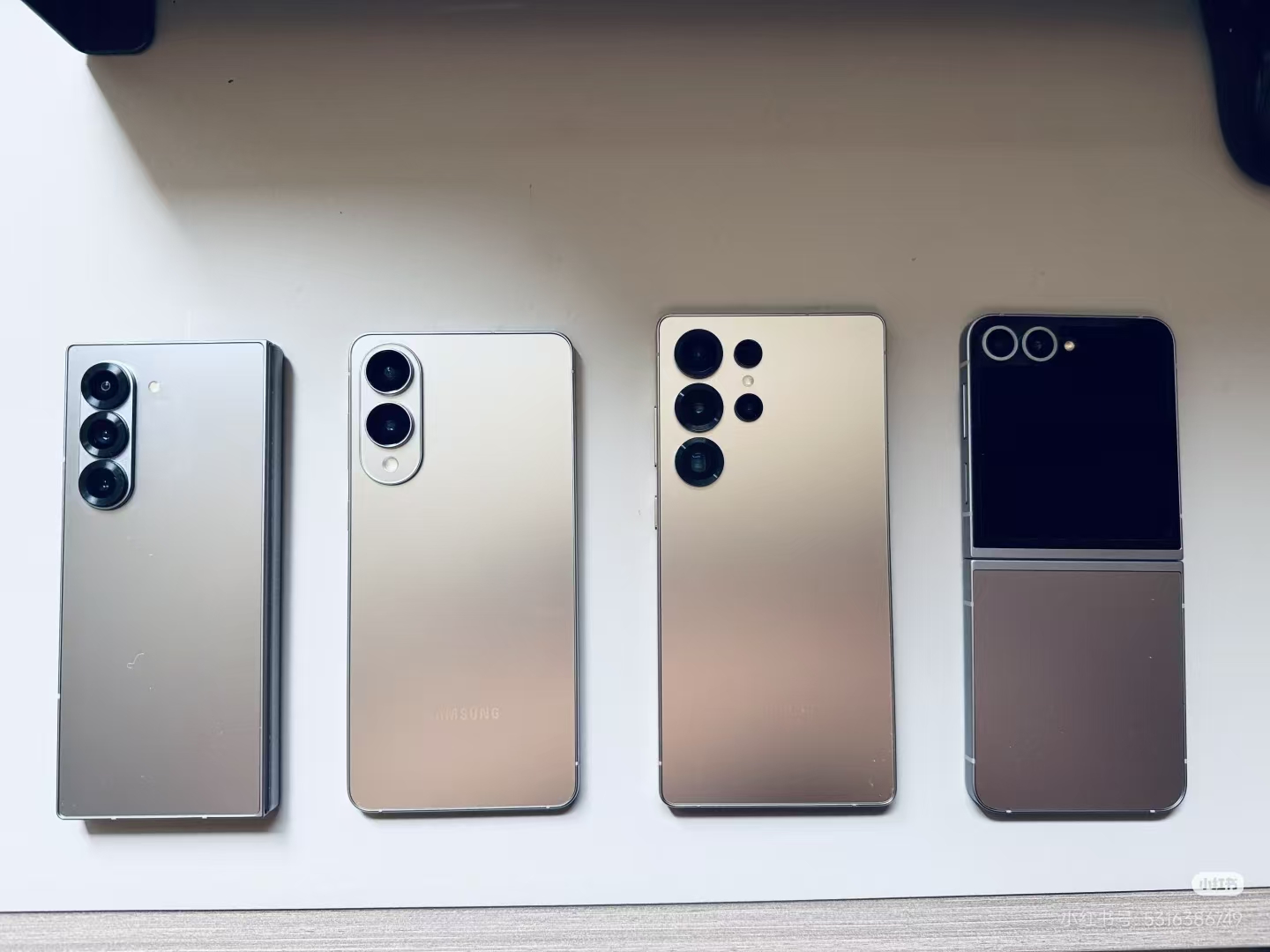}

\vspace{1mm}

\begin{ResponseCardNoBreak}{ConsAccent}{Gemini-3.1-Flash under the conservative prompt}
The image displays four smartphones arranged in a row on a white surface, viewed from the back.

From left to right:
\begin{enumerate}[leftmargin=1.35em,itemsep=0pt,topsep=1pt,parsep=0pt]
    \item \textbf{First phone:} A tall, narrow, light-colored device with a vertical camera module containing three lenses.
    \item \textbf{Second phone:} A light-colored device with a vertical camera module containing two lenses and a flash. The word ``SAMSUNG'' is visible near the bottom.
    \item \textbf{Third phone:} A light-colored device with a camera array consisting of five circular elements arranged in a ``P'' shape. The word ``SAMSUNG'' is faintly visible near the bottom.
    \item \textbf{Fourth phone:} A foldable device shown in its closed state. It has a dark upper section containing two camera lenses and a flash, and a light-colored lower section.
\end{enumerate}
\end{ResponseCardNoBreak}

\vspace{1mm}

\begin{ResponseCardNoBreak}{AggrAccent}{Gemini-3.1-Flash under the aggressive prompt}
This high-angle, eye-level shot displays four Samsung smartphones arranged in a row on a plain, light-colored surface. The phones are positioned vertically, showing their back panels.

Starting from the left:
\begin{enumerate}[leftmargin=1.35em,itemsep=0pt,topsep=1pt,parsep=0pt]
    \item \textbf{First Phone:} This is a tall, narrow device, likely a Galaxy Z Fold model. It has a matte, light-gray finish. On the left side, there is a vertical camera module containing three circular lenses stacked one above the other.
    \item \textbf{Second Phone:} This device has a smooth, metallic, champagne-gold finish. It features a vertical camera housing on the top left with two large, prominent circular lenses and a small flash below them. The ``SAMSUNG'' logo is faintly visible in the lower center of the back panel.
    \item \textbf{Third Phone:} This appears to be a Galaxy S24 Ultra or similar model. It has a similar champagne-gold, matte finish. The camera array is distinct, featuring five circular elements: three large lenses arranged vertically on the left, and two smaller circular components, a lens and a flash/sensor, to their right. The ``SAMSUNG'' logo is subtly embossed near the bottom.
    \item \textbf{Fourth Phone:} This is a foldable device, likely a Galaxy Z Flip model, shown in its closed state. The top half of the back is a glossy black panel containing two circular camera lenses and a small flash. The bottom half is a matte, muted-brown or taupe-colored panel. The phone has a metallic frame visible along the edges.
\end{enumerate}

The lighting is even and soft, casting minimal shadows, which highlights the different textures and finishes of the devices. The background is a clean, neutral white or off-white surface.
\end{ResponseCardNoBreak}

\vspace{-1mm}
\caption{
Qualitative responses of Gemini-3.1-Flash under conservative and aggressive prompts.
}
\label{fig:qualitative_response_cards}
\vspace{-2mm}
\end{figure*}

\begin{figure*}[t]
    \centering
    \includegraphics[width=\textwidth]{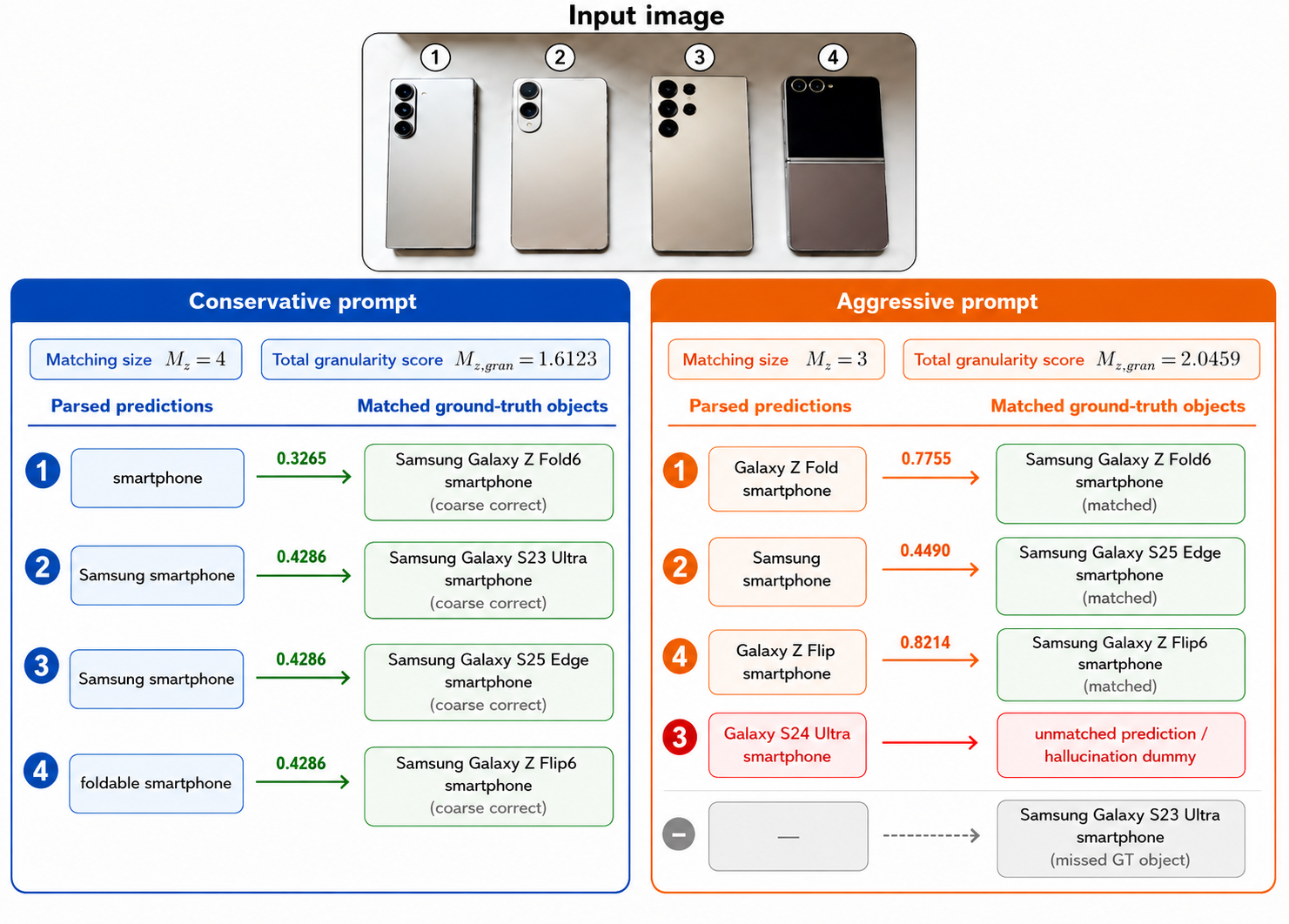}
    \caption{
    Evaluation outcomes for the qualitative example in Figure~\ref{fig:qualitative_response_cards}.
    The conservative response matches all four GT objects ($M_z=4$) but at coarse semantic levels ($M_{z,\mathrm{gran}}=1.6123$).
    The aggressive response achieves finer matched descriptions ($M_{z,\mathrm{gran}}=2.0459$) but introduces an unsupported fine-grained claim, reducing the reliablility score to $M_z=3$.
    }
    \label{fig:gpt_assisted_case_analysis}
\end{figure*}

\section{Artifact Use, License, and Intended Use}
\label{sec:appendix_artifact_use}

\paragraph{Artifacts used and created.}
This work uses publicly available and API-based multimodal large language models for research evaluation and, where applicable, model training. We cite the creators of the main external artifacts used or built upon, including pretrained MLLMs and Direct Preference Optimization. Prior datasets and benchmarks discussed for comparison are cited as related work.

We introduce \textsc{GranFact}, including expert-verified image annotations, a hierarchy-aware evaluation protocol, and evaluation code. The benchmark images are collected from public web sources and real-world photographs. For images collected from public sources, redistribution will follow the corresponding source licenses or terms of use. When redistribution of raw images is not permitted or cannot be verified, we will release only metadata, annotations, and evaluation scripts, or provide instructions for reconstructing the benchmark where appropriate. The annotations and evaluation code will be released under research-friendly licenses, such as CC BY-NC 4.0 for annotations and MIT License for evaluation code, subject to compatibility with the licenses and terms of the underlying resources.

\paragraph{Privacy and offensive content.}
\textsc{GranFact} is designed to evaluate object-level fine-grained visual description rather than person identification. During data collection and annotation, we screen images to avoid content that names or uniquely identifies individual people, such as faces, personal documents, account names, or other personally identifying information. We also filter out offensive, hateful, sexually explicit, or otherwise unsafe content. The released annotations describe object categories, quantities, and visual attributes, and do not include personal information. If any problematic content is later identified, we will remove or replace the corresponding example.

\paragraph{Documentation and statistics.}
We document the construction and coverage of \textsc{GranFact}, including image sources, visual domains, annotation protocol, category hierarchy, object-level annotations, and evaluation metrics. We also report dataset statistics, including the number of evaluation images, domain distribution, number of annotated objects per image, and granularity depth. The auxiliary training set used for RDPO is disjoint from the \textsc{GranFact} evaluation set, and its construction is described separately in the appendix.

\paragraph{Intended use.}
\textsc{GranFact} and the associated evaluation protocol are intended for research on reliable fine-grained multimodal generation and evaluation. They are designed to analyze whether MLLMs can generate visually grounded descriptions at appropriate semantic granularities. They should not be used as guarantees of correctness, nor for high-stakes decision making, privacy-sensitive identification, surveillance, or other applications where unsupported fine-grained predictions could cause harm. Users of the artifacts are responsible for complying with the licenses and terms of use of any underlying images, models, and APIs.

\end{document}